\documentclass[twoside,11pt]{article}

\usepackage{jmlr2e}

\usepackage{graphicx}
\usepackage{algorithm}
\usepackage{algorithmic}
\usepackage{tabularx}
\usepackage{mymacro}
\usepackage{amsmath,txfonts,bm}
\usepackage{multirow}

\usepackage{mathrsfs}

\jmlrheading{1}{2000}{1-48}{4/00}{10/00}{Katsuhiko Ishiguro, Issei Sato,
and Nanori Ueda}

\ShortHeadings{CVB Inference of IRM}{Ishiguro, Sato, and Ueda}
\firstpageno{1}

\begin{document}

\title{Collapsed Variational Bayes Inference of Infinite Relational Model}

\author{\name Katsuhiko Ishiguro \email ishiguro.katsuhiko@lab.ntt.co.jp \\
       \addr NTT Communication Science Laboratories\\
       NTT Corporation\\
       Kyoto 619-0237, Japan
       \AND
       \name Issei Sato \email sato@r.dl.itc.u-tokyo.ac.jp \\
       \addr Information Technology Center\\
       The University of Tokyo\\
       Tokyo 113-0033, Japan
	   \AND
	   \name Naonori Ueda \email ueda.naonori@lab.ntt.co.jp \\
       \addr NTT Communication Science Laboratories\\
       NTT Corporation\\
       Kyoto 619-0237, Japan}

\editor{Unknown Editor}

\maketitle

\begin{abstract}%   <- trailing '%' for backward compatibility of .sty file
The Infinite Relational Model (IRM) is a probabilistic model for relational data clustering 
that partitions objects into clusters based on observed relationships. 
This paper presents Averaged CVB (ACVB) solutions for IRM, 
convergence-guaranteed and practically useful fast Collapsed Variational Bayes (CVB) inferences. 
We first derive ordinary CVB and CVB0 for IRM based on the lower bound maximization. 
CVB solutions yield deterministic iterative procedures 
for inferring IRM given the truncated number of clusters. 
Our proposal includes CVB0 updates of hyperparameters including the concentration parameter of the
Dirichlet Process, which has not been studied in the literature. 
To make the CVB more practically useful, we further study the CVB inference in two aspects. 
First, we study the convergence issues and 
develop a convergence-guaranteed algorithm for any CVB-based 
inferences called ACVB, which enables automatic 
convergence detection and frees non-expert practitioners from 
difficult and costly manual monitoring of inference processes. 
Second, we present a few techniques for speeding up IRM inferences. 
In particular, we describe the linear time inference of CVB0, allowing the IRM for larger relational data uses. 
The ACVB solutions of IRM showed comparable or better performance 
compared to existing inference methods in experiments, 
and provide deterministic, faster, and easier convergence detection. 
\end{abstract}


\section{Introduction}
Analysis of pairwise relational data, such as
friend-links on social network services (SNS), 
customer records of purchases in online shops, 
and bibliographic citations between scientific articles, 
is useful in many ways. 
Many statistical models for relational data
have been presented in the literature~\citep{Clauset08,Erosheva04,LibenNowell_Kleinberg03,Zhu09}. 
Among them, 
the infinite relational model (IRM) proposed by~\citet{Kemp06} achieves simultaneous bi-clustering 
on the row and column dimensions of a given pairwise relational data matrix. 
For example, in the case of customer records, rows and columns correspond 
to users and items. In such a case, the row and column clusters are 
interpreted as latent user groups and item topics, respectively. 
IRM adopts nonparametric Bayes modeling and so can automatically estimate the number of clusters. 
This makes IRM a convenient tool for 
relational data analysis without the need for careful model selection. 

Two Bayesian inference algorithms are frequently used for 
probabilistic generative models including IRM: 
the Gibbs sampler and variational Bayes. 
The former guarantees asymptotic convergence to the true
posteriors of random variables given infinitely many stochastic samples.  
Variational Bayes (VB) solutions often enjoy faster convergence
with deterministic iterative computations and 
massively parallel computation thanks to the factorization. 
The VB approaches also allow easy and automatic detection of convergence.  
Instead of these favorable properties, VB only yields local optimal 
solutions due to factorized approximated posteriors. 

We can improve these inference methods by developing collapsed estimators,
which integrate out some parameters from inferences. 
Collapsed Gibbs samplers are one of the best inference solutions as they achieve faster convergence and better estimation than the original Gibbs samplers. 
Recently, collapsed variational Bayes (CVB) solutions have been
intensively studied, especially for 
topic models such as latent Dirichlet allocation 
(LDA)~\citep{Teh07NIPS,Asuncion09,Sato_Nakagawa12} 
and HDP-LDA~\citep{Sato12}. 
The original paper~\citep{Teh07NIPS} examined a 2nd-order Taylor approximation of the variational expectation. A simpler 0th order-approximated CVB 
(CVB0) solution also has been developed; it is an optimal solution in the sense of minimized $\alpha$-divergence ~\citep{Sato_Nakagawa12}. 
These papers report that CVB and CVB0 yield better inference results than VB solutions, even slightly better than exact collapsed Gibbs, 
in data modeling ~\citep{Kurihara07,Teh07NIPS,Asuncion09}, link 
predictions, and neighborhood search~\citep{Sato12}. 

Most IRM papers to date ~\citep{Kemp06,AISTATS12,Morup10,Albers13} 
rely on (collapsed) Gibbs samplers. 
However, the automatic convergence detection of stochastic sampling-based 
Gibbs is difficult to achieve~\citep{Cowles_Carlin96}. 
This is not preferable for non-expert users to employ IRM in practical uses. 
Further, \citep{Albers13} reported that the naive implementation of (collapsed) Gibbs 
is very slow in mixing for IRM applications. 
However, interestingly, 
there has been no attempt to use VB for IRM to the best of our knowledge, 
even though VB allows easy and automatic detection of convergence, plus fast deterministic computations. 
One reason is that VB may perform poorly for IRM because
IRM solves difficult partitioning problems with many local optima. 
CVB and CVB0 are promising alternatives to VB, but 
most CVB studies have focused on topic models, 
which well suits simple Bag-of-Word style data sets. 
The only exceptions are CVB for Probabilistic Context-Free Grammars~\citep{Wang_Blunsom13ACL} and Hidden Markov
Models~\citep{Wang_Blunsom13}. 

In this paper, we first formulate and derive the CVB inference of IRM for relational
data analysis as fast, deterministic and precise inference algorithms, 
which replace naive VB. 
Furthermore, we derive update rules of hyperparameters based on CVB0; 
thus we can automatically optimize all hyperparameters. 
In particular, the update of the concentration parameter of DP 
has not been studied in the literature, which plays an important role in nonparametric Bayes.  
In \mytabref{tab:IRM_works}, we summarize the existing inference algorithms used in the 
previous IRM studies and this paper. 

\begin{table}[t]
\caption{Inference algorithms presented in existing IRM works and this paper. }
\label{tab:IRM_works}
\begin{center}
\begin{tabularx}{150mm}{X||c|c|c|c|l}
Paper & (collapsed) Gibbs & VB & CVB & CVB0 & Comments \\ \hline 
\citet{Kemp06} & $\checkmark$ & - & - & - & The seminal paper \\
\citet{Morup10} & $\checkmark$ & - & - & - &  Application to fMRI \\
\citet{Hansen11} & $\checkmark$ & - & - & - & GPU impl. \\
\citet{Albers13} & $\checkmark$ & - & - & - & OpenMP impl. \\ \hline
\citet{AISTATS12} & $\checkmark$ & - & - & - & Noise filtering extension \\ \hline
This paper & $\checkmark$ & $\checkmark$ & $\checkmark$ & $\checkmark$ & Fully covers
\end{tabularx}
\end{center}
\end{table}

Next, we study CVB inference of IRM in two aspects 
to make the CVB inference easier to use for practitioners. 
% convergence
The first aspect is the convergence. 
The convergence behavior of CVB inference is still difficult to analyze theoretically, 
and there is no convergence guarantee for the general CVB inferences. 
This problem, interestingly, has not been much discussed in the literature. 
The sole exception is \citep{Foulds13}, which uses online stochastic learning valid for LDA. 
However, this is a tricky and problematic issue for practitioners who are not familiar with but want to try state-of-the-art machine learning techniques. 
Users are required to determine the convergence of CVB inference manually: 
this is not an easy task for non-expert users. 
In that sense, CVB is not as favorable as naive VB and EM algorithms. 
In this paper, we empirically study the convergence behaviors of CVB inference for IRM. 
We first monitor the naive variational lower bound and the pseudo leave-one-out 
training log likelihood, and empirically show that both may serve as
convergence detectors. 
Then, we develop a simple and effective technique that assures convergence of CVB 
for any probabilistic model including IRM. 
The proposed annealing technique called 
\textbf{Averaged CVB (ACVB)} guarantees the convergence of CVB and allows automatic convergence detection. 
ACVB has two advantages. 
First, ACVB posterior update offers assured convergence thanks to its simple annealing mechanism. 
Second, the stationary point of the CVB lower bound is equivalent to the converged solution of ACVB, 
if the lower bound has a stationary point (an issue unresolved in the literature). 
Our formulation is applicable to any model, and is equally valid for CVB and CVB0. 
Convergence-guaranteed ACVB is the preferred choice for practitioners who want to apply state-of-the-art inference to their problems. 
In \mytabref{tab:CVB_works},  we summarize the existing CVB works and this paper, 
based on applied models and the convergence issue. 

\begin{table}[t]
\caption{CVB-related studies summary: in terms of applied models and convergence. }
\label{tab:CVB_works}
\begin{center}
\begin{tabularx}{150mm}{X||c|c|p{10em}}
Paper & Applied model & Convergence & Comments \\ \hline 
\citet{Teh07NIPS} & LDA-CVB & - & The seminal paper \\
\citet{Asuncion09} & LDA-CVB, CVB0 & - & Introduces CVB0 \\
\citet{Sato_Nakagawa12} & LDA-CVB, CVB0 & - & Optimality analysis by $\alpha$-divergence \\
\citet{Foulds13} & LDA-CVB0 & $\checkmark$ & Stochastic approx. valid for LDA  \\ \hline
\citet{Kurihara07} & DPM-CVB & - & First attempt at DPM \\
\citet{Teh08NIPS_HDP} & HDP-CVB & - & First attempt at HDP \\
\citet{Sato12} & HDP-CVB, CVB0 & - & Approx. solution \\ \hline
\citet{Wang_Blunsom13} & HMM-CVB0 & - & First attempt at HMM \\ 
\citet{Wang_Blunsom13ACL} & PCFG-CVB0 & - & First attempt at PCFG \\ \hline
This paper & IRM-CVB, CVB0 & $\checkmark$ & First attempt at IRM, convergence assurance for any model
\end{tabularx}
\end{center}
\end{table}

% speed up
The second aspect is computational cost. 
In naive implementation, CVB inference of IRM requires square time over the number of objects. 
For example, if the buy-product relations $\bm{X}$ are observed between 
$N_{1}$ users and $N_{2}$ items, and we assume $K_{1}$ and $K_{2}$ latent clusters among 
users and items respectively, then the inference costs $O\left(K_{1} K_{2} N_{1} N_{2}\right)$ per iteration. 
This makes the IRM an impractical solution for large relational data. 
In this paper, we describe how to mitigate this computational cost, especially for the CVB0 solution. 
As a result, we show that we can solve the CVB0 solution 
by $O\left(K_{1} K_{2} L (N_{1} + N_{2})\right)$, 
linear to the number of users and items, where $L$ denotes the average degree. 
Combining these techniques, we propose 
a practically useful \textbf{ACVB0} solution for IRM, 
with easy detection of guaranteed convergence and 
linear-time computation. 

We experimentally show that ACVB solutions for IRM
offer comparable, or even better performance 
in terms of data modeling (test likelihood) than naive variational Bayes on multiple synthetic and real-world relational data sets.  
ACVB0 convergences are magnitude faster than the VB and ACVB in terms of CPU times, 
presented stable and nice convergence behaviors throughout the datasets. 
%and it is easy to detect convergence in our experiments. 
In addition, we demonstrate the scalability of 
the proposed ACVB0 solution to larger relational data 
by employing the linear time inference algorithm. 
Based on these findings, we conclude that 
the ACVB0 inference solution on IRM is convenient and appealing for 
practitioners who work with relational data; 
it shows good modeling performance, assures automatic convergence, 
and is fast by linear time inference. 

The contributions of this paper are three-fold. 
\begin{enumerate}
 \item We first present Collapsed  Bayes solutions (CVB and CVB0)
  for inference of IRM, which is used for relational data analysis.  
  The CVB solutions are fast, precise and deterministic inference algorithms. 
  We also present update rules of hyperparameters, including the concentration parameter of the Dirichlet Process. 
 \item We empirically study the convergence behaviors of CVB
 solutions. Along with that, we propose a simple but effective
 annealing technique called Averaged CVB (ACVB) that assures the convergence of the CVB inference for any model. 
 \item We show techniques to speed up the (A)CVB inference.  
  In particular, one of them allows us to solve (A)CVB0 of IRM in linear time to the number of objects. This linear time algorithm is more effective when the relational data is sparse, which is typical in real datasets. 
\end{enumerate}

The rest of this paper is organized as follows. 
In the next section, we introduce the IRM model. 
In the third section, we briefly review collapsed Gibbs sampler
solutions of IRM and related works.  
We introduce a naive VB solution of IRM in the fourth section. 
As stated, the VB solution is not good for IRM, 
but it serves as a stepping stone for the CVB. 
Section 5 presents our CVB solutions.
Section 6 presents convergence issues, including convergence-assured
annealing technique (ACVB). 
Section 7 discusses the speeding-up technique, including linear time inference of (A)CVB0. 
The eighth section is devoted to experimental evaluations,  
and the final section concludes the paper. 

\section{Infinite Relational Models and Related Works}
First, we introduce the infinite relational model (IRM)~\citep{Kemp06}, 
which estimates the unknown number of hidden clusters within a set of relational
data. Then, we review some related works. 

\subsection{Dirichlet Process Mixture}
In the infinite relational model (IRM)~\citep{Kemp06}, 
the Dirichlet process (DP) is used as a prior for clusters of an unknown
number; intuitively, it is equivalent to an infinite-dimensional Dirichlet distribution.
The Dirichlet Process Mixture (DPM) is a probabilistic generative model 
that uses DP for the prior of mixture proportions. 
We can implement DPM by using either a Stick-Breaking
Process (SBP)~\citep{Sethuraman94} or a 
Chinese restaurant process (CRP)~\citep{Blackwell_MacQueen73}, 
which is a marginalized form of SBP. 
CRP is employed for the (collapsed) Gibbs sampler, 
and SBP is employed for (collapsed) variational Bayes solutions typically. 

First, let us start by explaining CRP. 
CRP is introduced as a probability distribution over a partitioning of $N$ objects. 
Let $z_{i} = k, i \in \{1, \dots, N\}, k \in \{1, \dots, K\}$ 
denote the assignment of $i$th object to the $k$th partition (cluster)
among the total of $K$ partitions. 
Then, the CRP is defined by the following equations: 
\begin{alignat}{2}
 \text{CRP}(z_{1:N} | \alpha) &= \alpha^{K} 
 \frac{
 \prod_{k=1}^{K} \left(m_{k} - 1\right)!
 }{
 \prod_{i=1}^{N}\left(\alpha+i-1\right)
 }, 
 \label{eq:CRP_def} \\
 p(z_{i}=k | z^{\setminus i}, \alpha) &=
 \begin{cases}
  \frac{m_{k}^{\setminus i}}{N-1+\alpha} & m^{\setminus i}_{k} > 0, \\
  \frac{\alpha}{N-1+\alpha} & m_{k}^{\setminus i} = 0 . 
 \end{cases}
 \label{eq:CRP_construction_def}
\end{alignat}
$\alpha > 0$ is a hyperparameter called a concentration parameter. 
Equation \ref{eq:CRP_def} shows the joint probability of $K$ partitions. 
The equation is rewritten as \myeqref{eq:CRP_construction_def}, 
which gives the probability of object $i$ being allocated to 
partition $k$ given the assignments of other objects. 
$m_{k}$ denotes the number of objects assigned to partition $k$, 
and $m_{k}^{\setminus i}$ denotes the same number excluding object $i$. 
The first part of \myeqref{eq:CRP_construction_def} indicates that 
object $i$ will be assigned to existing cluster $k$ with a probability
proportional to its membership. 
The second part of \myeqref{eq:CRP_construction_def} indicates that 
object $i$ will be assigned to a new cluster with a probability proportional to $\alpha > 0$. 
Repeatedly applying \myeqref{eq:CRP_construction_def} for each object, 
we can randomly generate cluster partitions of $N$ objects. 
For each run of CRP, we will have different clustering
of objects, and will also have a different number of resulting
clusters. 

% % %
% Stick-breaking
% % %
Next, we explain SBP. 
SBP is another construction of DPM. 
In SBP, we explicitly sample an infinite-dimensional vector of 
mixture proportions, while CRP directly samples cluster assignment $z$ 
without such a vector. 
SBP construction of DPM is described as follows: 
\begin{alignat}{2}
 v_{k} 
 &\sim 
 \text{Beta}\left(1, \alpha \right) \, ,
 \quad k = 1, \dots, \infty 
 \label{eq:stick_v_k_def} 
 \\
 \pi_{k} 
 &= 
 v_{k} \prod_{l=1}^{v-1} \left(1 - v_{l}\right) \, ,
 \quad k = 1, \dots, \infty
 \label{eq:stick_pi_k_def}
 \\
 z_{i} 
 &\sim 
 \text{Multinomial}\left(\bm{\pi} \right) \, .
 \quad i = 1, \dots, N
 \label{eq:stick_z_i_def}
\end{alignat}
Computing $\bm{\pi} = \left\{ \pi_{1}, \pi_{2}, \dots \right\}$, 
a mixing proportion vector of the mixture, 
is explained by an analogy of "stick breaking". 
We first assume a stick of length $1$, which is the full portion of clusters. 
Then, we break the stick into a stick of length $v_{1}$ (\myeqref{eq:stick_v_k_def}) 
and the remainder. 
The proportion of the first cluster $\pi_{1}$ is the length of this stick, that is,  
$v_{1}$ (\myeqref{eq:stick_pi_k_def}). 
Then, we again break the remaining stick, which has a length of $1 - v_{1}$, 
based on the ratio of $v_{2}$ (\myeqref{eq:stick_v_k_def}). 
The proportion of the second cluster $\pi_{2}$ is equivalent to the broken stick, 
that is,  $v_{2} \left(1 - v_{1}\right)$ (\myeqref{eq:stick_pi_k_def}).  
Repeating this process, we obtain the infinite-dimensional vector $\bm{\pi}$, whose sum 
equals $1$, implying an infinite number of mixture components. 
Using $\bm{\pi}$ as a mixing proportion, we can sample cluster assignments as in 
\myeqref{eq:stick_z_i_def}. 
%This SBP is a good representation to explain how DP(M) is able to represent 
%a countably infinite number of partitions (clusters): 
%the mixing proportion vector $\bm{\pi}$ is the infinite-dimensional vector. 

\subsection{Generative Models of IRMs} 
Next, we describe a probabilistic generative model of IRM. 
We assume two-place relations throughout the paper, but
extension to cover higher-order relations is straightforward. 

IRM is an application of DP for relational data. 
Let us first assume a binary two-place relation on the two sets (domains) of objects, 
namely $D_{1} \times D_{2} \rightarrow \{0, 1\}$, where
$D_{1} = \left\{ 1, \dots, i, \dots, N_{1}\right\}$ and 
$D_{2} = \left\{ 1, \dots, j, \dots, N_{2}\right\}$. 
IRM divides the set of objects into 
multiple clusters based on the observed relational data matrix of 
$\bm{X}=\{x_{i,j}\in\{0,1\}\}$. 
Data entry $x_{i,j} \in \{0, 1\}$ denotes 
the existence of 
a relation between a row (the first domain) object 
$i \in \{1, 2, \ldots, N_{1}\}$ 
and a column (the second domain) object 
$j \in \{1, 2, \ldots, N_{2}\}$. 
In an online purchase record case, 
the first domain corresponds to a user list, 
and an object $i$ denotes a specific user $i$. 
The second domain corresponds to a list of product items, 
and an object $j$ denotes a specific item $j$. 
The data entry $x_{i,j}$ represents a relation between the user $i$ and
the item $j$: namely, the purchase record. 

For such data, we define an IRM as follows: 
\begin{alignat}{2}
 \theta_{k,l} | a_{k,l}, b_{k,l} 
 &\sim
 \text{Beta}\left(a_{k,l}, b_{k,l}\right), 
 \label{eq:IRM_2dom_theta_kl_def}
 \\
 z_{1,i} | \alpha_{1} 
 &\sim 
 \text{CRP}\left( \alpha_{1} \right), 
 \label{eq:IRM_2dom_z_1i_def}
 \\
 z_{2,j} | \alpha_{2} 
 &\sim 
 \text{CRP}\left( \alpha_{2} \right), 
 \label{eq:IRM_2dom_z_2j_def}
 \\
 x_{i,j} | \bm{Z}_{1}, \bm{Z}_{2}, \{\theta\} 
 &\sim
 \text{Bernoulli}
 \left( 
 \theta_{z_{1,i}, z_{2,j}}
 \right).
 \label{eq:IRM_2dom_x_ij_def}
\end{alignat}
In \myeqref{eq:IRM_2dom_theta_kl_def}, $\theta_{k,l}$ is the strength of the relation between cluster $k$ of the first domain and cluster $l$ in the
second domain. 
$z_{1,i}$ in \myeqref{eq:IRM_2dom_z_1i_def} and 
$z_{2,j}$ in \myeqref{eq:IRM_2dom_z_2j_def} denote 
the cluster assignments in the first domain and in the second domain,
respectively. 
Throughout the paper, we interchangeably choose the 1-of-$K$
representation of $Z$: $z_{1,i} = k$ is equivalently represented by 
$z_{1,i,k} = 1, z_{1,i,l \neq k} = 0$. 
For each domain, we have a CRP prior; 
this indicates that each domain may have a different number of clusters. 
Generating the observed relational data $x_{i,j}$ follows
\myeqref{eq:IRM_2dom_x_ij_def} conditioned by the cluster assignments 
$\bm{Z}_{1} = \{z_{1,i}\}_{i=1}^{N_{1}}$, 
$\bm{Z}_{2} = \{z_{2,j}\}_{j=1}^{N_{2}}$ and
the strengths, $\theta$. 
A typical example of IRM is shown in \myfigref{fig:IRM_example}. 
IRM infers appropriate cluster assignments of objects 
$\bm{Z}_{1} = \{ z_{1,i}\}$ and $\bm{Z}_{2} = \{ z_{2,j} \}$ given 
the observation relation matrix $\bm{X} = \{ x_{i,j} \}$. 
We can interpret the clustering as the permutation of object indices so as
to discover the ``block'' structure as in \myfigref{fig:IRM_example} (b). 

\begin{figure*}[t]
 \begin{center}
  \includegraphics[width=120mm]{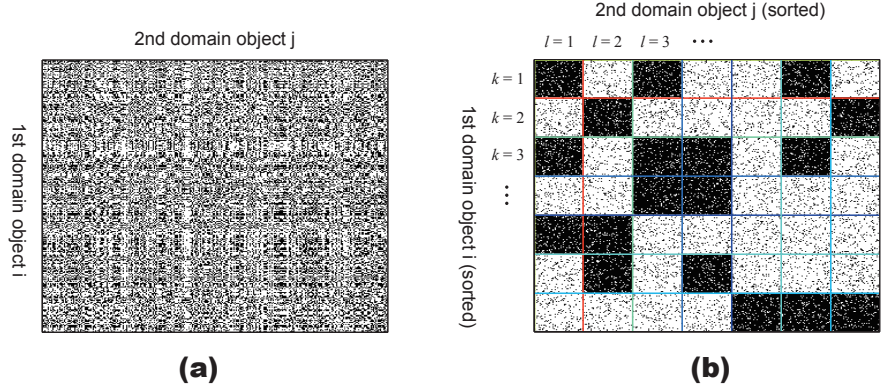}
  \caption{Example of Infinite Relational Models (IRM). (a) Input
  observation $\bm{X}$. (b) A visualization of inferred clusters $\bm{Z}$.  }
\label{fig:IRM_example}
 \end{center}
\end{figure*}

%% T1 x T1 
\subsubsection{Single-domain Model}
As a special case, we can build an IRM 
for a binary two-place relation between the same domain objects 
%$D = \left\{ 1, \dots, i, \dots, N\right\}$ as 
$D \times D \rightarrow \{0, 1\}$. 

The probabilistic generative model of the {\it single-domain} IRM is 
described as follows:  
\begin{alignat}{2}
 \theta_{k,l} | a_{k,l}, b_{k,l} 
 &\sim
 \text{Beta}\left(a_{k,l}, b_{k,l}\right), 
 \label{eq:IRM_theta_kl_def}
 \\
 z_{i} | \alpha 
 &\sim 
 \text{CRP}\left( \alpha \right), 
 \label{eq:IRM_z_i_def}
 \\
 x_{i,j} | \bm{Z}, \{\theta\} 
 &\sim
 \text{Bernoulli}
 \left( 
 \theta_{z_{i}, z_{j}}
 \right). 
 \label{eq:IRM_x_ij_def}
\end{alignat}
The generative model clearly shows the difference of the 
{\it multi-domain} IRM~(Eqs. (\ref{eq:IRM_2dom_theta_kl_def}-\ref{eq:IRM_2dom_x_ij_def}))
and the single-domain IRM~(Eqs. (\ref{eq:IRM_theta_kl_def}-\ref{eq:IRM_x_ij_def})). 
In the single-domain IRM, 
there are only $N$ objects in the domain $D$, 
and they serve as either from-nodes or to-nodes 
in the network. 
Object indices $i$ and $j$ point to the same domain. 
On the other hand, the multi-domain IRM distinguishes 
the first domain object $i$ and the second domain object $j$. 

Let us explain the difference by a simple SNS example. 
Imagine we are given the SNS relation data where 
$N$ users are mutually interconnected to others. 
An observed relation is a binary value $x_{i,j}$, which indicates 
there is a link from user $i$ to user $j$. 
In the case of (multi-domain) IRM, the first domain is a collection of
$N$ users who act as from-nodes. The second domain is a collection of 
$N$ users who act as to-nodes. 
We assume that each user has a different ``role'' in 
reaching a link to others (the first domain) and 
accepting a link from others (the second domain). 
Contrarily, in the case of the single-domain IRM, 
each user $i$ is assigned with a single cluster assignment. 
The user has her own ``role'' in the network, and this role 
is used in either reaching a link or accepting a link. 

Obviously, the single-domain IRM is not applicable when the 
number of from-nodes and to-nodes are different. 
Thus, the model has a specific and limited applicability compared to the 
multi-domain IRM. 
Afterward, we focus on the multi-domain IRM in this paper, 
but whole discussions are also valid for single-domain IRM. 

\subsection{Related Works}

% % %
% stochastic block models: IRM studies: dIRM, Subuset
% % %
IRM~\citep{Kemp06} is a rather old model for formulating general relationship observations. 
One drawback of IRM is that it allows a node to have only one cluster assignment. 
Mixed Membership Stochastic Blockmodel (MMSB)~\citep{Airoldi08} is a finite-cluster model that allows the nodes 
to have multiple cluster assignment, and change the clusters edge by edge. 
\citep{Miller09} employs the Indian Buffet Process (IBP, see ~\citep{Griffiths_Ghahramani11} for a review) to 
handle countably infinite binary factors for each node. 
The Infinite Latent Attribute Model~\citep{Palla12ICML} further allows flexible modeling of networks where 
an infinitely many ``views'' have their own clustering of nodes. 
Compared to these models, IRM is inferior in the potential modeling capability but is superior in 
easier interpretation of the clustering results: 
after all, ``multiple clusters for a single node'' is somewhat counter-intuitive for non-expert users. 
Of course, IRM is a simpler model; thus the inference scheme is also much simpler than these advanced models. 

% % %
% Network studies: mixed memberships, Infinite latent attributes, etc 
% % %
Recently, probabilistic models for pure networks have attracted much attention in machine learning. 
By ``pure'' networks, we mean that 
observed relations are limited to the single-domain case: that is,  $T \times T \rightarrow \{1, 0\}$. 
%Because of the rapid growth of several social network services, 
% Ho11AISTATS - hierarhcy
\citep{Ho11} introduces a nested Chinese Restaurant Process (nCRP)~\citep{Blei10} to 
incorporate multiscale membership for the MMSB. 
% Yin13NIPS - triangular
\citep{Ho12} proposed a bag of triangular representations of a network. 
The representation is based on the triplet of nodes. Possible connections among three nodes are 
(i) all three nodes are connected in a circuit, 
(ii) all three nodes are connected in a line (no link between a specific pair of nodes), 
(iii) two nodes are connected and one is isolated, 
and 
(iv) all nodes are separated. 
\citep{Yin13} combines the triangular representations with a simpler probabilistic generative model to 
achieve a scalable algorithm for large networks, which limits the cardinality of the triangle ``cluster assignments'' variety 
in the likelihood function. 
% Yang13ICDM: 
\citep{Yang13} employ the edge structure and node attributes to 
find communities within large networks. The model is called CESNA, 
consisting of a soft-max-based binary node attribute model and 
an affiliated network model~\citep{Yang_Leskovec13}. 
These recent works make the model scalable against very large networks consisting of 
millions of nodes. 
However, none of these works consider the cross-domain $T_{1} \times T_{2}$ relational observations 
that are mainly discussed in this IRM paper. 

% % %
% CVB works
% % %
As briefly described in the introduction, 
collapsed variational Bayes (CVB) solutions have been intensively studied, especially for 
topic models such as latent Dirichlet allocation (LDA)~\citep{Teh07NIPS,Asuncion09,Sato_Nakagawa12} 
and HDP-LDA~\citep{Sato12}. 
\citep{Hensman12} introduced a different view of CVB in a wide scope of exponential families. 
Only a few researchers have examined CVB in 
models for structured data such as Probabilistic Context-Free Grammars (PCFG)~\citep{Wang_Blunsom13ACL} and Hidden Markov
Models (HMM)~\citep{Wang_Blunsom13}. 
As stated, CVB solutions of IRM are first introduced to the literature 
to the best of our knowledge.  

%%%%%%%%%%%%%%%%%
% Gibbs solution
%%%%%%%%%%%%%%%%%
\section{Collapsed Gibbs Sampler Solution}
Before deriving CVB solutions, we review the collapsed Gibbs sampler 
here to facilitate the derivation of CVB solutions.  
Let us define the counting statistics that are maintained during sampling. 
$n_{k,l}$ and $N_{k,l}$ %in \myeqref{eq:IRM_2dom_Gibbs_n_N_kl_def} 
denote the number of positive ($x=1$) and
negative ($x=0$) relation observations in the $(k,l)$-cluster pairs, respectively. 
$m_{k}$ is the same quantity as used in 
Eqs. (\ref{eq:CRP_def}, \ref{eq:CRP_construction_def}). 
\begin{alignat}{2}
 n_{k,l} 
 &=
 \sum_{i} \sum_{j} z_{1,i,k} z_{2,j,l} x_{i,j}
 \, , \quad 
 N_{k,l} 
 =
 \sum_{i} \sum_{j} z_{1,i,k} z_{2,j,l} \left(1 - x_{i,j}\right) \, , 
 \label{eq:IRM_2dom_Gibbs_n_N_kl_def} 
 \\
 m_{1,k} 
 &=
 \sum_{i} z_{1,i,k}
 \, , \quad
 m_{2,l} 
 =
 \sum_{j} z_{2,j,l} \, .
 \label{eq:IRM_2dom_Gibbs_m_1k_2l_def} 
\end{alignat}

%%%%%%%%%%%%%%% 
% 1st domain
%%%%%%%%%%%%%%%
\subsection{Sampling $z_{1,i}$}
We review the collapsed Gibbs solution for inference of 
$\bm{Z}_{1} = \{z_{1,i}\}$: the solution is completely symmetric for the second domain, $\bm{Z}_{2}$. 

In the Gibbs sampler approach, 
we repeat the following process. 
First, we select one object $(1, i)$ (or $(2, j)$) from the data and 
take the object out from the model. More specifically, 
a clustering assignment of the object $z_{1,i}$ is temporarily set empty (undefined).   
Then, we reassign (sample) 
$z_{1,i} = k$ based on the posterior $p\left( z_{1,i} = k \right)$. 
%given all remaining cluster assignments in the first domain $\bm{Z}_{1}$, 
%all cluster assignments in the second domain $\bm{Z}_{2}$, and the whole observation. 

To start, we divide the observations into two parts. 
Let us denote $\bm{X}^{(1,i)} = \{ x_{i, \cdot} \}$ 
as the set of all observations 
concerning object $i$ of the first domain. 
The remaining observations, hidden variables excluding $z_{1,i}$, 
and statistics computed on these data are denoted by $\setminus (1,i)$. 
Our target posterior is formulated as follows: 
\begin{multline}
 p\left( 
 z_{1,i} = k | \bm{X}, \bm{Z}_{1}^{\setminus (1,i)}, \bm{Z}_{2} 
 \right) 
 \propto
 p\left( z_{1,i} | \bm{Z}_{1}^{\setminus (1,i)} \right) 
 p\left( 
 \bm{X} | z_{1,i} = k, \bm{Z}_{1}^{\setminus (1,i)}, \bm{Z}_{2} 
 \right) 
 \\
 =
 p\left( 
 z_{1,i} = k | \bm{Z}_{1}^{\setminus (1,i)} 
 \right) 
 \int
 p\left( 
 \bm{X}^{(1,i)} | z_{1,i} = k, \bm{Z}_{1}^{\setminus (1,i)}, 
 \bm{Z}_{2}, \bm{\Theta} 
 \right) 
 p\left(
 \bm{\Theta} | \bm{Z}_{1}^{\setminus (1,i)}, 
 \bm{Z}_{2}, \bm{X}^{\setminus (1,i)} \right)
 d \bm{\Theta} \, .
 \label{eq:IRM_2dom_Gibbs_z_1_def}
\end{multline}

The first term of the right-hand side of 
\myeqref{eq:IRM_2dom_Gibbs_z_1_def} becomes: 
\begin{equation}
 p\left( z_{1,i} = k | \bm{Z}_{1}^{\setminus (1,i)} \right) 
  \propto
  \begin{cases}
   m_{1,k}^{\setminus (1,i)} & \text{existing clusters}, \\ 
   \alpha_{1} & \text{a new cluster}, \\ 
  \end{cases}
  \label{eq:IRM_2dom_Gibbs_z_1_1}
\end{equation}
as in \myeqref{eq:CRP_construction_def}. 

Thanks to the conjugacy of 
\myeqref{eq:IRM_2dom_theta_kl_def} and \myeqref{eq:IRM_2dom_x_ij_def}, 
we can easily evaluate the second integral term of the r.h.s. of 
\myeqref{eq:IRM_2dom_Gibbs_z_1_def}. 
First, Eq. (\ref{eq:IRM_2dom_Gibbs_z_1_2_1}) is the 
posterior of the strength parameters $\bm{\Theta}$ 
given all information excepting the object $i$ in the first domain. 
\begin{alignat}{2}
 p\left(\bm{\Theta} 
 | \bm{Z}_{1}^{\setminus (1,i)}, \bm{Z}_{2}, \bm{X}^{\setminus (1,i)}\right)
 &= 
 \prod_{k} \prod_{l} 
 \text{Beta}
 \left( \theta_{k,l}; 
 a_{k,l} + n^{\setminus (1,i)}_{k,l}, 
 b_{k,l} + N^{\setminus (1,i)}_{k,l}
 \right) \, ,
 \label{eq:IRM_2dom_Gibbs_z_1_2_1} 
 \\
 p\left( \bm{X}^{(1,i)} | z_{1,i} = k, \bm{Z}_{1}^{\setminus (1,i)}, \bm{Z}_{2}, \bm{\Theta} \right) 
 &= 
 \prod_{l} 
 \theta_{k,l}^{n_{k,l}^{+(1,i,k)}}
 \left(1 - \theta_{k,l}\right)^{N_{k,l}^{+(1,i,k)}}
 \, .
 \label{eq:IRM_2dom_Gibbs_z_1_2_2}
\end{alignat}
$n^{+(1,i,k)}$ and $N^{+(1,i,k)}$ denote the statistics computed solely 
on $\bm{X}^{(1,i)}$ given $z_{1,i} = k$. 

Combining \myeqref{eq:IRM_2dom_Gibbs_z_1_2_1} and 
\myeqref{eq:IRM_2dom_Gibbs_z_1_2_2}, 
we obtain the following. 
\begin{multline}
 \int
 p\left( \bm{X}^{(1,i)} | z_{1,i} = k, \bm{Z}_{1}^{\setminus (1,i)}, \bm{Z}_{2}, 
 \bm{\Theta} \right) 
 p\left(\bm{\Theta} 
 | \bm{Z}_{1}^{\setminus (1,i)}, \bm{Z}_{2}, \bm{X}^{\setminus (1,i)}\right)
 d \bm{\Theta} 
 \\
 \propto
 \prod_{l} 
  \frac{
 B\left( 
   a_{k,l} + n^{\setminus (1,i)}_{k,l} + n_{k,l}^{+(1,i,k)}, 
   b_{k,l} + N^{\setminus (1,i)}_{k,l} + N_{k,l}^{+(1,i,k)}
  \right)
 }{
 B\left( 
   a_{k,l} + n^{\setminus (1,i)}_{k,l}, 
   b_{k,l} + N^{\setminus (1,i)}_{k,l}
  \right)
 } \, ,
 \label{eq:IRM_2dom_Gibbs_z_1_2}  
\end{multline}
where $B(\cdot,\cdot)$ denotes the beta function. 

Plugging \myeqref{eq:IRM_2dom_Gibbs_z_1_1} and
\myeqref{eq:IRM_2dom_Gibbs_z_1_2} into \myeqref{eq:IRM_2dom_Gibbs_z_1_def} yields 
the posterior probability for sampling the cluster assignment of object $i$ in the
first domain: 
\begin{equation}
 p\left( z_{1,i} = k | \bm{X}, \bm{Z}_{1}^{\setminus (1,i)}, \bm{Z}_{2} \right) 
 \propto 
 \begin{cases}
  m_{1,k}^{\setminus (1,i)}
  \prod_{l} 
  \frac{
  B\left( 
  a_{k,l} + n^{\setminus (1,i)}_{k,l} + n_{k,l}^{+(1,i,k)}, 
  b_{k,l} + N^{\setminus (1,i)}_{k,l} + N_{k,l}^{+(1,i,k)}
  \right)
  }{
  B\left( 
  a_{k,l} + n^{\setminus (1,i)}_{k,l}, 
  b_{k,l} + N^{\setminus (1,i)}_{k,l}
  \right)
  } 
  & 
  k~\text{is an existing cluster} \, ,
  \\
  \alpha_{1} 
  \prod_{l} 
  \frac{
  B\left( 
  a_{k,l} + n^{\setminus (1,i)}_{k,l} + n_{k,l}^{+(1,i,k)}, 
  b_{k,l} + N^{\setminus (1,i)}_{k,l} + N_{k,l}^{+(1,i,k)}
  \right)
  }{
  B\left( 
  a_{k,l} + n^{\setminus (1,i)}_{k,l}, 
  b_{k,l} + N^{\setminus (1,i)}_{k,l}
  \right)
  } 
  & 
  k~\text{is a new cluster} \, .
 \end{cases}
 \label{eq:IRM_2dom_Gibbs_z_1i_posterior}
\end{equation}

We iteratively take out $z_{1,i}$ from the statistics, 
compute the posterior and sample $z_{1,i}$ stochastically, 
then put $z_{1,i}$ back into the statistics.  

%%%%%%%%%%%%%%% 
% 2nd domain
%%%%%%%%%%%%%%%
\subsection{Sampling $z_{2,j}$}

We also present the final result for the second domain. 
The derivation is symmetric to $\bm{Z}_{1}$; 
thus, we omit details. 
Let us denote $\bm{X}^{(2,j)} = \{ x_{\cdot, j} \}$ 
as the set of all observations 
concerning the object $j$ of the second domain. 
The remaining observations, hidden variables, 
and statistics computed on these data are denoted by 
$\setminus (2,j)$. 

Our goal is to sample $z_{2,j}$ based on the following equation: 
\begin{multline}
 p\left( z_{2,j} = l | \bm{X}, \bm{Z}_{1}, \bm{Z}_{2}^{\setminus (2,j)} \right) 
 \\ 
 \propto
 p\left( z_{2,j} = l | \bm{Z}_{2}^{\setminus (2,j)} \right) 
 \int
 p\left( \bm{X}^{(2,j)} | z_{2,j} = l, \bm{Z}_{1}, 
 \bm{Z}_{2}^{\setminus (2,j)}, \bm{\Theta} \right) 
 p\left(\bm{\Theta} 
 | \bm{Z}_{1}, \bm{Z}_{2}^{\setminus (2,j)}, \bm{X}^{\setminus (2,j)}\right)
 d \bm{\Theta} \,. 
 \label{eq:IRM_2dom_Gibbs_z_2_def} 
\end{multline}
Each term in the right-hand side of 
\myeqref{eq:IRM_2dom_Gibbs_z_2_def} is calculated as follows: 
\begin{equation}
 p\left( z_{2,j} = l | \bm{Z}_{2}^{\setminus (2,j)} \right) 
  \propto
  \begin{cases}
   m_{2,l}^{\setminus (2,j)} & \text{existing clusters}, \\ 
   \alpha_{2} & \text{a new cluster}, \\ 
  \end{cases}
  \label{eq:IRM_2dom_Gibbs_z_2_1}
\end{equation}
\begin{multline}
 \int
 p\left( \bm{X}^{(2,j)} | z_{2,j} = l, \bm{Z}_{1}, \bm{Z}_{2}^{\setminus (2,j)}, 
 \bm{\Theta} \right) 
 p\left(\bm{\Theta} 
 | \bm{Z}_{1}, \bm{Z}_{2}^{\setminus (2,j)}, \bm{X}^{\setminus (2,j)}\right)
 d \bm{\Theta}
 \\
 \propto
 \prod_{l} 
  \frac{
 B\left( 
   a_{k,l} + n^{\setminus (2,j)}_{k,l} + n_{k,l}^{+(2,j,l)}, 
   b_{k,l} + N^{\setminus (2,j)}_{k,l} + N_{k,l}^{+(2,j,l)}
  \right)
 }{
 B\left( 
   a_{k,l} + n^{\setminus (2,j)}_{k,l}, 
   b_{k,l} + N^{\setminus (2,j)}_{k,l}
  \right)
 } \, . 
 \label{eq:IRM_2dom_Gibbs_z_2_2}  
\end{multline}
In the above equation, 
$n^{+(2,j,l)}$ and $N^{+(2,j,l)}$ denote the statistics computed solely 
on $\bm{X}^{(2,j)}$ given $z_{2,j} = l$. 
Thus, we obtain the Gibbs posterior probability for $z_{2,j}$ as follows: 
\begin{equation}
 p\left( z_{2,j} = l | \bm{X}, \bm{Z}_{1}, \bm{Z}_{2}^{\setminus (2,j)} \right) 
 \propto 
 \begin{cases}
  m_{2,l}^{\setminus (2,j)}
  \prod_{k} 
  \frac{
  B\left( 
  a_{k,l} + n^{\setminus (1,i)}_{k,l} + n_{k,l}^{+(2,j,l)}, 
  b_{k,l} + N^{\setminus (1,i)}_{k,l} + N_{k,l}^{+(2,j,l)}
  \right)
  }{
  B\left( 
  a_{k,l} + n^{\setminus (1,i)}_{k,l}, 
  b_{k,l} + N^{\setminus (1,i)}_{k,l}
  \right)
  } 
  & 
  l~\text{is an existing cluster} \, ,
  \\
  \alpha_{2} 
  \prod_{k} 
  \frac{
  B\left( 
  a_{k,l} + n^{\setminus (1,i)}_{k,l} + n_{k,l}^{+(2,j,l)}, 
  b_{k,l} + N^{\setminus (1,i)}_{k,l} + N_{k,l}^{+(2,j,l)}
  \right)
  }{
  B\left( 
  a_{k,l} + n^{\setminus (1,i)}_{k,l}, 
  b_{k,l} + N^{\setminus (1,i)}_{k,l}
  \right)
  } 
  & 
  l~\text{is a new cluster} \, .
 \end{cases}
 \label{eq:IRM_2dom_Gibbs_z_2j_posterior}
\end{equation}%
In practice, it is better to sample new assignments of objects in 
a domain-interleaving manner (
$z_{1,i} \rightarrow z_{2,j} \rightarrow z_{1,i^{\prime}} \rightarrow z_{2,j^{\prime}} \rightarrow \dots $
)
for faster convergence. 

We can also sample hyperparameters and parameters by putting in hyper
priors, or by solving marginal likelihood maximization. 
We omit these procedures because they are out of our scope. 

One difficulty in employing a Gibbs sampler is detection of convergence. 
A Gibbs sampler assures asymptotic convergence to the true posteriors 
as the number of samples is infinitely many. 
In practice, however, we will never have an infinite number of samples, 
so it is difficult to detect convergence 
in a theoretically valid manner~\citep{Cowles_Carlin96}. 

Another difficulty is the very slow mixing nature of collapsed Gibbs sampler on IRM, 
which has been recently reported by \citep{Albers13}. 
They showed that the several million iterations (sweeps)  \textit{are not enough} 
to mix the sampler, on $1000$-nodes, real-world network data. 
One possible reason is that one observed relation $x_{i,j}$ requires two hidden variables $z_{1,i}$ and $z_{2,j}$, 
unlike topic models. 
To alleviate the slow mixing, we need to introduce much more sophisticated samplers such as \citep{Williamson13}. 
But such techniques would make it difficult to implement the sampler. 

These two reasons motivate us to develop deterministic and fast VB-based inference solutions, 
though most of the existing IRM works rely on collapsed Gibbs sampler. 

%%%%%%%%%%%%%%%%%
% VB solution
%%%%%%%%%%%%%%%%%
\section{Variational Bayes Solution}
No report to date has described a variational Bayes (VB) solution for 
IRM. However, it is beneficial to quickly derive a VB solution for 
comparison with the proposed collapsed VB inference. 
In a general VB inference, we maximize the VB lower bound, 
which is defined as: 
\begin{equation}
 \mathscr{L} = 
  \int q\left(\bm{Z}_{1}, \bm{Z}_{2}, \bm{\Phi}\right) 
  \log \frac{
  p\left(\bm{X}, \bm{Z}_{1}, \bm{Z}_{2}, \bm{\Phi}\right)
  }{
  q\left(\bm{Z}_{1}, \bm{Z}_{2}, \bm{\phi}\right)
  }
  d \bm{Z}_{1} d \bm{Z}_{2} d \bm{\Phi}\, ,
  \label{eq:VB_lowerbound_def}
\end{equation}
where $\bm{Z}_{1}$ and $\bm{Z}_{2}$ denote hidden variables, 
$\bm{\Phi}$ denotes all associated parameters, 
$\bm{X}$ denotes all observations, 
and $q(\cdot)$ are the {\it variational} posteriors 
that approximate the true posteriors; 
all variational posteriors are assumed independent from each other. 
Maximizing the above lower bound is equivalent with minimizing the
Kullback-Leibler divergence between the true posteriors $p$ and the 
{\it variational} posteriors $q$. 

Generally speaking, the VB solution is analogous to the iterative
process of the EM algorithm. First, we maximize the VB lower bound w.r.t. 
the variational posteriors of hidden variables. Then, 
we maximize the VB lower bound w.r.t. remaining parameters. 
This iteration monotonically increases the VB lower bound
in~\myeqref{eq:VB_lowerbound_def}; therefore, the VB solution 
halts automatically when it reaches a local optimal point. 

\subsection{Generative models}
We alter the generative model of IRM in two points. 
First, we use a Stick-Breaking Process (SBP)
construction~\citep{Sethuraman94} 
of the DPM~(Eqs. (\ref{eq:stick_v_k_def}, \ref{eq:stick_pi_k_def}, \ref{eq:stick_z_i_def})), 
instead of CRP~(Eqs. (\ref{eq:CRP_def}, \ref{eq:CRP_construction_def})). 
Second, we ``truncate'' the maximum number of clusters, $K$, beforehand. 
Therefore, the VB solution of IRM is doubly approximated: 
the independence of variational posteriors and the finite number of
clusters. 
Fixing the number of clusters seems to destroy the virtue
of nonparametric Bayes: automatic model selection. 
In practice, the SBP prior makes the unrepresented (unnecessary) 
clusters very small (very small weights) 
after inference. Therefore, it is easy to infer the true
number of clusters even if we ``truncate'' the infinite
 cluster representation. 

Here is the generative model of IRM for VB: 
\begin{alignat}{2}
 v_{1,k} | \alpha_{1} 
 &\sim
 \text{Beta}\left(1, \alpha_{1}\right) 
 \, ,
 \label{eq:IRM_2dom_VB_v_1k_def}
 \\
 \pi_{1,k}
 &= 
 v_{1,k} \prod_{m=1}^{k-1} \left(1 - v_{1,m}\right), 
 \pi_{1,K_{1}} = 1 - \sum_{m=1}^{K_{1}-1} \pi_{1,m} 
 \, ,
 \label{eq:IRM_2dom_VB_pi_1k_def} 
 \\
 z_{1,i} | \bm{\pi}_{1}
 &\sim 
 \text{Multinomial}\left( \bm{\pi}_{1} \right) 
 \, ,
 \label{eq:IRM_2dom_VB_z_1i_def}
 \\
 v_{2,l} | \alpha_{2} 
 &\sim
 \text{Beta}\left(1, \alpha_{2}\right)
 \, ,
 \label{eq:IRM_2dom_VB_v_2l_def}
 \\
 \pi_{2,l}
 &= 
 v_{2,l} \prod_{m=1}^{l-1} \left(1 - v_{2,l}\right), 
 \pi_{2,K_{2}} = 1 - \sum_{m=1}^{K_{2}-1} \pi_{2,m} 
 \, ,
 \label{eq:IRM_2dom_VB_pi_2l_def} 
 \\
 z_{2,j} | \bm{\pi}_{2}
 &\sim 
 \text{Multinomial}\left( \bm{\pi}_{2} \right) 
 \, ,
 \label{eq:IRM_2dom_VB_z_2j_def}
 \\
 \theta_{k,l} | a_{k,l}, b_{k,l} 
 &\sim
 \text{Beta}\left(a_{k,l}, b_{k,l}\right) 
 \, ,
 \label{eq:IRM_2dom_VB_theta_kl_def}
 \\
 x_{i,j} | \bm{Z}_{1}, \bm{Z}_{2}, \{\theta\} 
 &\sim
 \text{Bernoulli}
 \left( 
 \theta_{z_{1,i}, z_{2,j}}
 \right)
 \, .
 \label{eq:IRM_2dom_VB_x_ij_def}
\end{alignat}
The stick-breaking process is described in
Eqs. (\ref{eq:IRM_2dom_VB_v_1k_def}-\ref{eq:IRM_2dom_VB_z_2j_def}). 
There are mainly two different parts compared to the original
generative models 
(\myeqref{eq:IRM_2dom_z_1i_def} and \myeqref{eq:IRM_2dom_z_2j_def}) 
formalized by CRP. 
The first one is the introduction of cluster mixing proportional vectors 
$\bm{\pi}_{1}$ and $\bm{\pi}_{2}$. 
The second one is the truncated cluster numbers: 
$K_{1}$ and $K_{2}$ indicate the maximum number of truncated
clusters for the first domain and the second domain, respectively. 
Because the numbers of clusters are finite, 
we simply sample the cluster assignment variables $\bm{Z}_{1}$ and
$\bm{Z}_{2}$ from the multinomial distributions as in 
\myeqref{eq:IRM_2dom_VB_z_1i_def} and
\myeqref{eq:IRM_2dom_VB_z_2j_def}. 

\subsection{Variational posteriors and the lower bound}
Given the generative models, 
we introduce variational posteriors that are assumed independent from each other. 
Thanks to the model conjugacy, we can specify the form of 
the variational posteriors, which are denoted by $q(\cdot)$ below: 
\begin{alignat}{2}
 q\left(\bm{v}_{1}; \hat{\alpha}_{1}, \hat{\beta}_{1} \right)
 &= 
 \prod_{k=1}^{K_{1}} \text{Beta}\left(v_{1,k}; \hat{\alpha}_{1,k},
 \hat{\beta}_{1,k}\right) 
 \, ,
 \label{eq:IRM_2dom_VB_v_1_q}
 \\
 q\left(\bm{Z}_{1}; \hat{\bm{\pi}}_{1}\right) 
 &= 
 \prod_{i=1}^{N_{1}} \text{Multinomial}\left(z_{1,i};
 \hat{\bm{\pi}}_{1,i}\right) 
 \, ,
 \label{eq:IRM_2dom_VB_z_1_q}  
 \\
 q\left(\bm{v}_{2}; \hat{\alpha}_{2}, \hat{\beta}_{2} \right)
 &= 
 \prod_{l=1}^{K_{2}} \text{Beta}\left(v_{2,l}; \hat{\alpha}_{2,l},
 \hat{\beta}_{2,l}\right) 
 \, ,
 \label{eq:IRM_2dom_VB_v_2_q}
 \\
 q\left(\bm{Z}_{2}; \hat{\bm{\pi}}_{2}\right) 
 &= 
 \prod_{j=1}^{N_{2}} \text{Multinomial}\left(z_{2,j};
 \hat{\bm{\pi}}_{2,j}\right) 
  \, ,
 \label{eq:IRM_2dom_VB_z_2_q}  
 \\
 q\left( \bm{\Theta} ; \hat{a}, \hat{b}\right)
 &= 
 \prod_{k=1}^{K_{1}} \prod_{l=1}^{K_{2}} 
 \text{Beta}\left(\theta_{k,l}; \hat{a}_{k,l}, \hat{b}_{k,l}\right) 
 \, .
 \label{eq:IRM_2dom_VB_theta_q}
\end{alignat}

Following the definition, 
we obtain the following VB lower bound: 
\begin{alignat}{2}
 \mathscr{L} 
 &=
 \iint 
 q\left(\bm{Z}_{1}, \bm{Z}_{2}, \bm{v}_{1}, \bm{v}_{2},
 \bm{\Theta}\right) 
 \log \frac{
 \log p\left(\bm{X}, \bm{Z}_{1}, \bm{Z}_{2}, \bm{v}_{1}, \bm{v}_{2},
 \bm{\Theta} \right) 
 }{
 \log  q\left(\bm{Z}_{1}, \bm{Z}_{2}, \bm{v}_{1}, \bm{v}_{2},
 \bm{\Theta} \right) 
 }
 d \bm{Z}_{1} d \bm{Z}_{2} d \bm{v}_{1} d \bm{v}_{2} d \bm{\Theta} d \alpha
 \notag 
 \\
 &=
 \mathbb{E}_{ \bm{Z}_{1}, \bm{Z}_{2}, \bm{\Theta} }\left[
 \log p\left(\bm{X} | \bm{Z}_{1}, \bm{Z}_{2}, \bm{\Theta}\right)
 \right] 
 \label{eq:IRM_2dom_VB_lowerbound_p_x} 
 \\
 &+
 \mathbb{E}_{ \bm{Z}_{1}, \bm{v}_{1} }\left[
 \log p\left(\bm{Z}_{1} | \bm{v}_{1}\right)
 \right]
 \label{eq:IRM_2dom_VB_lowerbound_p_z_1} 
 \\
 &+
 \mathbb{E}_{ \bm{Z}_{2}, \bm{v}_{2} }\left[
 \log p\left(\bm{Z}_{2} | \bm{v}_{2}\right)
 \right]
 \label{eq:IRM_2dom_VB_lowerbound_p_z_2} 
 \\
 &+
 \mathbb{E}_{ \bm{v}_{1} }\left[
 \log p\left(\bm{v}_{1} | \alpha_{1}\right)
 \right]
 \label{eq:IRM_2dom_VB_lowerbound_p_v_1} 
 \\
 &+
 \mathbb{E}_{ \bm{v}_{2} }\left[
 \log p\left(\bm{v}_{2} | \alpha_{2}\right)
 \right]
 \label{eq:IRM_2dom_VB_lowerbound_p_v_2} 
 \\
 &+
 \mathbb{E}_{ \bm{\Theta} }\left[
 \log p\left(\bm{\Theta} \right)
 \right]
 \label{eq:IRM_2dom_VB_lowerbound_p_theta} 
 \\
 &-
 \mathbb{E}_{ \bm{Z}_{1} }\left[
 \log q\left(\bm{Z}_{1}\right)
 \right]
 \label{eq:IRM_2dom_VB_lowerbound_q_z_1} 
 \\
 &-
 \mathbb{E}_{ \bm{Z}_{2} }\left[
 \log q\left(\bm{Z}_{2}\right)
 \right]
 \label{eq:IRM_2dom_VB_lowerbound_q_z_2} 
 \\
 &-
 \mathbb{E}_{ \bm{v}_{1} }\left[
 \log q\left(\bm{v}_{1}\right)
 \right]
 \label{eq:IRM_2dom_VB_lowerbound_q_v_1} 
 \\
 &-
 \mathbb{E}_{ \bm{v}_{2} }\left[
 \log q\left(\bm{v}_{2}\right)
 \right]
 \label{eq:IRM_2dom_VB_lowerbound_q_v_2} 
 \\
 &-
 \mathbb{E}_{ \bm{\Theta} }\left[
 \log q\left(\bm{\Theta}\right)
 \right] 
 \, .
 \label{eq:IRM_2dom_VB_lowerbound_q_theta} 
\end{alignat}
In the above equations, 
$\mathbb{E}_{x}$ indicates the expectation of the predicate 
computed over the variational posterior of $x$. 

\subsection{Variational posteriors of $\bm{Z}$}
We obtain the VB solution of the IRM by 
taking the derivative of the lower bound with respect to the 
variational posterior parameters in 
Eqs. (\ref{eq:IRM_2dom_VB_v_1_q}-\ref{eq:IRM_2dom_VB_theta_q}). 
Since the naive VB solution is not the primary interest of this paper, 
we omit the derivations and simply present the final results. 

For the VB E-step, we compute the variational posteriors of 
hidden cluster assignment variables $\bm{Z}_{1}$ and $\bm{Z}_{2}$. 
The variational posterior parameters in 
\myeqref{eq:IRM_2dom_VB_z_1_q} 
and 
\myeqref{eq:IRM_2dom_VB_z_2_q} 
are shown below.  

\begin{alignat}{2}
 \log \hat{\pi}_{1,i,k} 
 &= 
 \psi \left( \hat{\alpha}_{1,k} \right) 
 + 
 \sum_{m=1}^{k-1} \psi \left( \hat{\beta}_{1,m} \right)
 -
 \sum_{m=1}^{k} \psi \left( \hat{\alpha}_{1,m} + \hat{\beta}_{1,m} \right)
 \notag 
 \\
 &- 
 \sum_{j=1}^{N_{2}} \sum_{l=1}^{K_{2}} 
 \hat{\pi}_{2,j,l} \psi \left( \hat{a}_{k,l} + \hat{b}_{k,l} \right)
 \notag
 \\
 &+ 
 \sum_{j=1}^{N_{2}} \sum_{l=1}^{K_{2}} 
 \hat{\pi}_{2,j,l} \left[ 
 x_{i,j} \psi \left( \hat{a}_{k,l} \right)
 +
 \left( 1 - x_{i,j} \right) \psi \left( \hat{b}_{k,l} \right)
 \right] 
 + \text{Const.}
 \label{eq:IRM_2dom_VB_posterior_pihat_1ik}
\end{alignat}
\begin{alignat}{2}
 \log \hat{\pi}_{2,j,l} 
 &= 
 \psi \left( \hat{\alpha}_{2,l} \right) 
 + 
 \sum_{m=1}^{l-1} \psi \left( \hat{\beta}_{2,m} \right)
 -
 \sum_{m=1}^{l} \psi \left( \hat{\alpha}_{2,m} + \hat{\beta}_{2,m} \right)
 \notag 
 \\
 &- 
 \sum_{i=1}^{N_{1}} \sum_{k=1}^{K_{1}} 
 \hat{\pi}_{1,i,k} \psi \left( \hat{a}_{k,l} + \hat{b}_{k,l} \right)
 \notag
 \\
 &+ 
 \sum_{i=1}^{N_{1}} \sum_{k=1}^{K_{1}} 
 \hat{\pi}_{1,i,k} \left[ 
 x_{i,j} \psi \left( \hat{a}_{k,l} \right)
 +
 \left( 1 - x_{i,j} \right) \psi \left( \hat{b}_{k,l} \right)
 \right]
 + \text{Const.}
 \label{eq:IRM_2dom_VB_posterior_pihat_2jl  }
\end{alignat}
In the above equations, $\psi$ indicates the digamma function. 
$\pi_{1,i,\cdot}$ and $\pi_{2,j,\cdot}$ are to be normalized so as to
make the sum over cluster indices equal to 1. 

\subsection{Variational posteriors of parameters}
In the VB-M step, we compute the posteriors of remaining parameters, 
\myeqref{eq:IRM_2dom_VB_v_1_q}, 
\myeqref{eq:IRM_2dom_VB_v_2_q},
and 
\myeqref{eq:IRM_2dom_VB_theta_q}. 
We again omit the derivations and present the final results. 
\begin{equation}
 \hat{\alpha}_{1,k} = 1 + \sum_{i=1}^{N_{1}} \hat{\pi}_{1,i,k} 
 \, , 
  \label{eq:IRM_2dom_VB_posterior_alphahat_1k}
\end{equation}
\begin{equation}
 \hat{\alpha}_{2,l} = 1 + \sum_{j=1}^{N_{2}} \hat{\pi}_{2,j,l} 
 \, ,
  \label{eq:IRM_2dom_VB_posterior_alphahat_2l}
\end{equation}
\begin{equation}
 \hat{\beta}_{1,k} 
 = 
 \alpha_{1} 
 + 
 \sum_{i=1}^{N_{1}} \sum_{m=k+1}^{K_{1}} \hat{\pi}_{1,i,m} 
\, ,
  \label{eq:IRM_VB_posterior_betahat_1k}
\end{equation}
\begin{equation}
 \hat{\beta}_{2,l} 
 = 
 \alpha_{2} 
 + 
 \sum_{j=1}^{N_{2}} \sum_{m=l+1}^{K_{2}} \hat{\pi}_{2,j,m} 
 \, ,
  \label{eq:IRM_VB_posterior_betahat_2l}
\end{equation}
\begin{equation}
 \hat{a}_{k,l} = a_{k,l} 
  + \sum_{i=1}^{N_{1}} \sum_{j=1}^{N_{2}} 
  \hat{\pi}_{1,i,k} \hat{\pi}_{2,j,l} x_{i,j} 
  \, ,
  \label{eq:IRM_2dom_VB_posterior_ahat_kl}
\end{equation}
\begin{equation}
 \hat{b}_{k,l} = b_{k,l} 
  + \sum_{i=1}^{N_{1}} \sum_{j=1}^{N_{2}} 
  \hat{\pi}_{1,i,k} \hat{\pi}_{2,j,l} (1 - x_{i,j}) 
  \, .
  \label{eq:IRM_2dom_VB_posterior_bhat_kl}
\end{equation}

\subsection{Optimizing hyperparameters}
Optimization of the hyperparameters are also formulated as the 
maximization of the VB lower bound. 

It is easy to obtain update rules for hyperparameters by 
taking derivatives of the lower bound. 
Employing the fixed-point methods presented
in~\citep{Iwata12,Minka00,Wallach08}, we have the following 
update rules for hyperparameters. 

\begin{alignat}{2}
 \alpha_{1} 
 &= 
 \frac{
 K_{1}
 }{
 \sum_{k=1}^{K_{1}} \left[
 \psi \left( \hat{\alpha}_{1,k} + \hat{\beta}_{1,k} \right)
 -
 \psi \left( \hat{\beta}_{1,k} \right)
 \right] 
 }
 \, ,
 \label{eq:IRM_2dom_VB_posterior_lowerbound_alpha_1_update} 
 \\
 \alpha_{2} 
 &= 
 \frac{
 K_{2}
 }{
 \sum_{l=1}^{K_{2}} \left[
 \psi \left( \hat{\alpha}_{2,l} + \hat{\beta}_{2,l} \right)
 -
 \psi \left( \hat{\beta}_{2,l} \right)
 \right] 
 }
 \, ,
 \label{eq:IRM_2dom_VB_posterior_lowerbound_alpha_2_update} 
\\
 \tilde{a}_{k,l} 
 &= 
a_{k,l} 
 \frac{
 \psi\left(a_{k,l} + b_{k,l}\right) 
 -
 \psi\left(a_{k,l}\right) 
 }{
 \psi\left(\hat{a}_{k,l} + \hat{b}_{k,l}\right) 
 -
 \psi\left(\hat{a}_{k,l}\right) 
 }
 \, ,
 \label{eq:IRM_2dom_VB_posterior_lowerbound_a_kl_update} 
 \\
 \tilde{b}_{k,l} 
 &= 
 b_{k,l} 
 \frac{
 \psi\left(a_{k,l} + b_{k,l}\right) 
 -
 \psi\left(b_{k,l}\right) 
 }{
 \psi\left(\hat{a}_{k,l} + \hat{b}_{k,l}\right) 
 -
 \psi\left(\hat{b}_{k,l}\right) 
 }
 \, .
 \label{eq:IRM_2dom_VB_posterior_lowerbound_b_kl_update}  
\end{alignat}

%%%%%%%%%%%%%%%%%%%%%%%%%%%
% CVB solutions
%%%%%%%%%%%%%%%%%%%%%%%%%%
\section{Collapsed Variational Bayes (CVB) Solution of IRM}

\subsection{General Idea}
The general idea of CVB inferences for hierarchical probabilistic models
~\citep{Kurihara07,Teh07NIPS,Teh08NIPS_HDP,Asuncion09,Sato_Nakagawa12,Sato12} 
is to assume variational posteriors of hidden variables of 
the model where {\it parameters are marginalized out beforehand}. 

In \myeqref{eq:VB_lowerbound_def}, 
parameters $\bm{\Phi}$ are not marginalized (collapsed) out in 
ordinary VB inference~\citep{Attias00,Bishop06}. 
Thus, we need to compute the variational posteriors of the parameters as well. 
The variational posteriors of the parameters impact the inference results, and this may increase the danger of being trapped at a bad local
optimal point. 

CVB inference first marginalizes out 
the parameters in an exact way (as in a collapsed Gibbs sampler). 
After that, remaining hidden variables are assumed to be independent 
from each other.  
We can avoid the effect of parameter estimations 
and can reduce the number of quantities to be inferred 
because parameters are already marginalized. 
Further, it is known that 
the lower bound of CVB is always tighter than that of the 
original VB~\citep{Teh07NIPS}. 
The formal definition of CVB lower bound for IRM is: 
\begin{equation}
  \mathscr{L}\left(\bm{Z}_{1}, \bm{Z}_{2}\right)
 =
 \int q\left(\bm{Z}_{1},\bm{Z}_{2}\right) 
 \log 
 \frac{
 p\left(\bm{X}, \bm{Z}_{1}, \bm{Z}_{2}\right)
 }{
 q\left(\bm{Z}_{1}, \bm{Z}_{2}\right)
 }
 d \bm{Z} .
 \label{eq:IRM_2dom_CVB_lowerbound_def}
\end{equation}
As evident, this is the same formulation as 
\myeqref{eq:VB_lowerbound_def} except for the 
marginalized parameters. 

To the best of our knowledge, this is the first work to 
formulate and derive CVB solutions for IRM. 

\subsection{Generative model}
Similar to VB, CVB also employs the truncated version of the 
IRM generative model. 
Let us denote the truncated number of clusters of 
the first domain and the second domain as $K_{1}$ and $K_{2}$, 
respectively. 
For readers' convenience, we present the SBP presentation 
of IRM again: 
\begin{alignat}{2}
 v_{1,k} | \alpha_{1} 
 &\sim
 \text{Beta}\left(1, \alpha_{1}\right)
 \, , 
 \label{eq:IRM_2dom_CVB_v_1k_def}
 \\
 v_{2,l} | \alpha_{2} 
 &\sim
 \text{Beta}\left(1, \alpha_{2}\right)
  \, , 
 \label{eq:IRM_2dom_CVB_v_2l_def}
 \\
 \pi_{1,k}
 &= 
 v_{1,k} \prod_{m=1}^{k-1} \left(1 - v_{1,m}\right), 
 \, \, 
 \pi_{1,K_{1}} = 1 - \sum_{m=1}^{K_{1}-1} v_{1,m} 
  \, , 
 \label{eq:IRM_2dom_CVB_pi_1k_def} 
 \\
 \pi_{2,l}
 &= 
 v_{2,l} \prod_{m=1}^{l-1} \left(1 - v_{2,m}\right), 
 \, \, 
 \pi_{2,K_{2}} = 1 - \sum_{m=1}^{K_{2}-1} v_{2,m} 
  \, , 
 \label{eq:IRM_2dom_CVB_pi_2l_def} 
 \\
 z_{1,i} | \bm{\pi}_{1}
 &\sim 
 \text{Multinomial}\left( \bm{\pi}_{1} \right) 
  \, , 
 \label{eq:IRM_2dom_CVB_z_1i_def}
 \\
 z_{2,j} | \bm{\pi}_{2}
 &\sim 
 \text{Multinomial}\left( \bm{\pi}_{2} \right) 
  \, , 
 \label{eq:IRM_2dom_CVB_z_2j_def}
 \\
 \theta_{k,l} | a_{k,l}, b_{k,l} 
 &\sim
 \text{Beta}\left(a_{k,l}, b_{k,l}\right) 
  \, , 
 \label{eq:IRM_2dom_CVB_theta_kl_def}
 \\
 x_{i,j} | \bm{Z}_{1}, \bm{Z}_{2}\{\theta\} 
 &\sim
 \text{Bernoulli}
 \left( 
 \theta_{z_{1,i}, z_{2,j}}
 \right). 
 \label{eq:IRM_2dom_CVB_x_ij_def}
\end{alignat}

\subsection{Counting statistics}
The statistics required by the CVB solutions 
are defined in the same way as for the Gibbs samplers, 
except for 
\myeqref{eq:IRM_2dom_CVB_M_1k_2l}, 
which represents a kind of {\it negative}
membership. 

\begin{alignat}{2}
 m_{1,k} 
 &= 
 \sum_{i=1}^{N_{1}} \mathbb{I}(z_{1,i} = k) 
 = 
 \sum_{i=1}^{N_{1}} z_{1,i,k} 
 \, , \quad
 m_{2,l} 
 = 
 \sum_{j=1}^{N_{2}} \mathbb{I}(z_{2,j} = l) 
 = 
 \sum_{j=1}^{N_{2}} z_{2,j,l} 
 \, ,
 \label{eq:IRM_2dom_CVB_m_1k_2l}
 \\
 M_{1,k} 
 &= 
 \sum_{i=1}^{N_{1}} \mathbb{I}(z_{1,i} > k) 
 = 
 \sum_{k^{\prime}=k+1}^{K_{1}} m_{1,k^{\prime}}
 \, , \quad 
 M_{2,l} 
 = 
 \sum_{j=1}^{N_{2}} \mathbb{I}(z_{2,j} > l) 
 = 
 \sum_{l^{\prime}=l+1}^{K_{2}} m_{2,l^{\prime}}
 \, ,
 \label{eq:IRM_2dom_CVB_M_1k_2l}
 \\
 n_{k, l} 
 &= 
 \sum_{i=1}^{N_{1}} \sum_{j=1}^{N_{2}} z_{1,i,k} z_{2,j,l} x_{i,j}
 \, , \quad
 N_{k, l} 
 = 
 \sum_{i=1}^{N_{1}} \sum_{j=1}^{N_{2}} z_{1,i,k} z_{2,j,l} (1 - x_{i,j})
 \, .
 \label{eq:IRM_2dom_CVB_n_N_kl}
\end{alignat}

\subsection{Variational posterior of $\bm{Z}$}
Before deriving the variational posteriors, 
there are two points to note concerning the difference with 
the original VB. 

First, we assume the form of the 
variational posterior distributions beforehand in the case of VB. 
Then, we directly compute the variational posterior parameters. 
This is possible because the whole generative model retains the
conjugacy. 
In the case of CVB inference, however, the conjugacy does not hold because we marginalize
out parameters ($\bm{v}_{1}, \bm{v}_{2}, \bm{\Theta}$) in the inference. 
Therefore, we cannot assume specific forms of variational posteriors, 
$q\left(\bm{Z}_{1}\right)$ and $q\left(\bm{Z}_{2}\right)$. 

Second, in the case of naive VB, we write down the 
actual lower bound \myeqref{eq:VB_lowerbound_def} 
for VB(Eqs. (\ref{eq:IRM_2dom_VB_lowerbound_p_x}-\ref{eq:IRM_2dom_VB_lowerbound_q_theta}). 
In the case of CVB, we do not explicitly write down the lower bound. 
The reason is that convenient forms of the lower bound are different for 
the hidden variables $\bm{Z}$ and the hyperparameters $\alpha, a, b$. 
Both forms are equivalent, but in practice it is easier to 
choose different representations to derive inference algorithms. 

In fact, the procedure of CVB inference resembles the collapsed Gibbs samplers
more than ordinary VB inferences. 
We take one object out from the model, 
recompute the posterior of the object cluster assignment, 
and put the object back in the model. 
A difference is that CVB computes soft cluster assignments of 
$\bm{Z} = \{ \bm{Z}_{1}, \bm{Z}_{2}\}$ while 
the Gibbs sampler samples hard assignments for each process. 
We repeat this process on all objects, and one iteration of CVB inference is done. 

Let us derive the update rule of the hidden cluster assignment of the
first domain $z_{1,i}$. 
%We omit the discussion of $\bm{Z}_{2}$ because of the page
%imit, but the procedure and the final results are 
%completely symmetric to the first domain. 
First, we modify the representation of the CVB lower bound 
\myeqref{eq:IRM_2dom_CVB_lowerbound_def}. 
The integral is replaced by the summation because $\bm{Z}$ is
discrete. 
\begin{alignat}{2}
 \mathscr{L} (z_{1,i}, \bm{Z}_{1}^{\setminus (1,i)}, \bm{Z}_{2})
 &= 
 \sum_{z_{1,i}} \sum_{\bm{Z}_{1}^{\setminus (1,i)}, \bm{Z}_{2}} 
 \left[
 q\left(z_{1,i}\right) q\left(\bm{Z}_{1}^{\setminus (1,i)}\right) 
 q\left(\bm{Z}_{2}\right)
 \log 
 \frac{
 p\left( 
 \bm{X}^{(1,i)}, \bm{X}^{\setminus (1,i)}, z_{1,i}, 
 \bm{Z}_{1}^{\setminus (1,i)} \bm{Z}_{2}
 \right)
 }{
 q\left(z_{1,i}\right) q\left(\bm{Z}_{1}^{\setminus (1,i)}\right) 
 q\left(\bm{Z}_{2}\right)
 }
 \right] 
 \notag 
 \\
 &=
 \sum_{z_{1,i}} 
 \mathbb{E}_{\bm{Z}_{1}^{\setminus (1,i)}, \bm{Z}_{2}} 
 \biggl[
 q\left(z_{1,i}\right) 
 \Big\{
 \log  
 p\left( 
 \bm{X}^{(1,i)} | \bm{X}^{\setminus (1,i)}, z_{1,i}, 
 \bm{Z}_{1}^{\setminus (1,i)}, \bm{Z}_{2} 
 \right)
 \notag 
 \\
 &\qquad \qquad \quad
 +
 \log  
 p\left( 
 z_{1,i} | \bm{Z}_{1}^{\setminus (1,i)} 
 \right) 
 -
 \log q\left(z_{1,i}\right) 
 \notag 
 \\
 &\qquad \qquad \quad
 + (\text{Terms that are not related to} \, \, z_{1,i})
 \Bigr\}
 \biggr] \, . 
 \label{eq:IRM_2dom_CVB_lowerbound_z_1i}
\end{alignat}
As in the case of the Gibbs sampler solution, 
$\bm{X}^{(1,i)} = \{ x_{i, \cdot} \}$ denotes 
the set of all observations 
concerning object $i$ of the first domain. 
The remaining observations, hidden variables excluding $z_{1,i}$, 
and statistics computed on these data are denoted by 
$\setminus (1,i)$. 
$\mathbb{E}_{x}[y]$ indicates 
the expectation of $y$ on the variational posterior of $x$. 

The above rewriting yields a 
{\it Gibbs-like} likelihood and prior terms in \myeqref{eq:IRM_2dom_Gibbs_z_1_def}, 
averaged over other posteriors. 
Taking the derivative of 
\myeqref{eq:IRM_2dom_CVB_lowerbound_z_1i} 
w.r.t. $q(z_{1,i})$ and equating it to zero, 
we have the following update rule for $q(z_{1,i})$: 
\begin{equation}
 q\left(z_{1,i}\right)
 \propto  
 \exp \left\{
 \mathbb{E}_{\bm{Z}_{1}^{\setminus (1,i)}, \bm{Z}_{2}} \left[
 \log  
 p\left( 
 \bm{X}^{(1,i)} | \bm{X}^{\setminus (1,i)}, z_{1,i}, 
 \bm{Z}_{1}^{\setminus (1,i)}, \bm{Z}_{2} 
 \right)
 \right]
 +
 \mathbb{E}_{\bm{Z}_{1}^{\setminus (1,i)}, \bm{Z}_{2}} \left[
 \log  
 p\left( 
 z_{1,i} | \bm{Z}_{1}^{\setminus (1,i)} 
 \right)
 \right] 
\right\} \, .
 \label{eq:IRM_2dom_CVB_lowerbound_z_1i_posterior}
\end{equation}

Before evaluating the expectations, 
we need to derive two terms inside $\mathbb{E}$. 
Let us start from the prior part. 
Using \myeqref{eq:IRM_2dom_CVB_v_1k_def} 
and \myeqref{eq:IRM_2dom_CVB_z_1i_def}, we readily obtain the following: 
\begin{alignat}{2}
 p\left(\bm{Z}_{1}, \bm{v}_{1} | \alpha_{1}\right) 
 &= 
 \alpha_{1}^{K_{1}} \prod_{k=1}^{K_{1}} 
 \text{Beta}\left(v_{1,k}; m_{1,k} + 1, M_{1,k} + \alpha_{1}\right) 
 \text{B}\left( m_{1,k} + 1, M_{1,k} + \alpha_{1}\right). 
 \notag
\\
 p\left(\bm{Z}_{1} | \alpha_{1}\right) 
 &= 
 \int p\left(\bm{Z}_{1}, \bm{v}_{1} | \alpha_{1} \right) 
 d \bm{v}_{1} 
 = 
 \alpha_{1}^{K_{1}} \prod_{k=1}^{K_{1}} 
 \frac{ 
 \Gamma\left( m_{1,k} + 1 \right) 
 \Gamma\left( M_{1,k} + \alpha_{1} \right)
 }{
 \Gamma\left( m_{1,k} + M_{1,k} + \alpha_{1} + 1 \right)
 }. 
 \label{eq:IRM_2dom_CVB_z_1_marginal}
\end{alignat}

Now we are ready to evaluate 
$p \left( z_{1,i} = k | \bm{Z}_{1}^{\setminus (1,i)} \right)$. 
\begin{alignat}{2}
 p\left( z_{1,i} = k | \bm{Z}_{1}^{\setminus (1,i)}, \alpha_{1}\right) 
 &= 
 \frac{
 p\left(\bm{Z}_{1} | \alpha_{1} \right)
 }{
 p\left(\bm{Z}_{1}^{\setminus (1,i)} | \alpha_{1} \right)
 }
 \notag 
 \\
 &= 
 \frac{
 m_{1,k}^{\setminus (1,i)} + 1
 }{
 m_{1,k}^{\setminus (1,i)} + M_{1,k}^{\setminus (1,i)} + \alpha_{1} + 1
 }
 \prod_{k^{\prime}=1}^{k-1} 
 \frac{
 M_{1,k^{\prime}}^{\setminus (1,i)} + \alpha_{1}
 }{
 m_{1,k^{\prime}}^{\setminus (1,i)} + M_{1,k^{\prime}}^{\setminus (1,i)} + \alpha_{1} + 1
 }. 
 \label{eq:IRM_2dom_CVB_z_1_prior}
\end{alignat}

The likelihood term is also easily available thanks to conjugacy:  
\begin{multline}
 p\left(\bm{X}^{(1,i)} | z_{1,i} = k, 
 \bm{Z}_{1}^{\setminus (1,i)}, \bm{Z}_{2}, \bm{X}^{\setminus (1,i)}\right)
 \\
 = 
 \prod_{l=1}^{K_{2}}
 \frac{
 \Gamma\left( a_{k,l} + b_{k,l} + n_{k,l}^{\setminus (1,i)} + N_{k,l}^{\setminus (1,i)} \right)
 }{
 \Gamma\left( a_{k,l} + n_{k,l}^{\setminus (1,i)} \right)
 \Gamma\left( b_{k,l} + N_{k,l}^{\setminus (1,i)} \right)
 }
 \frac{
 \Gamma\left( a_{k,l} + n_{k,l}^{\setminus (1,i)} + n_{k,l}^{+(1,i,k)} \right)
 \Gamma\left( b_{k,l} + N_{k,l}^{\setminus (1,i)} + N_{k,l}^{+(1,i,k)} \right)
 }{
 \Gamma\left( a_{k,l} + b_{k,l} 
 + n_{k,l}^{\setminus (1,i)} + N_{k,l}^{\setminus (1,i)} 
 + n_{k,l}^{+(1,i,k)} + N_{k,l}^{+(1,i,k)}\right)
 }.  
 \label{eq:IRM_2dom_CVB_z_1_likelihood} 
\end{multline}
$n^{+(1,i,k)}$ and $N^{+(1,i,k)}$ denotes the statistics computed solely 
on $\bm{X}^{(1,i)}$ given $z_{1,i} = k$. 

Plugging 
\myeqref{eq:IRM_2dom_CVB_z_1_prior}
and
\myeqref{eq:IRM_2dom_CVB_z_1_likelihood} 
into 
\myeqref{eq:IRM_2dom_CVB_lowerbound_z_1i_posterior}, 
then we obtain the variational posterior $q(z_{1,i})$. 
$q(z_{1,i})$ must be normalized so that the summation for $K_{1}$ clusters equals one. 

In the same manner, we obtain the update rule for $q(z_{2,j})$. 
\begin{equation}
 q\left(z_{2,j}\right)
 \propto  
 \exp \left\{
 \mathbb{E}_{\bm{Z}_{1}, \bm{Z}_{2}^{\setminus (2,j)}} \left[
 \log  
 p\left( 
 \bm{X}^{(2,j)} | \bm{X}^{\setminus (2,j)}, 
 \bm{Z}_{1}, z_{2,j}, \bm{Z}_{2}^{\setminus (2,j)} 
 \right)
 \right]
 +
 \mathbb{E}_{\bm{Z}_{1}, \bm{Z}_{2}^{\setminus (2,j)}} \left[
 \log  
 p\left( 
 z_{2,j} | \bm{Z}_{2}^{\setminus (2,j)} 
 \right)
 \right] 
\right\} \, .
 \label{eq:IRM_2dom_CVB_lowerbound_z_2j_posterior}
\end{equation}

\begin{equation}
 p\left( z_{2,j} = l | \bm{Z}_{2}^{\setminus (2,j)}, \alpha_{2}\right) 
  =
  \frac{
  m_{2,l}^{\setminus (2,j)} + 1
  }{
  m_{2,l}^{\setminus (2,j)} + M_{2,l}^{\setminus (2,j)} + \alpha_{2} + 1
  }
  \prod_{l^{\prime}=1}^{l-1} 
  \frac{
  M_{2,l^{\prime}}^{\setminus (2,j)} + \alpha_{2}
  }{
  m_{2,l^{\prime}}^{\setminus (2,j)} + M_{2,l^{\prime}}^{\setminus (2,j)} + \alpha_{2} + 1
  } 
  \, . 
 \label{eq:IRM_2dom_CVB_z_2_prior}  
\end{equation}
\begin{multline}
 p\left(\bm{X}^{(2,j)} | z_{2,j} = l, 
 \bm{Z}_{1}, \bm{Z}_{2}^{\setminus (2,j)}, \bm{X}^{\setminus (2,j)}\right)
 \\
 = 
 \prod_{k=1}^{K_{1}}
 \frac{
 \Gamma\left( a_{k,l} + b_{k,l} + n_{k,l}^{\setminus (2,j)} + N_{k,l}^{\setminus (2,j)} \right)
 }{
 \Gamma\left( a_{k,l} + n_{k,l}^{\setminus (2,j)} \right)
 \Gamma\left( b_{k,l} + N_{k,l}^{\setminus (2,j)} \right)
 }
 \frac{
 \Gamma\left( a_{k,l} + n_{k,l}^{\setminus (2,j)} + n_{k,l}^{+(2,j,l)} \right)
 \Gamma\left( b_{k,l} + N_{k,l}^{\setminus (2,j)} + N_{k,l}^{+(2,j,l)} \right)
 }{
 \Gamma\left( a_{k,l} + b_{k,l} 
 + n_{k,l}^{\setminus (2,j)} + N_{k,l}^{\setminus (2,j)} 
 + n_{k,l}^{+(2,j,l)} + N_{k,l}^{+2jjl}\right)
 }
 \, .
 \label{eq:IRM_2dom_CVB_z_2_likelihood} 
\end{multline}

\subsection{Computing Variational Expectations}
The posterior of $z_{1,i}$ in 
\myeqref{eq:IRM_2dom_CVB_lowerbound_z_1i_posterior} 
requires an expectation computation over the 
cluster assignments of $\bm{Z}_{1}^{\setminus (1,i)}$ and 
$\bm{Z}_{2}$. 
However, this is an intractable discrete combinatorial computation: 
there are $K_{1}^{N_{1}-1} \times K_{2}^{N_{2}}$ possible combinations. 
CVB inference approximates these expectations by Taylor expansion. 
Let us denote the expectation of predicate $x$ as $a = \mathbb{E}[x]$. 
Then we have: 
\begin{equation}
 f\left(x\right) 
  \approx 
  f\left(a\right) 
  +
  f^{\prime}\left(a\right) \left(x - a\right) 
  +
  \frac{1}{2}
  f^{\prime \prime}\left(a\right) \left(x - a\right)^{2} 
  \, .
  \label{eq:Taylor_1}
\end{equation}
Taking the expectations of both sides of 
\myeqref{eq:Taylor_1}, we obtain the following equation: 
\begin{alignat}{2}
 \mathbb{E}[f\left(x\right)] 
 &\approx
 \mathbb{E}[f\left(a\right)] 
 +
 \mathbb{E}[f^{\prime}\left(a\right) \left(x - a\right)] 
 +
 \frac{1}{2}
 \mathbb{E}[f^{\prime \prime}\left(a\right) \left(x - a\right)^{2}] 
 \notag 
 \\
 &=  
 f\left(a\right) 
 +
 \frac{1}{2}
 \mathbb{E}[f^{\prime \prime}\left(a\right) \left(x - a\right)^{2}]
 \notag 
 \\
 &=  
 f\left(\mathbb{E}[x]\right) 
 +
 \frac{1}{2}
 f^{\prime \prime}\left(\mathbb{E}[x]\right) 
 \mathbb{V}[x]
 \, . 
 \label{eq:Taylor_2}
\end{alignat}
The 0th order term is constant. 
The 1st order term is canceled because 
$x - a$ becomes zero by taking the expectation. 
$\mathbb{V}$ denotes the posterior variance. 

There are two types of approximations in CVB studies. 
The original \textbf{CVB} such as~\citep{Teh07NIPS} employs 2nd-order 
Taylor approximation, and considers the variance as in
\myeqref{eq:Taylor_2}. 
\citep{Asuncion09} reveals that the 0th-order Taylor
approximation performs quite well in practice for LDA. 
This is called the \textbf{CVB0} solution, and it approximates the 
posterior expectation by 
\begin{equation}
 \mathbb{E}[f(x)] \approx f( \mathbb{E}[x] ). 
 \label{eq:Taylor_0}
\end{equation}
Obviously, the CVB0 solution is much simpler than 
that of the 2nd-order approximation, and 
is often superior in terms of the perplexity of the learnt
model~\citep{Asuncion09,Sato_Nakagawa12,Sato12}. 

Now we apply the Taylor approximation to
two terms in the r.h.s. of 
\myeqref{eq:IRM_2dom_CVB_lowerbound_z_1i_posterior}. 
%For example, 
%the likelihood (first) term of the r.h.s. of 
%\myeqref{eq:IRM_2dom_CVB_lowerbound_z_1i_posterior} 
%requires expectations of log-gamma function. 
We pick one term from 
\myeqref{eq:IRM_2dom_CVB_z_1_likelihood} and show how we can
approximate the expectation.  
\begin{alignat}{2}
 \mathbb{E}_{\bm{Z}_{1}^{\setminus (1,i)}, \bm{Z}_{2}} \left[
 \log \Gamma\left( a_{k,l} + n_{k,l}^{\setminus (1,i)} \right)
 \right]
 &\approx
 \log \Gamma\left( 
 \mathbb{E}_{\bm{Z}_{1}^{\setminus (1,i)}, \bm{Z}_{2}} \left[
 a_{k,l} + n_{k,l}^{\setminus (1,i)} 
 \right] 
 \right) 
 \notag 
 \\
 &+ 
 \frac{1}{2} 
 \Psi \left( 
 \mathbb{E}_{\bm{Z}_{1}^{\setminus (1,i)}, \bm{Z}_{2}} \left[
 a_{k,l} + n_{k,l}^{\setminus (1,i)} 
 \right] 
 \right) 
 \mathbb{E}_{\bm{Z}_{1}^{\setminus (1,i)}, \bm{Z}_{2}} \left[
 \left(
 a_{k,l} + n_{k,l}^{\setminus (1,i)} 
 -
 \mathbb{E}_{\bm{Z}_{1}^{\setminus (1,i)}, \bm{Z}_{2}} \left[
 a_{k,l} + n_{k,l}^{\setminus (1,i)} \right] 
 \right)^{2}
 \right] 
 \notag 
 \\
 &=
 \log \Gamma\left( 
 a_{k,l} + 
 \mathbb{E}_{\bm{Z}_{1}^{\setminus (1,i)}, \bm{Z}_{2}} \left[
 n_{k,l}^{\setminus (1,i)} \right] 
 \right) 
 \notag 
 \\
 &+ 
 \frac{1}{2} 
 \Psi \left( 
 a_{k,l} + 
 \mathbb{E}_{\bm{Z}_{1}^{\setminus (1,i)}, \bm{Z}_{2}} \left[
 n_{k,l}^{\setminus (1,i)} \right] 
 \right) 
 \mathbb{E}_{\bm{Z}_{1}^{\setminus (1,i)}, \bm{Z}_{2}} \left[
 \left(
 n_{k,l}^{\setminus (1,i)} 
 -
 \mathbb{E}_{\bm{Z}_{1}^{\setminus (1,i)}, \bm{Z}_{2}} \left[
 n_{k,l}^{\setminus (1,i)} \right] 
 \right)^{2}
 \right] \, , 
 \notag 
 \\
 &=
 \log \Gamma\left( 
 a_{k,l} + 
 \mathbb{E}_{\bm{Z}_{1}^{\setminus (1,i)}, \bm{Z}_{2}} \left[
 n_{k,l}^{\setminus (1,i)} \right] 
 \right) 
 \notag 
 \\
 &+ 
 \frac{1}{2}
 \Psi \left( 
 a_{k,l} + 
 \mathbb{E}_{\bm{Z}_{1}^{\setminus (1,i)}, \bm{Z}_{2}} \left[
 n_{k,l}^{\setminus (1,i)} \right] 
 \right) 
 \mathbb{V}_{\bm{Z}_{1}^{\setminus (1,i)}, \bm{Z}_{2}} \left[
 n_{k,l}^{\setminus (1,i)} \right] \, , 
 \, \textbf{(CVB)} 
 \label{eq:IRM_2dom_CVB_approx_example_likelihood}
 \\
 &\approx
 \log \Gamma\left( 
 a_{k,l} + 
 \mathbb{E}_{\bm{Z}_{1}^{\setminus (1,i)}, \bm{Z}_{2}} \left[
 n_{k,l}^{\setminus (1,i)} \right] 
 \right) \, , 
 \, \textbf{(CVB0)} 
 \label{eq:IRM_2dom_CVB_approx_example_likelihood_0}
\end{alignat}
where $\Psi$ is the trigamma function. 
%We omit $(\bm{Z}_{1}^{\setminus(1,i)}, \bm{Z}_{2})$ for presentation clarity. 
%
For \myeqref{eq:IRM_2dom_CVB_z_1_prior}, 
the approximation is much simpler: 
\begin{equation}
 \mathbb{E}_{\bm{Z}_{1}^{\setminus (1,i)}, \bm{Z}_{2}} \left[
 m_{1,k}^{\setminus (1,i)} + 1
 \right]
 =
 \mathbb{E}_{\bm{Z}_{1}^{\setminus (1,i)}, \bm{Z}_{2}} \left[
 m_{1,k}^{\setminus (1,i)} 
 \right] 
 + 1 \, \textbf{(CVB,CVB0)}.
 \label{eq:IRM_2dom_CVB_approx_example_prior}
\end{equation}

From the above examples, 
we see that the expectation computations in 
\myeqref{eq:IRM_2dom_CVB_lowerbound_z_1i_posterior} 
are achieved by  
replacing counting statistics $n, N, m, M$ with 
their expectations and variances. 

The expectations and variances of the counting statistics in 
Eqs. (\ref{eq:IRM_2dom_CVB_m_1k_2l}-\ref{eq:IRM_2dom_CVB_n_N_kl})
are computed as follows: 
\begin{alignat}{2}
 \mathbb{E} [ m_{1,k} ]
 &= 
 \sum_{i=1}^{N_{1}} q( z_{1,i,k} )
 \, , \quad
 \mathbb{E} [ m_{2,l} ]
 = 
 \sum_{j=1}^{N_{2}} q( z_{2,j,l} )
 \, ,
 \label{eq:IRM_2dom_CVB_m_1k_2l_E}
 \\
 \mathbb{E} [ M_{1,k} ] 
 &= 
 \sum_{k^{\prime}=k+1}^{K_{1}} \mathbb{E}[ m_{1,k^{\prime}} ]
 \, , \quad 
 \mathbb{E} [ M_{2,l} ]
 = 
 \sum_{l^{\prime}=l+1}^{K_{2}} \mathbb{E} [ m_{2,l^{\prime}} ]
 \, ,
 \label{eq:IRM_2dom_CVB_M_1k_2l_E}
 \\
 \mathbb{E} [ n_{k, l} ]
 &= 
 \sum_{i=1}^{N_{1}} \sum_{j=1}^{N_{2}} q( z_{1,i,k} ) q( z_{2,j,l} ) x_{i,j}
 \, , \quad
 \mathbb{E} [ N_{k, l} ]
 = 
 \sum_{i=1}^{N_{1}} \sum_{j=1}^{N_{2}} q( z_{1,i,k} ) q( z_{2,j,l} ) (1 - x_{i,j})
 \, ,
 \label{eq:IRM_2dom_CVB_n_N_kl_E}
 \\
  \mathbb{V} [ n_{k, l} ]
  &= 
  \sum_{i=1}^{N_{1}} \sum_{j=1}^{N_{2}} 
  q( z_{1,i,k} ) \left(1 - q( z_{1,i,k} ) \right) q( z_{2,j,l} ) \left(1 - q( z_{2,j,l} )\right) x_{i,j}^{2}
  \, , 
  \label{eq:IRM_2dom_CVB_n_kl_V}
  \\
  \mathbb{V} [ N_{k, l} ]
  &= 
  \sum_{i=1}^{N_{1}} \sum_{j=1}^{N_{2}} 
  q( z_{1,i,k} ) \left(1 - q( z_{1,i,k} ) \right) q( z_{2,j,l} ) \left(1 - q( z_{2,j,l} )\right) x_{i,j}^{2} (1 - x_{i,j})^{2}
  \, .
  \label{eq:IRM_2dom_CVB_N_kl_V}
\end{alignat}
All expectations and variances are computed on 
$q\left(\bm{Z}_{1}\right)$ and $q\left(\bm{Z}_{2}\right)$. 

Based on 
Eqs. (\ref{eq:IRM_2dom_CVB_m_1k_2l_E}-\ref{eq:IRM_2dom_CVB_N_kl_V}), 
we can easily derive the expectations and variances for the first domain updates: 
\begin{alignat}{2}
 \mathbb{E}[m_{1,k}^{\setminus (1,i)}] 
  &= 
  \sum_{i^{\prime} \neq i} q\left(z_{1,i^{\prime},k}\right) 
  =
  \mathbb{E}[m_{1,k}] - q\left(z_{1,i,k}\right) 
  \, , \quad 
 \mathbb{E}[M_{1,k}^{\setminus (1,i)}]  
  = 
  \sum_{k^{\prime}=k+1}^{K} \mathbb{E}[m_{1,k^{\prime}}^{\setminus (1,i)}]
 \label{eq:IRM_2dom_CVB_mM_1k_approx}
  \\
 \mathbb{E}[n_{k, l}^{+(1,i,k)}]  
 &=  
 \sum_{j=1}^{N_{2}}  
 q\left(z_{2,j,l}\right) x_{i,j}
 \, , \quad
  \mathbb{E}[N_{k, l}^{+(1,i,k)}]  
  =  
  \sum_{j=1}^{N_{2}}  
  q\left(z_{2,j,l}\right) \left( 1 - x_{i,j} \right)
  \, , 
 \label{eq:IRM_2dom_CVB_nN_1kl_approx_plus}
 \\
  \mathbb{E}[n_{k, l}^{\setminus (1,i)}]  
  &= 
  \mathbb{E}[n_{k, l}]
  -
  q\left(z_{1,i,k}\right) \mathbb{E}[n_{k, l}^{+(1,i,k)}] 
  \, , 
 \label{eq:IRM_2dom_CVB_n_1kl_approx} 
 \\
 \mathbb{E}[N_{k, l}^{\setminus (1,i)}] 
 &= 
 \mathbb{E}[N_{k, l}]
 -
 q\left(z_{1,i,k}\right) \mathbb{E}[N_{k, l}^{+(1,i,k)}] 
 \, , 
 \label{eq:IRM_2dom_CVB_N_1kl_approx} 
\\
 \mathbb{V}[n_{k, l}^{+(1,i,k)}]  
 &= 
 \sum_{j=1}^{N_{2}} 
 q\left(z_{2,j,l}\right)
 \left(1 - q\left(z_{2,j,l}\right)\right)
 x_{i,j}^{2} 
 \, , \, 
 \mathbb{V}[N_{k, l}^{+(1,i,k)}] 
 = 
 \sum_{j=1}^{N_{2}}  
 q\left(z_{2,j,l}\right)
 \left(1 - q\left(z_{2,j,l}\right)\right) 
 \left( 1 - x_{i,j} \right)^{2}
 \, , 
 \label{eq:IRM_2dom_CVB_nN_1kl_approx_var_plus}
 \\
  \mathbb{V}[n_{k, l}^{\setminus (1,i)}]  
  &= 
 \mathbb{V}[n_{k, l}]
  -
  q\left(z_{1,i,k}\right) \left(1 - q\left(z_{1,i,k}\right)\right) 
  \mathbb{V}[n_{k, l}^{+(1,i,k)}] 
  \, , 
  \label{eq:IRM_2dom_CVB_n_1kl_approx_var}
  \\
 \mathbb{V}[N_{k, l}^{\setminus (1,i)}] 
 &= 
 \mathbb{V}[N_{k, l}]
  -
  q\left(z_{1,i,k}\right) \left(1 - q\left(z_{1,i,k}\right)\right) 
  \mathbb{V}[N_{k, l}^{+(1,i,k)}] 
 \, .
 \label{eq:IRM_2dom_CVB_N_1kl_approx_var}   
\end{alignat}
All expectations and variances are computed on 
$q\left(\bm{Z}_{1}^{\setminus (1,i)}\right)$ and $q\left(\bm{Z}_{2}\right)$. 
By plugging Eqs. (\ref{eq:IRM_2dom_CVB_mM_1k_approx}-\ref{eq:IRM_2dom_CVB_N_1kl_approx_var}) 
into 
\myeqref{eq:IRM_2dom_CVB_z_1_prior} and 
\myeqref{eq:IRM_2dom_CVB_z_1_likelihood}, 
we can evaluate $q(z_{1,i} = k)$ for CVB and CVB0 solutions. 

For the second domain updates, we have a completely symmetric story. 
Approximated expectations are: 
\begin{alignat}{2}
 \mathbb{E}[m_{2,l}^{\setminus (2,j)}] 
 &= 
 \sum_{j^{\prime} \neq j} q\left(z_{2,j^{\prime},l}\right)
 \, ,
 \quad 
 \mathbb{E}[M_{2,l}^{\setminus (2,j)}]  
 = 
 \sum_{l^{\prime}=l+1}^{K_{2}} 
 \mathbb{E}[m_{2,{l^{\prime}}}^{\setminus (2,j)}]
 \, ,
 \label{eq:IRM_2dom_CVB_mM_2l_approx}
 \\
 \mathbb{E}[n_{k, l}^{+(2,j,l)}]  
  &= 
  \sum_{i=1}^{N_{1}} 
  q\left(z_{1,i,k}\right) x_{i,j}
  \, , \quad 
  \mathbb{E}[N_{k, l}^{+(2,j,l)}] 
  = 
  \sum_{i=1}^{N_{1}}  
  q\left(z_{1,i,k}\right) \left( 1 - x_{i,j} \right)
  \, , 
  \label{eq:IRM_2dom_CVB_nN_2kl_approx_plus}
  \\
 \mathbb{E}[n_{k, l}^{\setminus (2,j)}]  
 &= 
 \mathbb{E}[n_{k, l}] 
 -
 q\left(z_{2,j,l}\right) \mathbb{E}[n_{k, l}^{+(2,j,l)}] 
 \, ,  
 \label{eq:IRM_2dom_CVB_n_2kl_approx}
 \\
 \mathbb{E}[N_{k, l}^{\setminus (2,j)}] 
 &= 
  \mathbb{E}[N_{k, l}] 
  -
  q\left(z_{2,j,l}\right) \mathbb{E}[N_{k, l}^{+(2,j,l)}] 
 \, , 
 \label{eq:IRM_2dom_CVB_N_2kl_approx}
 \\
  \mathbb{V}[n_{k, l}^{+(2,j,l)}]  
  &= 
  \sum_{i=1}^{N_{1}} 
  q\left(z_{1,i,k}\right)
  \left(1 - q\left(z_{1,i,k}\right)\right)
  x_{i,j}^{2}
  \, ,  
   \mathbb{V}[N_{k, l}^{+(2,j,l)}] 
   = 
   \sum_{i=1}^{N_{1}}  
   q\left(z_{1,i,k}\right) 
   \left(1 - q\left(z_{1,i,k}\right) \right)
   \left( 1 - x_{i,j} \right)^{2}
   \, , 
  \label{eq:IRM_2dom_CVB_n_2kl_approx_var_plus}
  \\
 \mathbb{V}[n_{k, l}^{\setminus (2,j)}]  
 &= 
 \mathbb{V}[n_{k, l}] 
  -
  q\left(z_{2,j,l}\right) \left(1 - q\left(z_{2,j,l}\right) \right)
  \mathbb{V}[n_{k, l}^{+(2,j,l)}] 
 \, , 
 \label{eq:IRM_2dom_CVB_n_2kl_approx_var}
 \\
 \mathbb{V}[N_{k, l}^{\setminus (2,j)}] 
 &= 
 \mathbb{V}[N_{k, l}] 
 -
 q\left(z_{2,j,l}\right) \left(1 - q\left(z_{2,j,l}\right) \right)
 \mathbb{V}[N_{k, l}^{+(2,j,l)}] 
 \, . 
 \label{eq:IRM_2dom_CVB_N_2kl_approx_var}
\end{alignat}
All expectations and variances are computed on 
$q\left(\bm{Z}_{1}\right)$ and $q\left(\bm{Z}_{2}^{\setminus (2,j)}\right)$. 
Plugging the above equations into 
\myeqref{eq:IRM_2dom_CVB_z_2_prior} and 
\myeqref{eq:IRM_2dom_CVB_z_2_likelihood}, 
we can then evaluate $q(z_{2,j} = l)$ for CVB and CVB0 solutions. 

By iterating the variational posterior updates for all $i, j, k, l$
in an interleaving manner, we obtain the local optimal solutions of 
$q(\bm{Z}_{1})$ and $q(\bm{Z}_{2})$. 
The overall procedure for CVB inference of IRM is 
described as pseudo codes in \myfigref{fig:CVB_pseudocode}. 

\begin{figure*}
 \begin{center}
  \includegraphics[width=100mm]{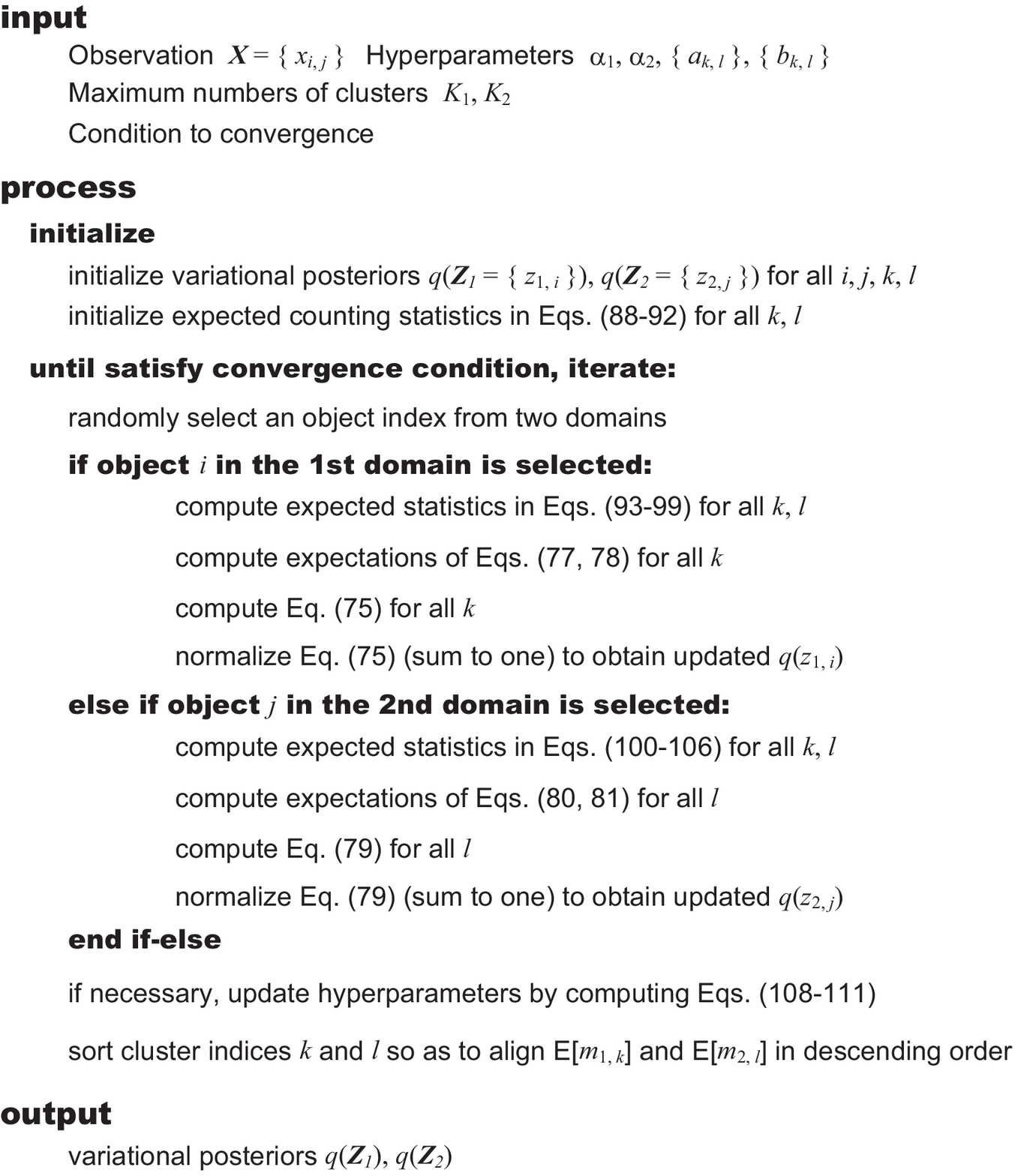}
 \end{center}
 \caption{Pseudo code of CVB inference for IRM. }
\label{fig:CVB_pseudocode}
\end{figure*}

\subsection{Optimizing hyperparameters}
We can also derive CVB0-based update rules of hyperparameters 
by taking derivatives of 
\myeqref{eq:IRM_2dom_CVB_lowerbound_def} w.r.t. 
the hyperparameters. 
In particular, 
optimizing the concentration parameters of DP, $\alpha_{1}$ and
$\alpha_{2}$ are important. 
They play an important role in nonparametric Bayes modeling, 
but their update rule for CVB has never been studied. 

For the hyperparameters, we use another representation of 
the lower bound. 
\begin{alignat}{2}
 \mathscr{L} \left(
 \alpha_{1}, \alpha_{2}, a, b 
 \right)
 &= 
 \sum_{\bm{Z}_{1}, \bm{Z}_{2} } 
 \left[
 q\left(\bm{Z}_{1}\right)
 q\left(\bm{Z}_{2}\right)
 \log 
 \frac{
 p\left( 
 \bm{X}, \bm{Z}_{1}, \bm{Z}_{2}, 
 \alpha_{1}, \alpha_{2}, a, b 
 \right)
 }{
 q\left(\bm{Z}_{1}\right)
 q\left(\bm{Z}_{1}\right)
 }
 \right] 
 \notag 
 \\
 &=
 \mathbb{E}_{\bm{Z}_{1}, \bm{Z}_{2}} \left[
 \log  
 p\left( 
 \bm{X} | \bm{Z}_{1}, \bm{Z}_{2}, a, b 
 \right)
 +
 \log  
 p\left( 
 \bm{Z}_{1} | \alpha_{1}
 \right)
 +
 \log  
 p\left( 
 \bm{Z}_{2} | \alpha_{2}
 \right)
 -
 \log q\left(\bm{Z}_{1}\right)
 -
 \log q\left(\bm{Z}_{2}\right)
 \right] 
 \notag 
 \\
 &=
 \sum_{k} \sum_{l} 
 \mathbb{E}_{\bm{Z}_{1}, \bm{Z}_{2}} \left[
 \log 
 \frac{
 \Gamma \left( a_{k,l} + b_{k,l} \right)
 }{
 \Gamma \left( a_{k,l} \right)
 \Gamma \left( b_{k,l} \right)
 }
 \frac{
 \Gamma \left( a_{k,l} + n_{k,l} \right)
 \Gamma \left( b_{k,l} + N_{k,l} \right)
 }{
 \Gamma \left( a_{k,l} + b_{k,l} + n_{k,l} + N_{k,l} \right)
 }
 \right]
 \notag 
 \\
 &+
 K_{1} \log \alpha_{1} 
 +
 \sum_{k}
 \mathbb{E}_{\bm{Z}_{1}, \bm{Z}_{2}} \left[
 \log \Gamma \left(M_{1,k} + \alpha_{1}\right)
 -
 \log \Gamma \left(m_{1,k} + M_{1,k} + \alpha_{1} + 1\right)
 \right]
 \notag 
 \\
 &+
 K_{2} \log \alpha_{2} 
 +
 \sum_{l}
 \mathbb{E}_{\bm{Z}_{2}, \bm{Z}_{2}} \left[
 \log \Gamma \left(M_{2,l} + \alpha_{2}\right)
 -
 \log \Gamma \left(m_{2,l} + M_{2,l} + \alpha_{2} + 1\right)
 \right]
 \notag 
 \\
 &-
 \mathbb{E}_{\bm{Z}_{1}, \bm{Z}_{2}} \left[
 \log q\left(\bm{Z}_{1}\right)
 \right]
 -
 \mathbb{E}_{\bm{Z}_{1}, \bm{Z}_{2}} \left[
 \log q\left(\bm{Z}_{2}\right)
 \right] \, .
 \label{eq:IRM_2dom_CVB_lowerbound_hyprm}
\end{alignat}
The last line is irrelevant to the hyperparameters. 

To derive the update rules, 
we use the fixed point iteration technique used in 
\citep{Iwata12,Minka00,Wallach08}. 
For the concentration parameter $\alpha_{1}$, 
we have the following inequality between 
the current values of $\alpha_{1}$ and the 
new value $\hat{\alpha}_{1}$: 
\begin{multline}
 \log \Gamma \left( 
 M_{1,k} + \hat{\alpha}_{1} \right)
 -
 \log \Gamma \left(
 m_{1,k} +  M_{1,k} + \hat{\alpha}_{1} + 1 \right)
 \\
 \geq
 \log \Gamma \left( 
 M_{1,k} + \alpha_{1} \right)
 -
 \log \Gamma \left(
 m_{1,k} +  M_{1,k} + \alpha_{1} + 1 \right) 
 \\
 + 
 \left( \alpha_{1} - \hat{\alpha}_{1} \right) 
 \left\{
 \Psi \left( 
 m_{1,k} +  M_{1,k} + \alpha_{1} + 1 \right) 
 -
 \Psi \left( 
 M_{1,k} + \alpha \right) 
 \right\}
 \, .
 \label{eq:IRM_2dom_CVB_lowerbound_alpha_1_inequality}
\end{multline}
We insert the above inequality into the lower bound \myeqref{eq:IRM_2dom_CVB_lowerbound_hyprm}, 
take derivative w.r.t. $\hat{\alpha}_{1}$ with 
CVB0 approximation and equate to zero. 
Then, we obtain the following update rule for $\alpha_{1}$:
\begin{equation}
 \hat{\alpha}_{1} = 
  \frac{
  K_{1}
  }{
  \sum_{k}
 \left\{
 \Psi \left( 
 \mathbb{E}_{\bm{Z}_{1}, \bm{Z}_{2}} \left[ m_{1,k} \right] 
 +  \mathbb{E}_{\bm{Z}_{1}, \bm{Z}_{2}} \left[ M_{1,k} \right] 
 + \alpha_{1} + 1 \right) 
 -
 \Psi \left( 
 \mathbb{E}_{\bm{Z}_{1}, \bm{Z}_{2}} \left[ M_{1,k} \right] 
 + \alpha_{1} \right) 
 \right\} 
  } \, .
\label{eq:IRM_2dom_CVB_alpha_1_update}
\end{equation}
For $\alpha_{2}$, 
\begin{equation}
 \hat{\alpha}_{2} = 
  \frac{
  K_{2}
  }{
  \sum_{l}
 \left\{
 \Psi \left( 
 \mathbb{E}_{\bm{Z}_{1}, \bm{Z}_{2}} \left[ m_{2,l} \right] 
 +  \mathbb{E}_{\bm{Z}_{1}, \bm{Z}_{2}} \left[ M_{2,l} \right] 
 + \alpha_{2} + 1 \right) 
 -
 \Psi \left( 
 \mathbb{E}_{\bm{Z}_{1}, \bm{Z}_{2}} \left[ M_{2,l} \right] 
 + \alpha_{2} \right) 
 \right\} 
  } \, .
\label{eq:IRM_2dom_CVB_alpha_2_update}
\end{equation}

In the same manner, we have the update rules for 
the observation hyperparameters $a_{k,l}$ and $b_{k,l}$. 
\begin{equation}
 \hat{a}_{k,l} 
  =
  a_{k,l}
  \frac{
  \Psi \left( 
		a_{k,l} + 
		\mathbb{E}_{\bm{Z}_{1}, \bm{Z}_{2}} \left[ n_{k,l} \right] 
	   \right)
  -
  \Psi \left( a_{k,l} \right)
  }{
  \Psi \left( 
		a_{k,l} + b_{k,l} 
		+ \mathbb{E}_{\bm{Z}_{1}, \bm{Z}_{2}} \left[ n_{k,l} \right] 
		+ \mathbb{E}_{\bm{Z}_{1}, \bm{Z}_{2}} \left[ N_{k,l} \right] 
	   \right) 
  -
  \Psi \left( a_{k,l} + b_{k,l} \right) 
  } \, .
  \label{eq:IRM_2dom_CVB_lowerbound_a_kl_update}  
\end{equation}
\begin{equation}
 \hat{b}_{k,l} 
  =
  b_{k,l}
  \frac{
  \Psi \left( 
		b_{k,l} + 
		\mathbb{E}_{\bm{Z}_{1}, \bm{Z}_{2}} \left[ N_{k,l} \right] 
	   \right)
  -
  \Psi \left( b_{k,l} \right)
  }{
  \Psi \left( 
		a_{k,l} + b_{k,l} 
		+ \mathbb{E}_{\bm{Z}_{1}, \bm{Z}_{2}} \left[ n_{k,l} \right] 
		+ \mathbb{E}_{\bm{Z}_{1}, \bm{Z}_{2}} \left[ N_{k,l} \right] 
	   \right) 
  -
  \Psi \left( a_{k,l} + b_{k,l} \right) 
  } \, .
  \label{eq:IRM_2dom_CVB_lowerbound_b_kl_update}  
\end{equation}

As evident, the update rules for CVB are very similar to those of 
VB (Eqs. (\ref{eq:IRM_2dom_VB_posterior_lowerbound_alpha_1_update}-\ref{eq:IRM_2dom_VB_posterior_lowerbound_b_kl_update})).
The only difference is that 
the CVB update rules incorporate the (approximated) counts, while 
the VB update rules use the VB posterior parameters instead. 

%%%%%%%%%%%%%%%%%%%%%%
%% convergence
%%%%%%%%%%%%%%%%%%%%%%
\section{Convergence Detection of CVB}
It is theoretically guaranteed that 
each iteration of VB monotonically increases the
variational lower bound~\citep{Attias00,Bishop06}, 
and VB eventually converges to its local optimal solutions in finite iterations. 
Thus, the VB inference yields easy detection of convergence; 
the algorithm automatically detects (technically sound) convergence by
computing and monitoring the variational lower bound 
(\myeqref{eq:VB_lowerbound_def}) for each iteration. 

Unfortunately, no theoretical guarantee of CVB convergence has been provided so far. 
This is due to the fact that we cannot correctly evaluate 
the posterior expectations over $\bm{Z}$ as we have seen. 
What we try to find in the CVB solutions is a stationary point of a
Taylor-approximated CVB lower bound; 
thus, we are not sure that the procedure actually
monotonically improves the true lower bound. 
Convergence analysis of CVB inferences remains an important open problem
in the machine learning field. 
However, the problem of CVB convergence is not so much discussed in the
literature since CVB inference yields better posterior estimations in many cases. 

Convergence analysis of CVB remains an important but difficult problem. 
Instead of tackling this problem directly, 
we study two aspects of CVB convergence in this paper. 
First, we empirically study the convergence behaviors of CVB (in IRM)
by monitoring a couple of quantities: a naive VB lower bound and the
pseudo leave-one-out log likelihood. 
We found that these two quantities are potentially useful for CVB
convergence detection. 
Second, we propose a simple and effective technique to 
assure the automatic detection of CVB inference convergence called 
Averaged CVB (ACVB). 
We also prove that the ACVB reaches the stationary point of
the true CVB lower bound, if it exists. 

From the non-expert user's viewpoint, 
it is highly preferable if we can devise an easy convergence
detection algorithm for CVB in general (not restricted to IRM). 
For many practitioners, it is not easy to manually determine the convergence 
of the inference algorithms. 
This might be a part of reasons why Maximum Likelihood estimators and EM-based algorithms
are preferred. 
Therefore, convergence-guaranteed ACVB would allow users to use CVB inference,  
which is more precise than naive VB in theory, 
with automatic computation termination at guaranteed convergence. 

To the best of our knowledge, \citep{Foulds13} is the only work that proposes 
convergence-assured CVB inference so far. 
This model is based on the Robbins and Monro stochastic approximation~\citep{Robbins_Monro51} 
and is only valid for LDA-CVB0. 
More precisely, the solution presented in \citep{Foulds13} is a MAP solution, leveraging the fact that 
the MAP solution is very similar to the CVB0 solution in the case of LDA. 
On the contrary, ACVB proposed in this paper is valid for any probabilistic model, and for both CVB and CVB0. 
The \citep{Foulds13} approach is no longer valid for IRM because the MAP solution and the CVB0 solution are different.  

% Quantities to be monitored 
\subsection{Assessing Candidate Quantities for CVB Convergence Detection}
We cannot correctly compute the true CVB lower bound in 
~\myeqref{eq:IRM_2dom_CVB_lowerbound_def}. 
Therefore, our first approach would be to find some quantities that 
serve as the proxies of the true CVB lower bound. 

For that purpose, we examine two quantities. 
The first candidate is the naive VB lower bound in~\myeqref{eq:VB_lowerbound_def}. 
The VB lower bound is a lower bound of the true CVB lower bound. 
The variational posteriors of parameters are computed using
$q(\bm{Z})$, which is obtained by the CVB inference. 
Using these variational posteriors, 
we may compute the VB lower bound as a proxy of the true CVB lower bound. 

The second one is the pseudo leave-one-out log likelihood 
of the training data set.  
The CVB solutions (and the Gibbs sampler) compute the predictive
distribution of an object, say, $z_{1,i}$, in a leave-one-out manner. 
Therefore, we might be able to detect the convergence of the CVB
inference by watching these predictive distributions. 
More precisely, 
we normalize \myeqref{eq:IRM_2dom_CVB_lowerbound_z_1i_posterior} sum to
one, in computing $q(z_{1,i})$. 
The normalize constant, namely 
$C_{1,i} = \sum_{k} q\left(z_{1,i,k}\right)$ serves as a pseudo
leave-one-out log likelihood of the object $(1,i)$ given the model. 
Then, the whole sum of this log likelihood, 
$C = \sum_{i} C_{1,i} + \sum_{j} C_{2,j}$, 
is the pseudo log likelihood of the training data set. 

Figure \ref{fig:convergence_CVB} and \ref{fig:convergence_CVB0}
present the evolutions of these two quantities in 
CVB (\myfigref{fig:convergence_CVB}) 
and CVB0 (\myfigref{fig:convergence_CVB0}), respectively. 
Hyperparameters are set to the best setup found in our experimental
validations (see the Experiment section). 
As evident from the cites, both of the quantities converge
within a few iterations. 

% naive VB is interesting, but very heavy. not recomended
It is interesting to see that the naive VB lower bound (presented in solid
lines) increases as the CVB inference proceeds. 
The CVB inference does not increase the VB lower bound directly, which
is a looser bound than that of CVB. However, the learned model
actually decreases the discrepancy between the variational posteriors and
the true posteriors, resulting in increasing VB lower bound. 
Unfortunately, the computational load is considerably heavy compared to
the CVB updates. 
% LOO is more or like evidence. No additional computational cost for
% CVBs. Good. 
In contrast, pseudo leave-one-out training log likelihood (presented in
dashed lines) costs no extra computational loads from the original CVB updates. 
Statistically, the property of the quantity is more or less similar to the 
model evidence; thus the quantity would be a good choice for convergence
detection. 

\begin{figure*}[t]
 \begin{center}
  \includegraphics[width=70mm]{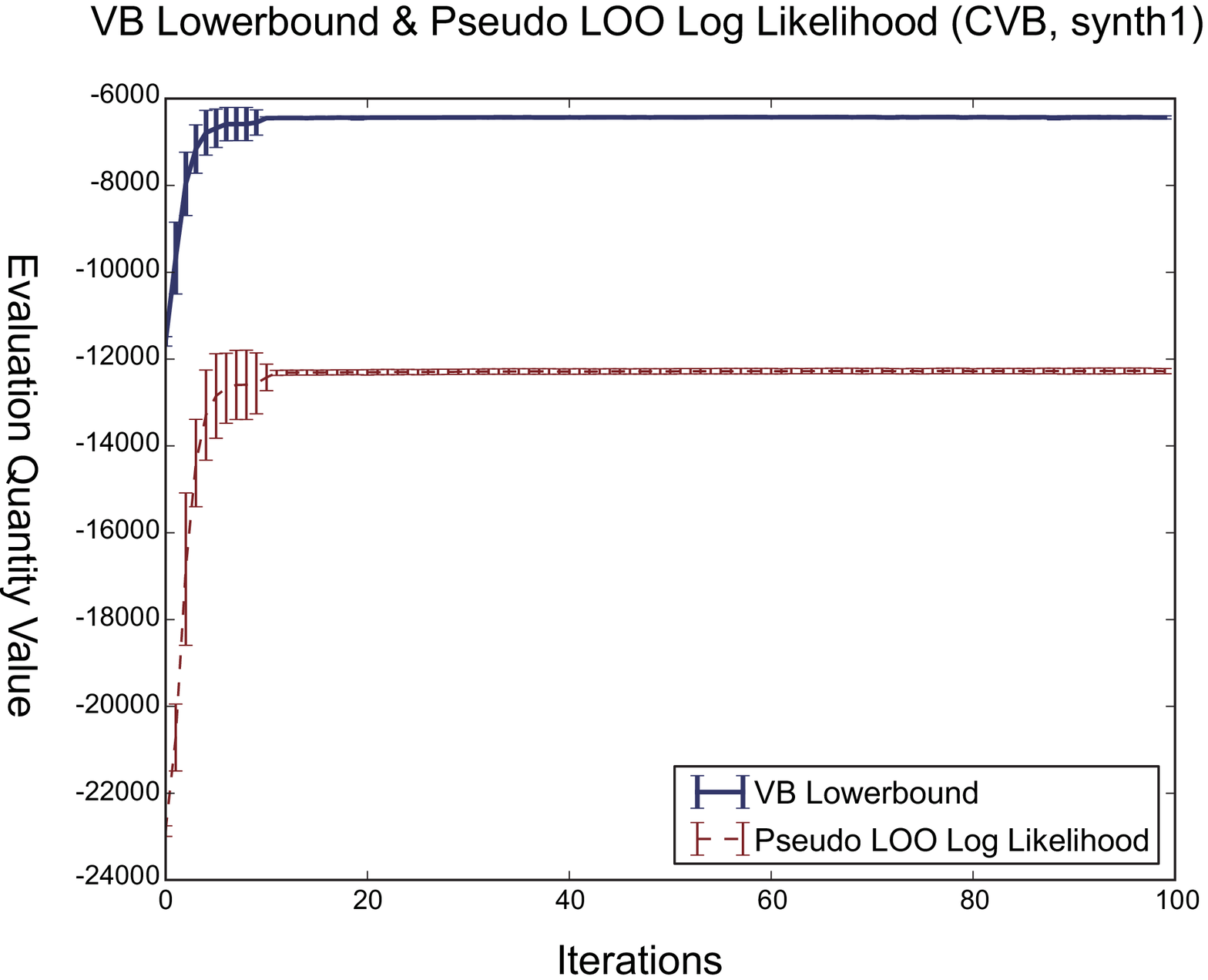}
  \includegraphics[width=70mm]{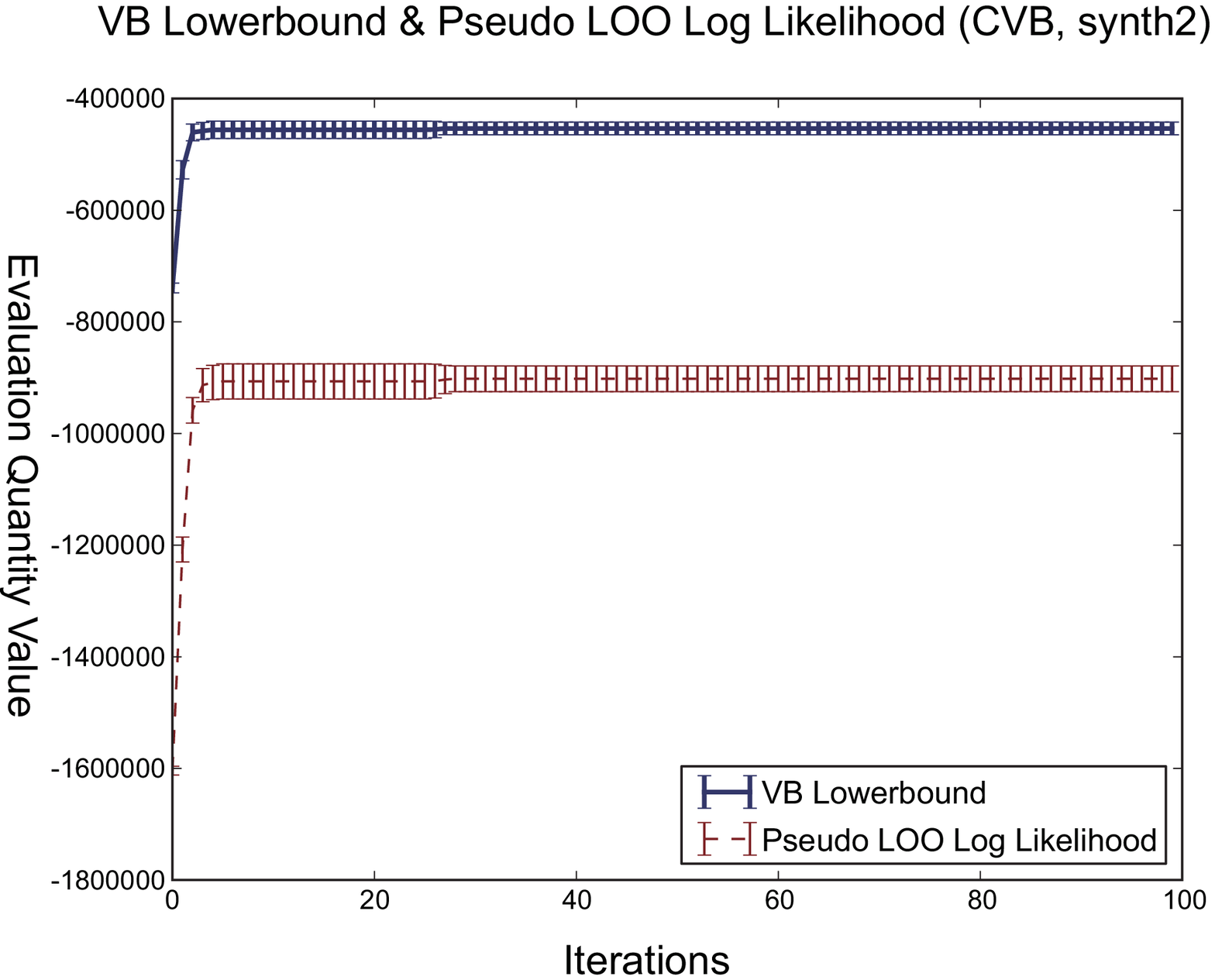}
 \end{center}
 \caption{Evolutions of two quantities over CVB iterations. Solid lines
 indicate the evolutions of naive VB lower bound. Dashed lines indicate
 the evolution of pseudo leave-one-out log likelihood on training
 data. Error bars denote the standard deviations. Left:
 computed on Synth1 dataset. Right: computed on Synth2 dataset. For
 details of the datasets, see the experiment section. }
 \label{fig:convergence_CVB}
\end{figure*}

\begin{figure*}[t]
 \begin{center}
  \includegraphics[width=70mm]{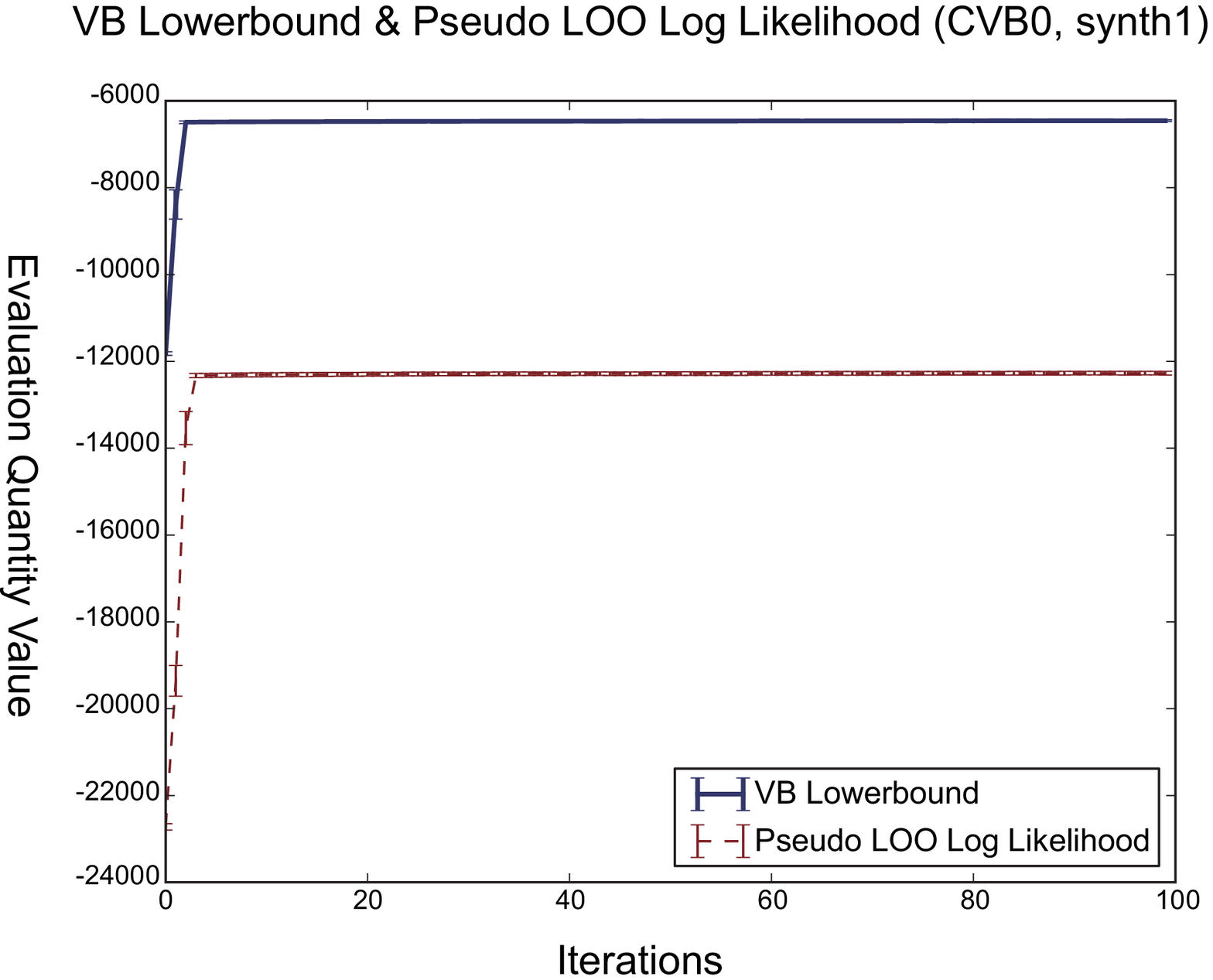}
  \includegraphics[width=70mm]{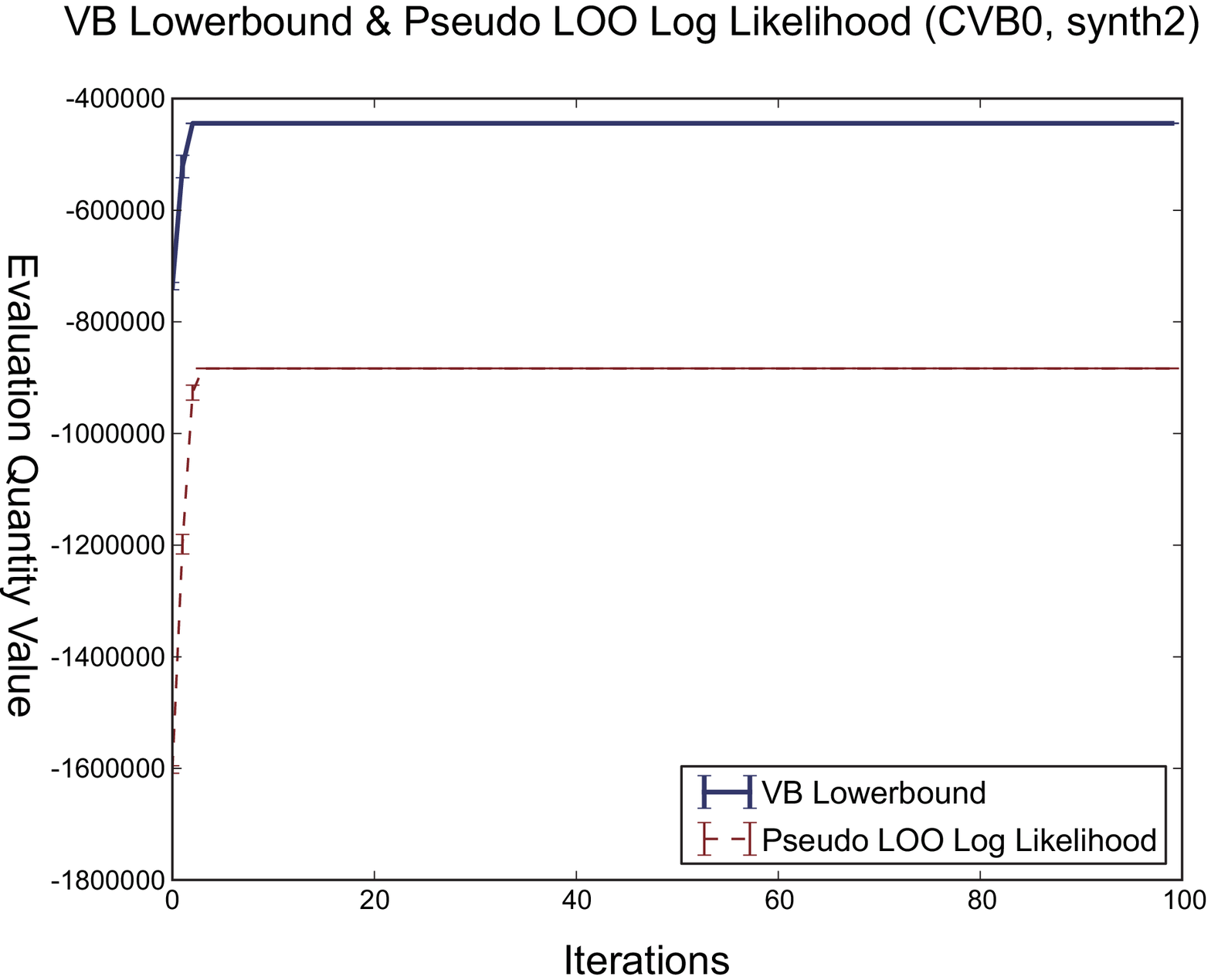}
 \end{center}
 \caption{Evolutions of two quantities over CVB0 iterations. Solid lines
 indicate the evolutions of naive VB lower bound. Dashed lines indicate
 the evolution of pseudo leave-one-out log likelihood on training
 data. Error bars denote the standard deviations. Left:
 computed on Synth1 dataset. Right: computed on Synth2 dataset. For
 details of the datasets, see the experiment section. }
 \label{fig:convergence_CVB0}
\end{figure*}

%%%
% A-CVB start from here
%%%
\subsection{Averaged CVB: Convergence technique for general CVB}
Next, we propose a more direct and convergence-guaranteed 
technique for general CVB inferences. 
The technique is based on monitoring the changes of $q(\bm{Z})$. 
The rationale is simple: 
it is reasonable to monitor $q(\bm{Z})$ since 
the CVB solutions try to obtain the stationary point of the Taylor
approximated lower bound with respect to $q(\bm{Z})$. 

Then, we develop a simple annealing technique, 
called Averaged CVB (ACVB) to assure the convergence of CVB solutions.  
We would like to emphasize that the discussion of ACVB is not limited to the IRM: 
this technique is applicable to CVB inference on any model. 
Also, ACVB is valid for CVB (2nd order) and CVB0 equally. 

%Let $q(z_i)$ be the variational posterior for latent variable $z_i$ corresponding to data $x_i$ in the CVB inference.
%For simplicity, $q^{(s)}$ denotes $q^{(s)}(z_i)$ which is the $s$-th step update of $q(z_i)$ by using the CVB inference.
After a certain number of iterations for ``burn-in'', 
we gradually decrease the portion of variational posterior changes 
using the following equation: 
\begin{equation}
 \bar q^{(s+1)}
  =
  \left(1-\frac{1}{s+1} \right) \bar q^{(s)} 
  + 
  \frac{1}{s+1} q^{(s+1)} \, ,
\, \, \, \text{or} \, \, \,\ 
\bar q^{(S)}=\frac{1}{S}\sum_{s=1}^{S} q^{(s)},
\label{eq:ACVB}
\end{equation}
where 
$s$ denotes the iterations after completion of the ``burn-in'' period, 
$\bar{q}^{(s)}$ denotes the ``annealed'' variational posterior at the 
$s$th iteration, 
$q^{(s)}$ denotes the variational posterior by CVB inference at
the $s$th iteration, 
and $S$ is the total number of iterations. 
After the ``burn-in'' period, we monitor the ratio of changes of $\bar{q}$ and 
detect convergence when the ratio falls below a predefined
threshold. As the final result, 
we do not finally use $q^{(s)}$ but $\bar q^{(s)}$. 
During the burn-in period, we monitor the changes of $q$: in most cases,
$q$ quickly converges before entering the annealing process. 

Concerning the convergence of ACVB, there are two points to note. 
The first point is rather evident but makes the ACVB useful for practical CVB inference. 
ACVB updates assure convergence, 
and we can easily detect the convergence by taking the difference of $\bar{q}$ 
in successive iterations. 
\begin{theorem}
\label{theorem:acvb:convergence_obvious}
The averaged variational posterior $\bar q^{(s)}$ is convergence-assured: 
$\forall \epsilon >0,~\exists S_0,~\text{s.t.}~\forall S>S_0 \Rightarrow \frac{1}{N}\sum_{i=1}^{N}\left|\bar{q}_i^{(S)}-\bar{q}_i^{(S-1)} \right|<\epsilon$. 
\end{theorem}

\begin{proof}
Since
\begin{align*}
\frac{1}{S}\sum_{s=1}^{S}q^{(s)}=\left(1-\frac{1}{S}\right) \frac{1}{S-1}\sum_{s=1}^{S-1}q^{(s)}+\frac{1}{S} q^{(S)},
\end{align*}
we have
\begin{align*}
\left|\frac{1}{S}\sum_{s=1}^{S}q^{(s)}- \frac{1}{S-1}\sum_{s=1}^{S-1}q^{(s)}\right|
&=\left|-\frac{1}{S} \frac{1}{S-1}\sum_{s=1}^{S-1}q^{(s)}+\frac{1}{S} q^{(S)}\right| \\
&\leq \frac{1}{S} \frac{1}{S-1}\sum_{s=1}^{S-1}|q^{(s)}|+\frac{1}{S} |q^{(S)}|\\
&\leq \frac{1}{S} \frac{1}{S-1}(S-1)+\frac{1}{S}=\frac{2}{S}.
\end{align*}
Thus,
\begin{align*}
\frac{1}{N}\sum_{i=1}^{N}\left|\frac{1}{S}\sum_{s=1}^{S}q_i^{(s)}-\frac{1}{S-1} \sum_{s=1}^{S-1}q_i^{(s)} \right|\leq \frac{2}{S}.
\end{align*}
If we set $S_0=\frac{2}{\epsilon}$, then $\forall S>S_0$,
\begin{align*}
\frac{1}{N}\sum_{i=1}^{N}\left|\frac{1}{S}\sum_{s=1}^{S}q_i^{(s)}-\frac{1}{S-1} \sum_{s=1}^{S-1}q_i^{(s)} \right|\leq \frac{2}{S} < \frac{2}{S_0}=\epsilon.
\end{align*}
This means 
\begin{align*}
\frac{1}{N}\sum_{i=1}^{N}\left|\bar{q}_i^{(S)}-\bar{q}_i^{(S-1)} \right| < \epsilon.
\end{align*}
\end{proof}
Thus, we can automatically stop the ACVB inference by using a stopping rule based on the difference of ACVB posteriors. 

The second point is much noteworthy and validates the use of ACVB in Bayesian inference: 
we can prove that the converged $\bar{q}$ is asymptotically equivalent to the
stationary point of the CVB lower bound, if it exists 
(note that it
is not clear whether the true CVB lower bound has a stationary point in theory). 
\begin{theorem}
\label{theorem:acvb:convergence}
If the variational posterior $q^{(s)}$ converges to a stationary point in the CVB lower bound, then the averaged variational posterior $\bar q^{(s)}$ also converges to a stationary point in the CVB lower bound.
\end{theorem}

\begin{proof}
Let $q^{*}$ be a stationary point in the CVB lower bound.
By using this assumption,
\begin{align*}
\lim_{s \rightarrow \infty} q^{(s)}=q^{*} \Leftrightarrow \forall \epsilon > 0, \exists s_0~\text{s.t.}~\forall s>s_0 \Rightarrow |q^{(s)}-q^{*}|<\epsilon/2.
\end{align*}
Here, we define
\begin{align*}
\left|\sum_{s=1}^{s_0} (q^{(s)}-q^{*})\right|=M>0,
\end{align*}
and thus,
\begin{align*}
\lim_{s\rightarrow\infty}\frac{M}{s}=0 \Leftrightarrow \forall \epsilon > 0, \exists s_0'~\text{s.t.}~\forall s>s_0' \Rightarrow \frac{M}{s}<\epsilon/2.
\end{align*}

When $S_0=\max \{s_0,s_0'\}$,  we have
\begin{align*}
 \forall S > S_0, |\bar q^{(S)}-q^{*}|&=\left|
 \sum_{s=1}^{S}\frac{1}{S}(q^{(s)}-q^{*})
 \right| \\
 &< \frac{M}{S}+\sum_{s=S_0+1}^{S} 
 \left| \frac{1}{S}(q^{(s)}-q^{*}) \right| 
 \\
&\leq \epsilon/2+\left|\frac{S-S_0}{S}\right|\epsilon/2 
 \leq \epsilon/2+\epsilon/2=\epsilon.
\end{align*}
Therefore,
\begin{align*}
\lim_{s \rightarrow \infty} \bar q^{(s)}=q^{*}.
\end{align*}
\end{proof}
We want to stress that it remains unknown in the literature as to whether 
CVB inference has a stationary point. 
However, we can still safely use ACVB because it assures convergence of the 
inference process and ACVB will find the "true" solution if CVB has a stationary point. 
Such solutions for the convergence of CVB have been never studied, to the best of our knowledge. 

Hereafter, we denote the (naive) CVB solution and the CVB0 solution, both with ACVB, as the \textbf{ACVB} solution and the \textbf{ACVB0} solution, respectively.

\section{Speeding up CVB inferences of IRM}
Convergence assurance by ACVB is beneficial for users because it enables 
automatic and easy detection of inference convergence. 
However, the computational speed is also an important factor for practical uses. 
In this section, we introduce two possible speed-up techniques for IRM-(A)CVB solutions. 

\subsection{Cluster shrinkage}
One drawback of CVB (and VB) inference is the computational cost per iteration. 
The Gibbs sampler dynamically shrinks and expands the cardinality of hidden clusters during inference.  
Thus, the computational cost of the Gibbs sampler is a function of the complexity of the hidden clusters. 
However, (C)VB solutions maintain all $K$ clusters throughout the inference. Thus, there is no shrinkage of clusters in CVB, 
which results in heavy computational costs, larger than the
intrinsic complexity of the data. 

A simple solution is to ignore small clusters. 
As described in~\citep{Kurihara07}, (C)VB inferences need to 
maintain the order of clusters so that the memberships of clusters are
aligned in descending order for better performance, in accordance with
the stick-breaking process. 
Therefore, it is easy to implement a heuristic that 
i) excludes small clusters from variational posterior updates, 
and 
ii) avoids evaluating the contributions of small clusters in the updates. 

For example, we can evaluate how small the cluster $k$ is by 
$\frac{m_{k}}{\sum_{l} m_{l}}$. 
In practice, the small clusters become truly negligible in the sense of
memberships: for example, $\frac{m_{k}}{\sum_{l} m_{l}} \sim 1.0 \times 10^{-5}$ or smaller. 
Setting a very small value for the threshold, 
the size of {\it effective} clusters automatically shrinks to the
size of the intrinsic data complexity. 
Then, the computational cost is reduced proportional to the size of
effective clusters. This dramatically speeds up the inference while
barely harming the inference performance. 
In all experiments, we implemented this heuristic to eliminate unnecessary
computational costs for VB, CVB and CVB0 solutions. 

\subsection{Linear time inference for (A)CVB0}
A naive implementation of IRM-CVB inference requires 
$O\left( N_{1} N_{2} K_{1} K_{2} \right)$ for 
one full sweep of hidden variable $\bm{Z}$. 
We can see this instantly from 
the expectations $\mathbb{E}$ and variances $\mathbb{V}$ of 
$n_{k,l}^{\setminus (1,i)}$ and $N_{k,l}^{\setminus (1,i)}$ in 
Eqs. (\ref{eq:IRM_2dom_CVB_nN_1kl_approx_plus}, \ref{eq:IRM_2dom_CVB_nN_1kl_approx_var_plus}). 
For update of $q\left(z_{1,i,k}\right)$ on specific $i$ and $k$, 
we need to evaluate these expectations and evaluations for $K_{2}$ times. 
This requires $O\left( N_{2} K_{2} \right)$ computations; 
thus the full sweep for $N_{1}$ objects on $K_{1}$ clusters requires 
$O\left( N_{1} N_{2} K_{1} K_{2} \right)$.  
The same holds for $q(\bm{Z}_{2})$. 

This prohibits applying IRM to larger data. 
However, 
%for CVB0 and Collapsed Gibbs sampler, 
for the (A)CVB0 solution, 
we can reduce to
%the computational cost to 
$O\left( L (N_{1} + N_{2}) K_{1} K_{2} \right)$ where $L$ denotes the average degree ``1'' links of objects, 
without any approximation. 
This is remarkable: we can solve IRM linear to the number of objects. 
Further, many real-world relational data are very sparse: $L$ is small. 
This makes (A)CVB0 even more efficient. 
This is almost just an implementational issue, but we believe it is very beneficial for 
readers that are interested in IRM for the first time. 

To obtain this, we rewrite the variational expectation terms in 
\myeqref{eq:IRM_2dom_CVB_nN_1kl_approx_plus} in the following way: 
\begin{alignat}{2}
\mathbb{E}[n_{k, l}^{+(1,i,k)}] 
 &= 
 \sum_{j=1}^{N_{2}}  
 q\left(z_{2,j,l}\right) x_{i,j} 
 = 
 \sum_{j \in J^{+}} q\left(z_{2,j,l}\right) \, , 
 \label{eq:IRM_2dom_CVB_LinearTime_n}
 \\
 \mathbb{E}[N_{k, l}^{+(1,i,k)}] 
  &= 
  \sum_{j=1}^{N_{2}}  
  q\left(z_{2,j,l}\right) \left( 1 - x_{i,j} \right)  
  = 
  \sum_{j \in J^{-}} q\left(z_{2,j,l}\right) \, .
  \label{eq:IRM_2dom_CVB_LinearTime_N}
\end{alignat}
where $J^{+} = \left\{ j: x_{i,j} = 1 \right\}$ 
and $J^{+} = \left\{ j: x_{i,j} = 0 \right\}$.  
It is evident that $J^{+} \cup J^{-} = \{1, 2, \dots, N_{2}\}$. 
The key observation is: 
\begin{equation}
\mathbb{E}[n_{k, l}^{+(1,i,k)}] 
+
\mathbb{E}[N_{k, l}^{+(1,i,k)}] 
=
\sum_{j \in J^{+}} q\left(z_{2,j,l}\right)
+
\sum_{j \in J^{-}} q\left(z_{2,j,l}\right)
=
\sum_{j=1}^{N_{2}} q\left(z_{2,j,l}\right)
=
\mathbb{E}[m_{2, l}] \, .
\label{eq:IRM_2dom_CVB_LinearTime_m}
\end{equation}
The right-most term is a membership count of clusters in the second domain, 
defined in \myeqref{eq:IRM_2dom_CVB_m_1k_2l_E}, 
which can be cached during inference. 
To compute \myeqref{eq:IRM_2dom_CVB_LinearTime_n}, we only need to 
take $L$ (the average degree) computations,  
which is much smaller than $N_{2}$ for many real-world relational data. 
Combined with \myeqref{eq:IRM_2dom_CVB_LinearTime_m}, 
we can evaluate \myeqref{eq:IRM_2dom_CVB_nN_1kl_approx_plus},  
namely Eqs. (\ref{eq:IRM_2dom_CVB_LinearTime_n}, \ref{eq:IRM_2dom_CVB_LinearTime_N}) 
by $O\left(L\right)$, instead of $N_{2}$. 
Thus, one full sweep of $q(\bm{Z})$ computation is reduced to 
$O\left(L (N_{1} + N_{2}) K_{1} K_{2}\right)$. 

This linear-time inference algorithm has two limitations. 
One is that this does not allow missing entries within $\bm{X}$. 
More precisely, we can run the algorithm, but the resulting inference 
is not accurate. 
This is because the above linear-time inference cannot correctly evaluate the existence of missing data. 
If we encounter the relational data with missing entries, we need some preprocessing 
to impute the missing entries. 
Another limitation is that this algorithm is not applicable to the CVB solution: 
the posterior variation $\mathbb{V}[n]$ requires square terms of $q(z)$, 
for which we cannot use the trick of \myeqref{eq:IRM_2dom_CVB_LinearTime_m}. 

In the case of collapsed Gibbs, we can implement it in a similar way 
by replacing $q(z)$ with $\mathbb{I}(z)$ on the current sample of $\bm{Z}$. 
For VB, the update algorithm is very different from the others, but 
in general, VB allows massive parallelization on the updates of $q(z)$. 

\citep{Hansen11,Albers13}, which employed collapsed Gibbs, 
presented a different representation to avoid direct computations of negative counts ($N$ in this paper). 
However, they did not analyze the order of linear inference computations, 
nor did not emphasize the usefulness of the sparsity. 

%%%%%%%%%%%%%
% Experiments
%%%%%%%%%%%%%
\section{Experiments}
In this section, we present the experimental validations. 
In summary, we confirmed the following facts. 
\begin{enumerate}
\item ACVB inferences achieved better modeling performances than the naive VB in large relational data sets. No significant differences in smaller relational data. 
\item The ACVB0 solution is the fastest among deterministic inferences. 
\item ACVB0 with linear time computation scales very well agasint large relational data. 
\end{enumerate}

\subsection{Procedure}
We compare the performance of 
proposed Averaged CVB solutions (\textbf{ACVB}, \textbf{ACVB0}) 
with a naive variational Bayes (\textbf{VB}), 
which is a baseline deterministic inference. 
As a reference, we also include comparisons with the collapsed Gibbs samplers (\textbf{Gibbs}) with 
very small number of iterations. 

Initializations and hyperparameter choices are important for 
fair comparisons of inference methods. 
We employ hyperparameter updates for 
all solutions: fixed point iterations for VB, ACVB, and ACVB0 and 
hyper-prior sampling for Gibbs. 
We test several initial hyperparameter values, and 
report the results computed on the best hyperparameter setting. 
All hidden variables are initialized in a completely random manner: 
we use the uniform distribution to assign soft values of $p(z_{i} =k)$.  
In the case of Gibbs, we perform hard assignments of $z_{i} = k$ to the 
most weighted cluster. 
For VB, ACVB, and ACVB0 solutions, we normalized the assigned weights 
sum to one over clusters. 
All solutions without Gibbs 
require the number of truncated clusters a priori. 
To assess the effect of truncation level, the experiments examined 
$K_{1} = K_{2} = K \in \{20, 40, 60\}$. 
In practice, we just need to prepare a sufficient number of $K$ 
to explain data complexity. 

% marginal?
Data modeling performance is evaluated by averaged 
test data marginal log likelihood. 
Given a relational data matrix, 
we exclude a part of the relational observations (roughly 
$10\%$ of matrix entries) from the inference as held-out test data. 
After the inference is finished, we compute marginal log likelihoods 
of these test data. 
The number of test data, and chosen data entries are randomized for
every evaluation run. Thus, we average the log likelihoods by the number
of test data entries. 
A per-test data entry log likelihood is computed for 20 runs with
different initializations and hyperparameter settings. 

To compare the inference solutions in terms of the computational cost, 
we also monitored the convergence behaviors and the computational time
of the solutions. 
For the VB solution, we monitored the VB lower bound.
For ACVB and ACVB0 solutions, we use the annealed posteriors.  
We determined solutions' convergence by the relative changes of the monitored quantity: 
if the changes were smaller than 
$0.001 \%$ 
of the current value of the quantity, we assumed that the
algorithm had converged. 
As stated in the speeding-up section, 
we employed a cluster shrinkage technique for VB, ACVB and ACVB0. 
We did not utilize the linear-time inference because of the existence of test data that must be kept missing during inferences.  
%We want to evaluate the test data log likelihood precisely for comparison of inference solutions. 
%we must keep the test data entries "missing" during the inference, 
%and no imputations are allowed to conduct linear time inference. 

For the reference Gibbs sampler, we iterated the sampling procedure 
$3,000$ times; the first $1,500$ iterations were discarded as the burn-in
period. As repeated explained, 
the collapsed Gibbs would require tremendous amount of iterations (some millions) 
to obtain better modeling results~\citep{Albers13}. 
However, we can not afford such computational resources for the collapsed Gibbs, 
which has no easy way to detect convergence, not preferable for practitioners. 

\subsection{Datasets}
We prepared several synthetic and real-world datasets for the experiments; 
they allow us to assess the inferences in several scales and densities.  

%%%
% synthetic data
%%%
We generated two synthetic relation datasets. 
The size and true numbers of clusters of these datasets were: 
$N_{1} = 100, N_{2} = 200, K_{1} = 4, K_{2} = 5$ (\textbf{synth 1}), 
and 
$N_{1} = 1000, N_{2} = 1500, K_{1} = 7, K_{2} = 6$ (\textbf{synth 2}). 

%%%
% real data
%%%

% Enron 151 by 151
The first real-world relational dataset is 
the \textbf{Enron} e-mail dataset~\citep{Klimt_Yang04}. 
This is a famous relational dataset 
used in many studies~\citep{Tang08,Fu09,NIPS10,AISTATS12}. 
We extracted monthly e-mail transactions for 2001. 
The dataset contained $N=N_{1}=N_{2}=151$ company members of Enron. 
$x_{i,j} = 1 (0)$ if there is (not) an e-mail sent from 
member $i$ to member $j$. 
Out of twelve months, we selected the transactions of 
June (\textbf{Enron Jun.}), 
August (\textbf{Enron Aug.}), 
October (\textbf{Enron Oct.}), 
and December (\textbf{Enron Dec.}). 

% Last FM, User by User 1892 by 1892
The second real-world relational dataset is 
the \textbf{Lastfm} 
dataset.\footnote{provided by HetRec2011. http://ir.ii.uam.es/hetrec2011/} 
This dataset contains several records for the Last.fm music service, 
including lists of users' most listened-to musicians, 
tag assignments for artists, and friend relations between users. 
We employ the friend relations between $N=N_{1}=N_{2}=1892$ users (\textbf{Lastfm UserXUser}). 
$x_{i,j} = 1 (0)$ if there is (not) a friend relation from 
a user $i$ to a user $j$. 
This dataset is 10 times larger than the \textbf{Enron} dataset in the
number of objects, and 100 times larger in the number of matrix
entries. 
% Last FM, Artist by Tags,  6099 by 1088
We also employ the artist-tag relations between $17,632$ artists and 
$11,946$ tags. 
Since the relation matrix is too large for inference of Gibbs and naive
VB, we truncate the number of artists and tags. 
The original observations are the co-occurrence counts of (artist name,
tag) pairs. 
We binarize the observations as to whether the (artist name, tag) pair counts
is greater than 1 or not: that is,  we ignore one single occasional
co-occurrence of (artist name, tag). 
If the counts are greater than 1, then the observation entries are set
to 1; otherwise, set to 0. 
Then, all rows (artists) and columns (tags) that have no ``1'' entries
are removed. 
The resulting binary matrix consists of 
$N_{1} = 6099$ artists and $N_{2} = 1088$ tags 
(\textbf{Lastfm ArtistXTag}). 
$x_{i,j} = 1 (0)$ if the artist $i$ is (not) attached with 
the tag word $j$ more than once. 

The data sizes and densities are summarized in 
\mytabref{tab:datasize}. 
The data sizes are rather small compared to CVB research in LDA~\citep{Asuncion09,Sato_Nakagawa12,Sato12}. 
One reason is that IRM deals with two hidden variables ($z_{1,i}, z_{2,j}$) for 
one observation ($x_{i,j}$), while LDA requires one hidden variable for one observation (word). 
This makes the inference difficult and hinders the scale up of the problem. 
In fact, existing IRM studies work in a somewhat similar volume of datasets~\citep{Kemp06,NIPS10,AISTATS12}. 
Also, we cannot implement linear time inference when we have missing data (test data) for 
model evaluations. 

\begin{table}[t]
 \caption{Data sizes used in our experiments. }
 \label{tab:datasize}
 \begin{center}  
  \begin{tabularx}{110mm}{X||c|c|c|c}
   Dataset & $N_{1}$ & $N_{2}$ & \# of ``1'' entries & Density \\
   \hline \hline
   synth1 & 100 & 200 & 7608 & 38.0\% \\
   synth2 & 1000 & 1500 & 578802 & 38.6\% \\ 
   \hline
   Enron Jun. & 151 & 151 & 257 & 1.13\% \\
   Enron Aug. & 151 & 151 & 439 & 1.93\% \\
   Enron Oct. & 151 & 151 & 707 & 3.10\% \\
   Enron Dec. & 151 & 151 & 377 & 1.66\% \\
   \hline
   Lastfm UserXUser & 1892 & 1892 & 21512 & 0.60\% \\
   Lastfm ArtistXTag & 6099 & 1088 & 23253 & 0.35\% 
  \end{tabularx}
 \end{center}
\end{table}

\subsection{Results}

\subsubsection{Numerical Performance}
The modeling performances of the solutions are presented in 
\mytabref{tab:results_loglk_K20} ($K = 20$), 
\mytabref{tab:results_loglk_K40} ($K = 40$), 
and \mytabref{tab:results_loglk_K60} ($K = 60$).  
They show the averages of 
test data marginal log likelihood after convergence. 
Results of the best setup are presented for each solution. 
% statistic tests 
In addition, we conducted statistical significance tests using 
$t$-tests. 
%Based on the tables and the $t$-tests, 
%we found that the ACVB0 solution is significantly better 
%(in $p < 0.05$) than other methods in most datasets and $K$ choices. 
%In some cases, the naive ACVB showed the best results, 
%and collapsed Gibbs was the best only on \textbf{Lastfm ArtisrtXTag} at $K=60$. 

\begin{table}[t]
\caption{Marginal test data log likelihood per test data
 entry ($K=20$, $10\%$ test data). Parenthesized numbers indicate standard deviations. 
 Larger values are better. 
 Boldfaces indicate the best method, which is significantly better 
 than the method(s) marked with $*$ (by $t$-test, $p=0.05$). 
 }
\label{tab:results_loglk_K20}
 \begin{center}
  \begin{tabularx}{155mm}{X||c|c|c|c}
   Dataset & Gibbs & VB & ACVB & ACVB0 \\ \hline \hline
   Synth1 & -0.3696* (0.0282) & \textbf{-0.3260} (0.0155) & -0.3337 (0.0120) & -0.3372* (0.0163)\\ \hline 
   Synth2 & -0.3721* (0.0191) & -0.3737* (0.0090) & -0.3348* (0.0107)  & \textbf{-0.3261} (0.0016)\\ \hline 
   Enron Jun. & -0.0585 (0.0107) & -0.0547 (0.0087) & -0.0559 (0.0086)  & \textbf{-0.0540} (0.0064)\\ \hline 
   Enron Aug. & -0.0830* (0.0109) & -0.0789 (0.0076) & -0.0766 (0.0095)  & \textbf{-0.0763} (0.0081)\\ \hline 
   Enron Oct. & -0.1268* (0.0103) & -0.1164* (0.0091) & -0.1098 (0.0095)  & \textbf{-0.1099} (0.0107)\\ \hline 
   Enron Dec. & -0.0740 (0.0119) & -0.0693 (0.0068) & -0.0686 (0.0107)  & \textbf{-0.0685} (0.0088)\\ \hline
   Lastfm (UserXUser) & -0.0283* (0.0006) & -0.0287* (0.0005) & -0.0271* (0.0005) & \textbf{-0.0267} (0.0005)\\ \hline 
   Lastfm (ArtistXTag) & -0.0160* (0.0003) & -0.0165* (0.0003) & -0.0161* (0.0002) & \textbf{-0.0158} (0.0003) 
  \end{tabularx}
 \end{center}
\end{table}

\begin{table}[t]
\caption{Marginal test data log likelihood per test data
 entry ($K=40$, $10\%$ test data). Parenthesized numbers indicate standard deviations. 
 Larger values are better. 
 Boldfaces indicate the best method which is significantly better 
 than the method(s) marked with $*$ (by $t$-test, $p=0.05$). 
 }
\label{tab:results_loglk_K40}
 \begin{center}
  \begin{tabularx}{155mm}{X||c|c|c|c}
   Dataset & Gibbs & VB & ACVB & ACVB0 \\ \hline \hline
   Synth1 & -0.3657* (0.0224) & \textbf{-0.3246} (0.0141) & -0.3430* (0.0197) & -0.3380* (0.0146) \\ \hline 
   Synth2 & -0.3712* (0.0122) & -0.3743 (0.0108)* & -0.3285* (0.0050) & \textbf{-0.3254} (0.0015) \\ \hline 
   Enron Jun. & -0.0595* (0.0110) & -0.0541 (0.0107) & \textbf{-0.0531} (0.0068) & -0.0558 (0.0070) \\ \hline 
   Enron Aug. & -0.0838* (0.0088) & -0.0795 (0.0088) & -0.0770 (0.0084) & \textbf{-0.0766} (0.0072) \\ \hline 
   Enron Oct. & -0.1256 (0.0111) & -0.1143 (0.0112) & \textbf{-0.1141} (0.0125) & -0.1145 (0.0115) \\ \hline 
   Enron Dec. & -0.0750* (0.0095) & \textbf{-0.0672} (0.0062) & -0.0688 (0.0099) & -0.0678 (0.0101) \\ \hline
   Lastfm (UserXUser) & -0.0280 (0.0008)* & -0.0289* (0.0004) & -0.0272* (0.0005) & \textbf{-0.0267} (0.0004) \\ \hline 
   Lastfm (ArtistXTag) & \textbf{-0.0161} (0.0003) & -0.0167* (0.0004) & -0.0162 (0.0003) & -0.0162 (0.0003) 
  \end{tabularx}
 \end{center}
\end{table}

\begin{table}[t]
\caption{Marginal test data log likelihood per test data
 entry ($K=60$, $10\%$ test data). Parenthesized numbers indicate standard deviations. 
 Larger values are better. 
 Boldfaces indicate the best method, which is significantly better 
 than the method(s) marked with $*$ (by $t$-test, $p=0.05$). 
 }
\label{tab:results_loglk_K60}
 \begin{center}
  \begin{tabularx}{155mm}{X||c|c|c|c}
   Dataset & Gibbs & VB & ACVB & ACVB0 \\ \hline \hline
   Synth1 & -0.3668* (0.0155) & \textbf{-0.3281} (0.0127) & -0.3379* (0.0154)  & -0.3452* (0.0113) \\ \hline 
   Synth2 & -0.3624* (0.0190) & -0.3736* (0.0090)  & -0.3261 (0.0015)  & \textbf{-0.3258} (0.0015) \\ \hline 
   Enron Jun. & -0.0573 (0.0087) & -0.0569 (0.0069)  & -0.0572 (0.0103)  & \textbf{-0.0563} (0.0064) \\ \hline 
   Enron Aug. & -0.0862* (0.0102) & -0.0772 (0.0098)  & -0.0781 (0.0107)  & \textbf{-0.0754} (0.0105) \\ \hline 
   Enron Oct. & -0.1281 (0.0148) & -0.1162 (0.0100)  & -0.1145 (0.0115)  & \textbf{-0.1139} (0.0096) \\ \hline 
   Enron Dec. & -0.0794* (0.0120) & \textbf{-0.0682} (0.0098)  & -0.0686 (0.0109)  & -0.0690 (0.0124) \\ \hline
   Lastfm (UserXUser) & -0.0283 (0.0006)* & -0.0287* (0.0005)  & -0.0272 (0.0006)*  & \textbf{-0.0267} (0.0006) \\ \hline 
   Lastfm (ArtistXTag) & \textbf{-0.0160} (0.0003) & -0.0167* (0.0003)  & -0.0163* (0.0002)  & -0.0163* (0.0003) 
  \end{tabularx}
 \end{center}
\end{table}

%%% Discussion of the results
These results reveal characteristics of the solutions in a few aspects. 

First, ACVB inferences are significantly better than VB for larger datasets: 
synth2, and two Lastfm datasets. 
Especially, we confirmed that ACVB0 always performed significantly better than VB, 
and often recorded significantly better results than ACVB for those datasets. 
This indicates that in potential ACVB inferences are superior to the naive VB inference as expected. 

Second, we found no advantages of ACVB inferences over VB for smaller datasets: synth1 and Enron datasets. 
Specifically, the VB performed significantly better than ACVB solutions in synth1 data. 
The data is an artificial, dense and small cross-domain relation data. 
In such cases, the VB still may finds good estimations of the true parameters $\theta$. 
If so, VB may obtain better test data log likelihood since ACVB marginalizes out 
all possibilities of the parameters, including "bad" estimations. 
Anyway, the synth1 data set is very small and dense. 
In general, we don't face such data in our practical data analysis thus 
the results on larger and sparser data cases are more informative for practical uses. 

Third, the $3,000$ iterations of collapsed Gibbs sampler did not work well as we expected. 
Interestingly, in the case of Lastfm ArtixtXTag with $K=60$, the Gibbs sampler performs 
significantly better than others. 
To explain this, we focus on the fact that the ACVB0 with $K=20$ is significantly better 
than the Gibbs with $K=60$. It indicates that the data has much smaller complexity than we expected. 
With greater $K$, the (C)VB inference algorithms may trapped at bad local optimum. 
Contrary, the collapsed Gibbs sampler obtained stable but not good solutions regardless of initial $K$. 
As reported in \citep{Albers13}, the collapsed Gibbs for IRM would require millions of iterations to 
obtain better results. 
Thus it is perfectly possible that the collapsed Gibbs outperforms all 
VB-based techniques provided the sophisticated sampling techniques and much more iterations. 

Figure \ref{fig:result_assignment_synth2} and \ref{fig:result_assignment_lastfmUserXUser} present 
examples of obtained clustering for \textbf{Synth2} and \textbf{Lastfm UserXUser} data in $K=60$. 
All object indices in the cites are sorted so that the objects are grouped into blocks in the cites. 
Horizontal and vertical color lines indicates the borderlines of object clusters for the first domain $i$ and the second domain $j$, 
respectively. 
We show the MAP assignments: we assign an object into the cluster with the highest posterior probability. 
%For relational data with $N > 1000$ objects, it is difficult to instantly derive some knowledge from these visualizations, or to argue differences between solutions. 

\begin{figure*}
\begin{center}
\includegraphics[width=140mm]{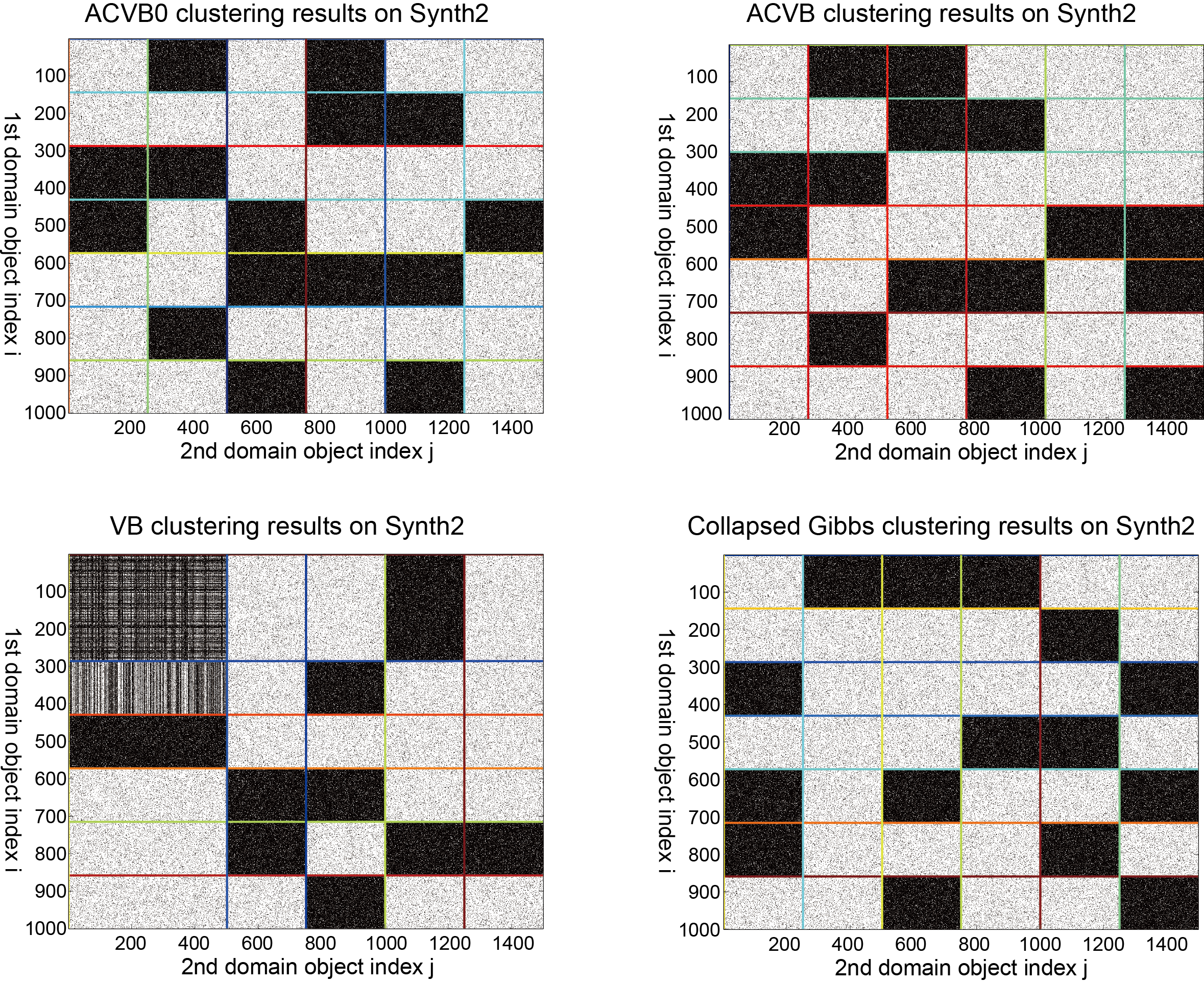}
\end{center}
\caption{MAP clustering assignments of \textbf{Synth2} dataset. All object indices are sorted. }
\label{fig:result_assignment_synth2}
\end{figure*}
\begin{figure*}
\begin{center}
\includegraphics[width=140mm]{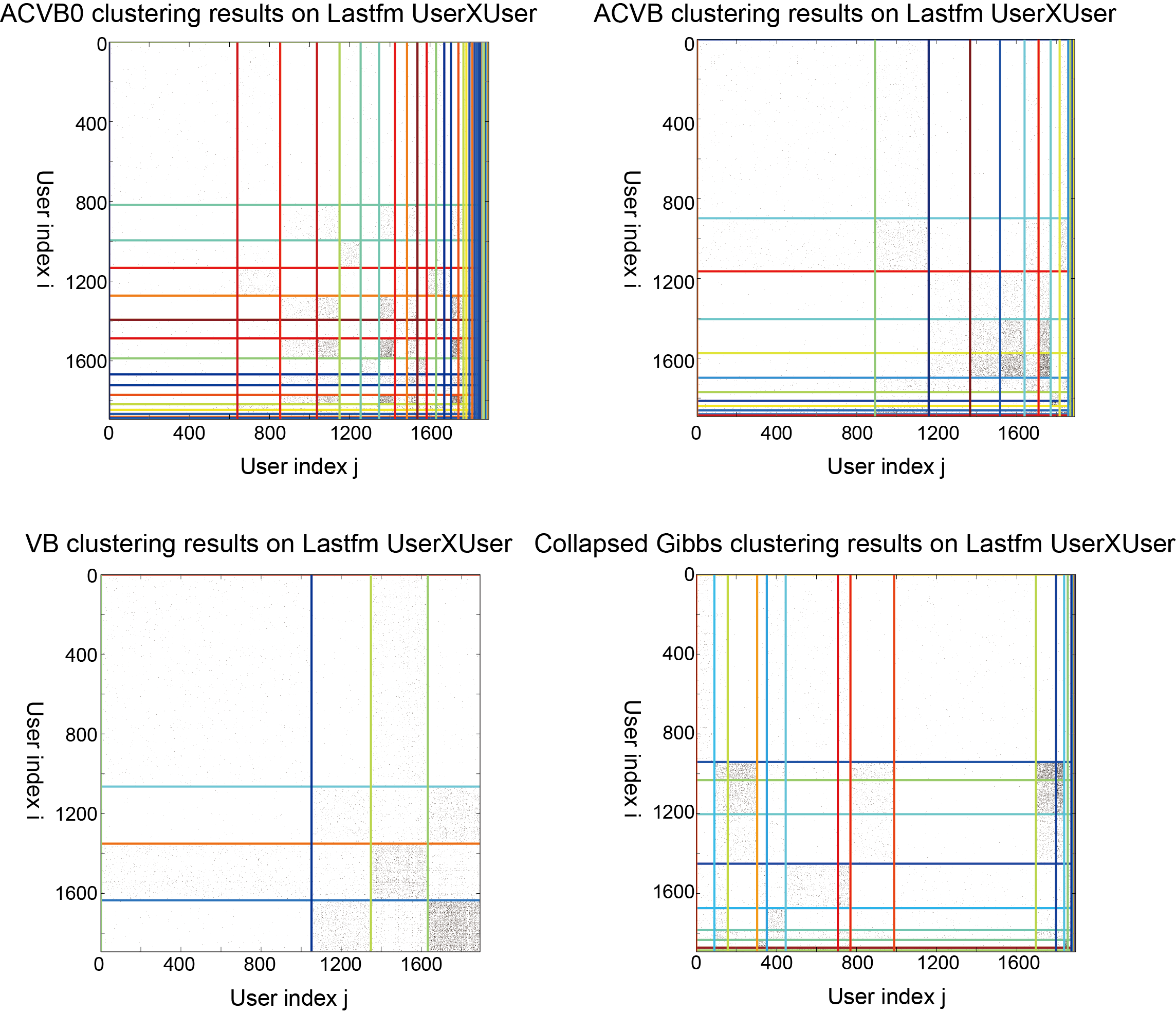}
\end{center}
\caption{MAP clustering assignments of \textbf{Lastfm UserXUser} dataset. All object indices are sorted. }
\label{fig:result_assignment_lastfmUserXUser}
\end{figure*}

%We will discuss later 
%the comparisons between ACVB solutions and the collapsed Gibbs.  

\subsubsection{Computational load and convergence behaviors}
To assess the computational loads of four solutions, 
we have monitored CPU times for convergence. 

First, we report the overall trends in convergence CPU times based on 
the average convergence time presented in 
\mytabref{tab:results_CPUtime_K20} ($K=20$), 
\mytabref{tab:results_CPUtime_K40} ($K=40$), and
\mytabref{tab:results_CPUtime_K60} ($K=60$). 
Aside from the collapsed Gibbs, which has no no definite way to detect convergence of the inference, 
ACVB0 was magnitude-faster than the naive VB and the ACVB (2nd order) for almost all datasets. 
There are several possible reasons. 
First, the update equations of (A)CVB0 is much simpler than that of 2nd-order ACVB. 
Second, ACVB inference has fewer unknown variables to estimate than the VB. 
Third, count statistics maintained in ACVB0 are able to efficiently cache and compute thanks to their simplicity. 
Fourth, the landscape of ACVB0 posteriors may have smoother charactersitics than those of the ACVB and the VB. 
Concerning the ACVB and the VB, the VB was faster when truncated $K$ is small. 
Also, the VB was faster than ACVB for dense synthetic data. In other cases, the ACVB was faster than the VB. 

We also need to note that 
%We will show a part of these monitoring results, but 
%we found that it is difficult to be conclusive about computational time. 
%This is because 
the CPU times for convergence are deeply affected by the convergence threshold. 
In our experiments, we choose the threshold of 
$1.0 \times 10^{-5}$ relative changes of the monitored quantities for 
VB, ACVB and ACVB0 solutions. 
If we change the threshold to $1.0 \times 10^{-4}$, 
convergence times of these solutions become 10 times or more faster. 
%making it difficult to be conclusive about computational loads. 
%Also, collapsed Gibbs samplers have no definite way to detect convergence of the inference. 
%We will obtain better results as we continue MCMC samplings much longer. 
%

Finally, we show a few plots of test data likelihood evolutions over CPU times. 
Figure \ref{fig:results_enron06_K40_vsTime}, 
\ref{fig:results_lastfmUserXUser_K20_vsTime}, 
\ref{fig:results_lastfmArtistXTag_K20_vsTime}, and 
\ref{fig:results_lastfmArtistXTag_K60_vsTime}
respectively illustrate the time evolution of test data likelihood versus CPU time on different datasets. 
End points of the plots indicate the average convergence time of 
(AC)VBs, or 3,000 iterations of collapsed Gibbs sampler. 
For all cases, we observe fast convergence evolutions of ACVB0. 
%The status of plots is very different among cites, but there are three common points to note. 
%\begin{enumerate}
%\item The VB solution converges very fast for larger datasets, at the expense of worse test data likelihood. We can even speed up VB inference because it allows massive parallelization of computations. 
%\item The collapsed Gibbs samplers keep oscillating near the final solution, even in the later part of the plots. 
%\item ACVB0 first rapidly improves the modeling, then slows down to the final convergence. 
%\end{enumerate}

%Combined wi th the automatic convergence detection, 
%we can say that ACVB0 has a good inference performance with very fast computations, 
%which is preferable for practical uses. 

From these experimental results, 
we conclude that ACVB0 solutions are good for practitioners 
who require good enough clustering results (possibly not a global
optimum) with very fast computations and assured convergence. 

\begin{table}[t]
\caption{Average CPU time for convergence of the solutions 
($K=20$, $10\%$ test data). 
 Parenthesized numbers indicate standard deviations. 
 Computational time for convergence detection is excluded. 
 }
\label{tab:results_CPUtime_K20}
 \begin{center}
  \begin{tabularx}{148mm}{X||c|c|c|c}
   Dataset & Gibbs & VB & CVB & CVB0 \\ \hline \hline
   Synth1 & 72.90 (6.617) & 1.750 (1.894) & 33.45 (23.10)  & 1.550 (1.071) \\ \hline 
   Synth2 & 3593.45 (22.17) & 42.10 (4.571) & 153.3 (69.34)  & 13.30 (11.20) \\ \hline 
   Enron Jun. & 62.70 (17.33) & 48.95 (37.14) & 159.7 (92.97) & 3.400 (2.130) \\ \hline 
   Enron Aug. & 66.85 (5.480) & 26.25 (32.64) & 39.55 (5.408) & 1.250 (0.433) \\ \hline 
   Enron Oct. & 94.50 (7.710) & 80.15 (38.70) & 113.4 (54.49) & 3.400 (2.437) \\ \hline 
   Enron Dec. & 46.10 (1.136) & 70.70 (43.65) & 58.80 (26.45) & 3.800 (2.482) \\ \hline
   Lastfm (User x User) & 15238 (811.4) & 4627 (8006) & 17668 (10427) & 400.5 (213.4) \\ \hline 
   Lastfm (Artist X Tag) & 35074 (1975) & 11237 (12183) & 73186 (44579) & 1024 (358.9) 
  \end{tabularx}
 \end{center}
\end{table}

\begin{table}[t]
\caption{Average CPU time for convergence of the solutions 
($K=40$, $10\%$ test data). 
 Parenthesized numbers indicate standard deviations. 
 Computational time for convergence detection is excluded. 
 }
\label{tab:results_CPUtime_K40}
 \begin{center}
  \begin{tabularx}{148mm}{X||c|c|c|c}
   Dataset & Gibbs & VB & CVB & CVB0 \\ \hline \hline
   Synth1 & 48.45 (1.396) & 3.750 (2.118) & 74.60 (32.67) & 2.700 (1.269) \\ \hline 
   Synth2 & 3970 (58.21) & 140.2 (12.51) & 474.0 (111.5) & 9.450 (9.227) \\ \hline 
   Enron Jun. & 73.50 (7.533) & 75.15 (69.41) & 39.40 (8.581) & 1.700 (0.557) \\ \hline 
   Enron Aug. & 66.20 (3.789) & 85.45 (92.13) & 30.05 (5.172) & 2.200 (0.400) \\ \hline 
   Enron Oct. & 69.05 (4.685) & 134.2 (79.49) & 36.60 (7.158) & 4.200 (1.470) \\ \hline 
   Enron Dec. & 46.85 (1.424) & 103.8 (76.04) & 25.10 (6.196) & 4.600 (1.685) \\ \hline
   Lastfm (User x User) & 15172 (592.7) & 19802 (18438) & 13809 (5017)  & 492.3 (238.7) \\ \hline 
   Lastfm (Artist X Tag) & 21509 (1131) & 14406 (12184) & 69993 (13541) & 2271 (287.7) 
  \end{tabularx}
 \end{center}
\end{table}

\begin{table}[t]
\caption{Average CPU time for convergence of the solutions 
($K=60$, $10\%$ test data). 
 Parenthesized numbers indicate standard deviations. 
 Computational time for convergence detection is excluded. 
 }
\label{tab:results_CPUtime_K60}
 \begin{center}
  \begin{tabularx}{148mm}{X||c|c|c|c}
   Dataset & Gibbs & VB & CVB & CVB0 \\ \hline \hline
   Synth1 & 72.20 (4.976) & 9.800 (4.045) & 62.75 (28.98) & 3.850 (1.236) \\ \hline 
   Synth2 & 4147 (65.96) & 318.7 (6.034) & 908.9 (64.66) & 11.55 (3.694) \\ \hline 
   Enron Jun. & 79.50 (24.52) & 202.1 (151.0) & 58.45 (12.61) & 1.000 (0.000) \\ \hline 
   Enron Aug. & 75.25 (7.203) & 287.4 (166.7) & 42.45 (9.516) & 2.100 (0.740) \\ \hline 
   Enron Oct. & 65.05 (6.087) & 291.3 (126.6) & 89.60 (34.02) & 2.200 (0.400) \\ \hline 
   Enron Dec. & 59.95 (5.643) & 225.6 (150.2) & 27.40 (5.054) & 3.550 (0.9206) \\ \hline
   Lastfm (User x User) & 14730 (785.9) & 21341 (26698) & 17357 (6121)  & 450.0 (182.8) \\ \hline 
   Lastfm (Artist X Tag) & 33704 (2391) & 33567 (32485)  & 64431 (8146) & 1870 (178.4) 
  \end{tabularx}
 \end{center}
\end{table}

% enron June K40
\begin{figure*}
\begin{center}	
\includegraphics[width=70mm]{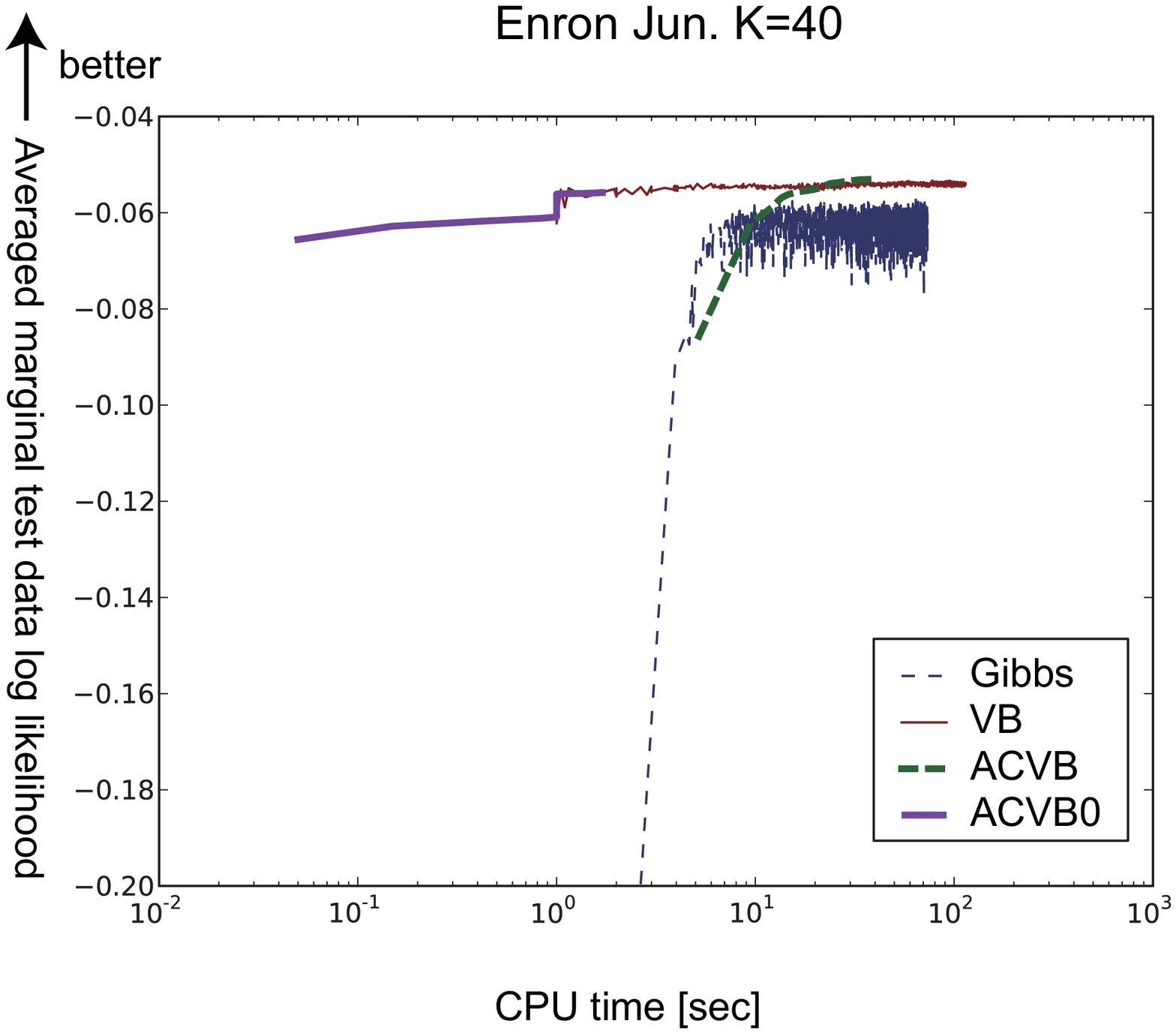}
\end{center}
\caption{Averaged test data marginal log likelihoods vs. inference
 CPU time on Enron Jun. data, $K$=40. 
 The horizontal axis denotes CPU time [sec], and the
 vertical axis denotes average test data marginal log likelihoods per relation entry. 
 Presented Gibbs results are those of sampled assignments, not of averaged posteriors. }
 \label{fig:results_enron06_K40_vsTime}
\end{figure*}

% Lastfm UserXUser K20
\begin{figure*}
\begin{center}	
\includegraphics[width=70mm]{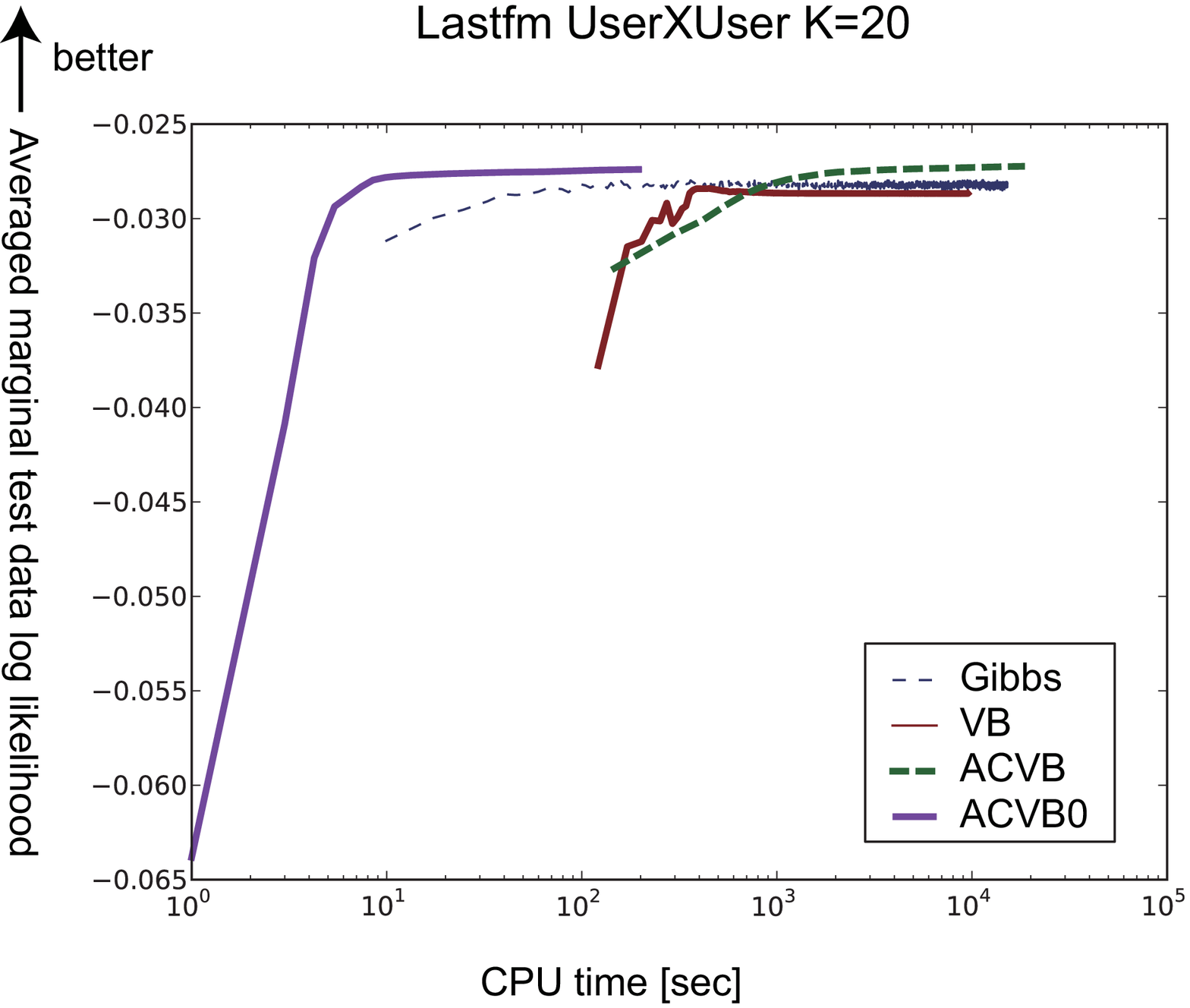}
\end{center}
\caption{Averaged test data marginal log likelihoods vs. inference
 CPU time on Lastfm UserXUser data, $K$=20. 
 The horizontal axis denotes CPU time [sec], and the
 vertical axis denotes average test data marginal log likelihoods per relation entry. 
 Presented Gibbs results are those of sampled assignments, not of averaged posteriors. }
 \label{fig:results_lastfmUserXUser_K20_vsTime}
\end{figure*}

% Lastfm ArtistXTag K20
\begin{figure*}
\begin{center}	
\includegraphics[width=70mm]{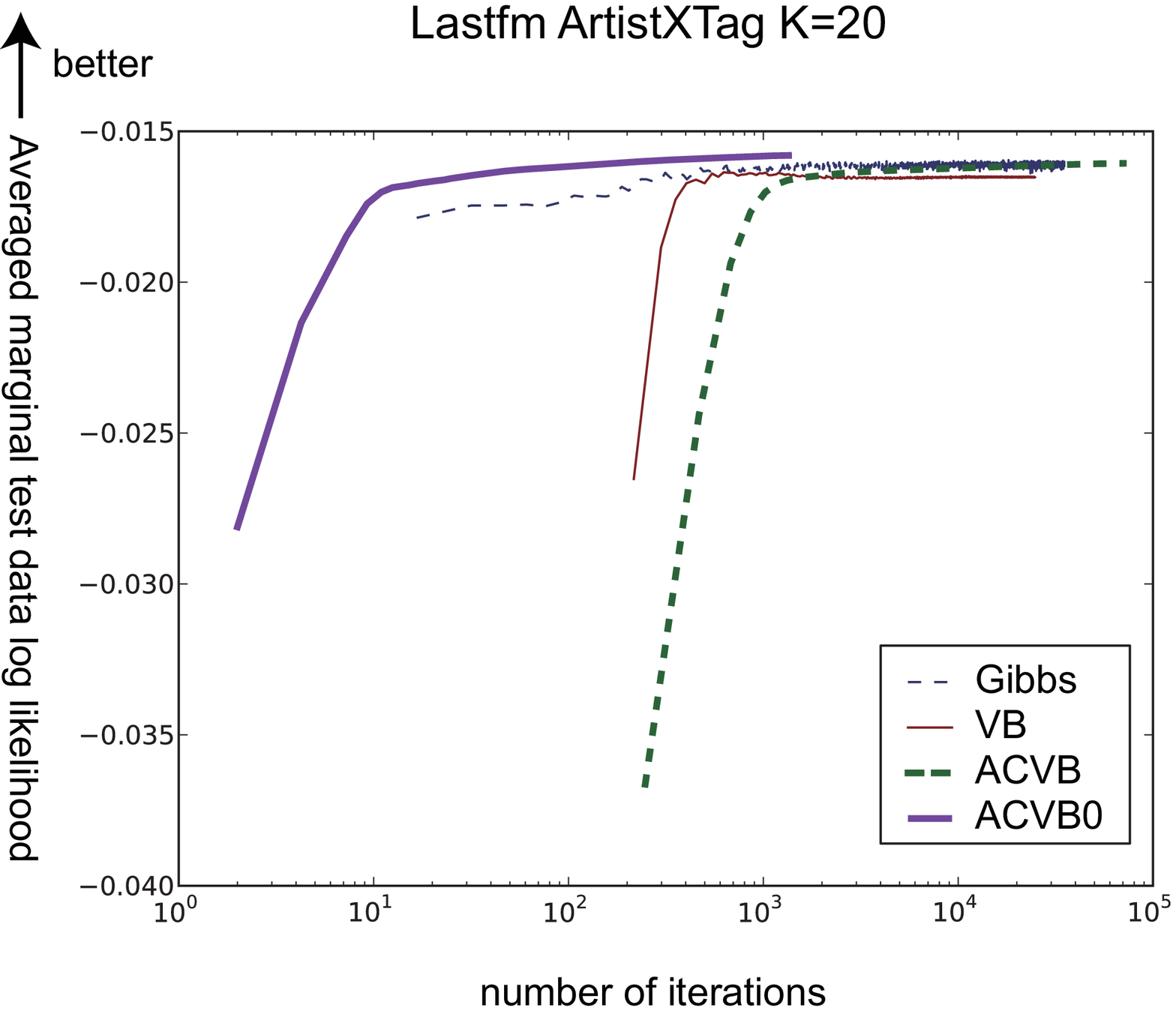}
\end{center}
\caption{Averaged test data marginal log likelihoods vs. inference
 CPU time on Lastfm ArtistXTag data, $K$=20. 
 The horizontal axis denotes CPU time [sec], and the
 vertical axis denotes average test data marginal log likelihoods per relation entry. 
 Presented Gibbs results are those of sampled assignments, not of averaged posteriors. }
 \label{fig:results_lastfmArtistXTag_K20_vsTime}
\end{figure*}

% Lastfm ArtistXTag K60
\begin{figure*}
\begin{center}	
\includegraphics[width=70mm]{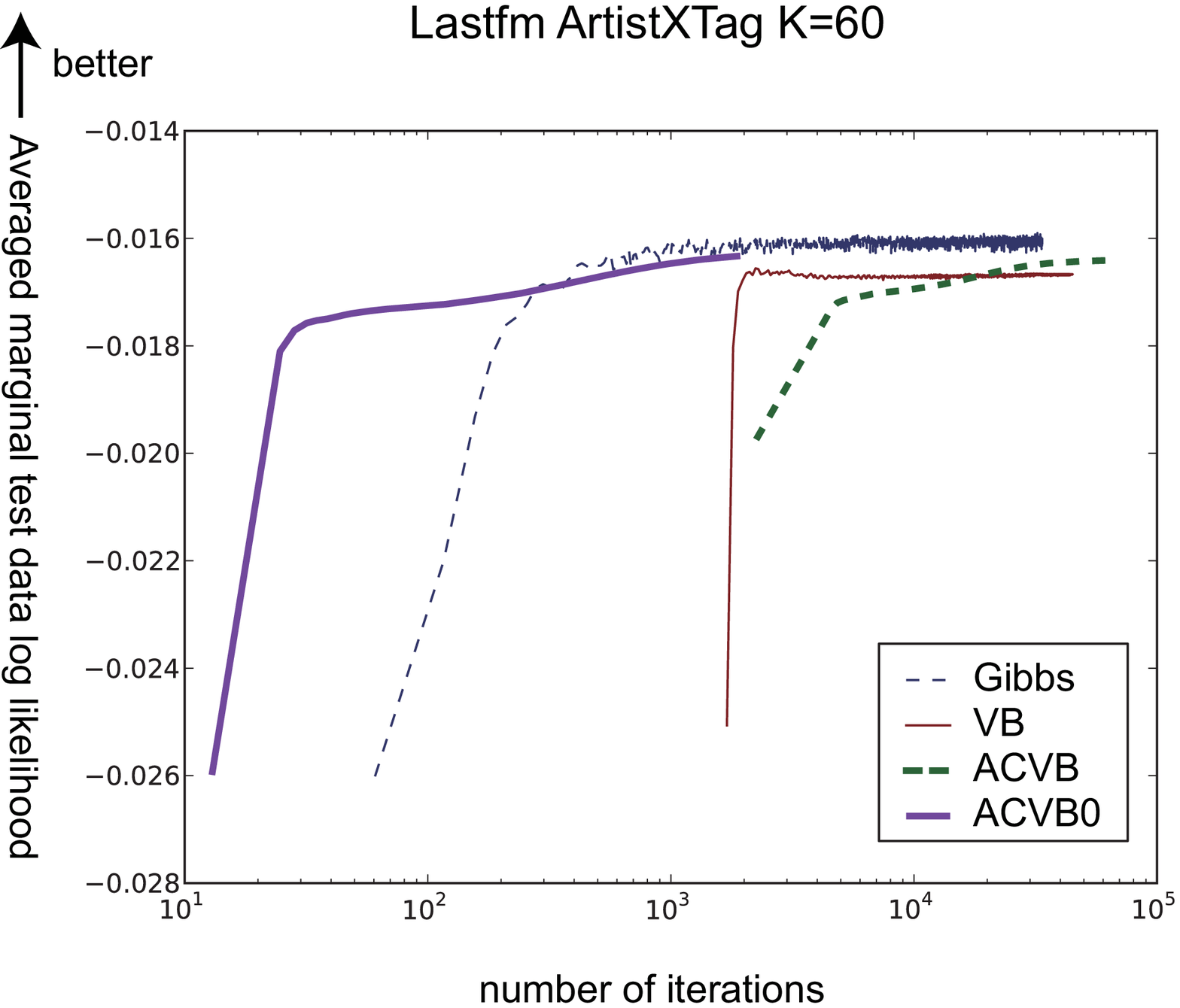}
\end{center}
\caption{Averaged test data marginal log likelihoods vs. inference
 CPU time on Lastfm ArtistXTag data, $K$=60. 
 The horizontal axis denotes CPU time [sec], and the
 vertical axis denotes average test data marginal log likelihoods per relation entry. 
 Presented Gibbs results are those of sampled assignments, not of averaged posteriors. }
 \label{fig:results_lastfmArtistXTag_K60_vsTime}
\end{figure*}

\subsection{Large data clustering experiment}
To further demonstrate the usefulness of IRM with ACVB0, 
we conduct further experiments on clustering of a larger dataset. 

Our scenario is a typical situation of practical relational data analysis. 
Our goal is to perform clustering of relational data, 
hoping to extract some knowledge from the data. 
We do not need to evaluate the test generalization performance; 
thus we assume no missing entries within the relational matrix $\bm{X}$ (or 
impute missing entries in preprocessing). 
Therefore, we can use a linear time ACVB0 inference. 
We cannot evaluate the data modeling performance by test data, 
thus we only show the computational time until convergence, 
with different $N$ and different truncation level $K$. 

We employ five relational data for clustering experiments. 
% Last FM, User by User 1892 by 1892
% Last FM, Artist by Tags,  6099 by 1088
First, we borrow the two largest datasets from the previous experiments: 
\textbf{Lastfm UserXUser} with the size of $N_{1}=N_{2}=1892$ users 
and \textbf{Lastfm ArtistXTag} with $N_{1} = 6,099$ artists by $N_{2} = 1,088$ tags. 
%
% Movielens-10M, 10681 movies by 69878 users, 6 ratings
The third data is the \textbf{Movielens-10M} dataset. 
The data consists of ratings on $N_{1} = 10,681$ unique movies by $N_{2} = 69,878$ unique users. 
As the name indicates, there are about 10 million ratings. 
We treat all rated entries (regardless of the rating points) as 
positive relations. 
%
% Netflix
Also, we prepare the fourth and fifth largest datasets from the \textbf{Netflix} data. 
The data consists of ratings by $N_{1} = 480,189$ unique users 
on $N_{2} = 17,770$ unique movies.  
% netflix rate_1 480189 users by 17770 movies, 4617990 NNZ
% netflix rate_5 480189 users by 17770 movies, 23168232 NNZ
We can use the full dataset, but 
we prepare two subsets of the Netflix data to measure the impact of 
the number of non-zero elements, which would affect the average degrees. 
\textbf{Netflix-rate1} data consists of rating entries that have ``1'' (worst) values. 
There are about 4 million ``1'' entries, and we treat them as the positive relation between users and movies. 
\textbf{Netflix-rate5} data consists of those with ``5'' (best) values. 
There are about 23 million ``5'' entries that are assumed as positive relations. 

\begin{table*}
\caption{Data sizes and ACVB0 CPU times [sec] in large data clustering experiments.  }
 \label{tab:large_data_result}
 \begin{center}  
 \small
  \begin{tabularx}{156mm}{X||c|c|c||c|c|c|c|c|c}   
   \multirow{2}{*}{Dataset} & \multirow{2}{*}{$N_{1}$} & \multirow{2}{*}{$N_{2}$} & \multirow{2}{*}{Density} & \multicolumn{3}{c|}{CPU times (linear)} & \multicolumn{3}{c}{CPU times (naive)}\\
    & & & & $K$=20 & $K$=40 & $K$=60 & $K$=20 & $K$=40 & $K$=60 \\
   \hline \hline
   Lastfm U.XU. & 1892 & 1892 & 0.60\% & 183.2 & 236.6 & 194.6 & 400.5 & 492.3 & 450.0 \\
   Lastfm A.XT. & 6099 & 1088 & 0.35\% & 414.3 & 621.9 & 826.6 & 1024 & 2271 & 1870 \\
   Movielens-10M & 10681 & 69878 & 1.34\% & 2087 & 9569 & 22466 & NA & NA & NA \\
   Netflix-rate1 & 480189 & 17770 & 0.05\% & 24973 & 112116 & 126482 & NA & NA & NA \\
   Netflix-rate5 & 480189 & 17770 & 0.27\% & 13440 & 104397 & 208604 & NA & NA & NA
  \end{tabularx}
 \end{center}
\end{table*}

\mytabref{tab:large_data_result} presents the CPU times for convergence. 
We tested on several hyperparameter setups, and report the average CPU times 
of the setup of the best training data log likelihood. 
As evident from the table, the linear inference of ACVB0 enables clustering computations on 
large relational data. 

We can observe that the computational times are affected by several factors, 
but not as predicted from the theory. 
The computational order of IRM-ACVB0 is $O(L \times (N_{1} + N_{2}) \times K_{1} \times K_{2})$. 
Thus, computational time would grows linear to the number of objects and the density, 
and square to $K$. 
%For example, the computational order of IRM-ACVB0 is $O(L(N_{1} + N_{2}))$ in the case of linear inferences. 
In general, datasets with large $N_{1}$ , $N_{2}$ took more CPU time for convergence but 
CPU time is not proportional to $N$. 
%
%We see this holds for two \textbf{Lastfm} datasets. 
%The number of objects of \textbf{Lastfm ArtistXTag} is approximately double 
%that of \textbf{Lastfm UserXUser}. The CPU times result in a similar proportion. 
For the effect of data density, please consult the rows of \textbf{Netflix} data. 
We see that the density does not necessarily governs the CPU times for convergence.  
%
%The factor of model complexity is also notable. 
The CPU time does not grow larger as expected from the model complexity as well. 
We expect to have four and nine times larger CPU times at $K=40$ models and $K=60$ models, 
compared to the $K=20$ models. 
However, this does not hold for all datasets excepting \textbf{Movielens-10M}. 
This is because the ACVB0 model shrinks as the model discovers lesser 
numbers of latent clusters from the given data, thanks to the cluster shrinkage technique 
introduced in the Speeding-Up section. 
%In other words, the computational time of 
%the ACVB0 inference is affected by the complexity of the hidden cluster structure. 

Also note the convergence CPU time is deeply affected by the threshold of convergence detection: 
if we loosen the threshold from $0.001\%$ relative changes to $0.01\%$, 
convergence typically becomes 10 times (or more) faster. 

We argue that the convergence-guaranteed ACVB0 is 
especially beneficial for large data analysis.  
We can solve collapsed Gibbs in linear time as well, but 
several millions of iterations are not enough to obtain good posterior estimations~\citep{Albers13}. 
Also the collapsed Gibbs requires to monitor 
the inference process because we have no measure to detect convergence. 
In large data analysis, this is costly and painful. % to manually monitor experiments. 
In contrast, ACVB0 does not require such elaboration because 
it can detect the assured convergence easily. 
Combining the test data modeling results and very fast computation times, 
the proposed ACVB0 solution is a good practical choice for IRM, even for large relational data.

\section{Conclusion}
In this paper, we proposed Averaged collapsed variational Bayes (ACVB) inference of 
the Infinite Relational Model (IRM), 
which is a 
convergence-guaranteed and practically useful deterministic inference algorithm 
to replace naive VB. 

First, we formulated a CVB lower bound for IRM based on the standard procedure, 
which is intractable to evaluate exactly. 
For this problem, we used Taylor approximations 
as in CVB research on topic models, 
and derived the full formulations and 
the inference procedure for two types of CVB solutions. 
We also provided the CVB0-based update rules of hyperparameters, including the
concentration parameter of the Dirichlet Process, which has been never reported in the literature. 

To make the CVB inference more practically useful, we studied the CVB inference in two aspects. 
First is the convergence issue, which is an open problem for the CVB inference. 
We started by examining two possible quantities to assess the convergence of CVB solutions of IRM. 
After that, we proposed a simple and effective
annealing technique, Averaged CVB (ACVB), to assure the convergence of CVB solutions. 
ACVB posterior update offers assured convergence thanks to its simple annealing mechanism. 
Moreover, the stationary point of the CVB lower bound is equivalent to the converged solution of ACVB, 
if the lower bound has a stationary point (an issue unresolved in the literature). 
ACVB is applicable to any model, and is equally valid for CVB and CVB0. 

The second aspect is the computational speed of CVB. 
We proposed a cluster shrinkage technique and a linear-time inference implementation. 
These techniques make the IRM inference more scalable against the data size, 
and open the door to larger and more complex relational data analysis applications. 

The resulting CVB solutions offer more precise inference than naive VB in experiments. 
At the same time, the annealing ACVB technique allows us to 
automatically detect convergence and yields short computational time. 
We also confirmed that the linear time inference of (A)CVB0 allows us to analyze large 
two-place relational data. 

As future work, we will further enhance inference speed. 
One possible solution is to stochastically approximate the sample size 
as in SGD. Recently, \citep{Foulds13} proposed such approximation for LDA. 
Another way is to parallelize the inference procedure, 
as 
\citep{Hansen11,Albers13} have examined the parallelization of collapsed Gibbs samplers on IRM. 
%The fixed number truncation of CVB simplifies the parallelization compared to 
%collapsed Gibbs samplers.  
It is also important to explore efficient CVB algorithms for more advanced models 
such as MMSB and its followers~\citep{Airoldi08, Miller09, Griffiths_Ghahramani11}. 
Aside from the representation of multiple cluster assignments, 
a few studies have headed toward to other issues. 
For example, \citep{Fu09,NIPS10} focused on dynamics of network evolution in the context of 
stochastic blockmodels (MMSB and IRM). 
Subset IRM~\citep{AISTATS12} is another extension of IRM that automatically "filters out" 
nodes from the clustering that are not so informative to group. 
Applying CVB for these models may make it easier for practitioners to examine the depth of various relational data.

\vskip 0.2in
\bibliography{myref}

\end{document}